\documentclass{article}

\PassOptionsToPackage{numbers, compress}{natbib}
\usepackage[final]{neurips_2022}

\usepackage[utf8]{inputenc} % allow utf-8 input
\usepackage[T1]{fontenc}    % use 8-bit T1 fonts
\usepackage{hyperref}       % hyperlinks
\usepackage{url}            % simple URL typesetting
\usepackage{booktabs}       % professional-quality tables
\usepackage{amsfonts}       % blackboard math symbols
\usepackage{nicefrac}       % compact symbols for 1/2, etc.
\usepackage{microtype}      % microtypography
\usepackage{xcolor}         % colors

\title{Concept Activation Regions: A Generalized Framework For Concept-Based Explanations}

\author{%
	Jonathan Crabbé \\
	University of Cambridge \\
	\texttt{jc2133@cam.ac.uk} \\
	\And
	Mihaela van der Schaar\\
	University of Cambridge \\
	The Alan Turing Institute \\
	UCLA \\
	\texttt{mv472@cam.ac.uk} \\
}

\usepackage{amsmath,amsfonts,amssymb}
\usepackage{amsthm}
\usepackage{graphicx}
\usepackage{caption}
\usepackage{subcaption}
\usepackage{soul}
\usepackage{hyperref}
\usepackage{tikz}
\usepackage{wrapfig}
\usepackage{lipsum}
\usepackage[ruled]{algorithm2e}
\usepackage{setspace}
\usepackage{makecell}
\usepackage{adjustbox}
\usepackage{multirow}
\usepackage{diagbox}
\usepackage{pifont}
\usepackage{transparent}

\SetKwComment{Comment}{/* }{ */}

\hypersetup{
	colorlinks=true,
	linkcolor=blue,
	filecolor=magenta,      
	urlcolor=cyan,
	citecolor=blue,
	pdfpagemode=FullScreen,
}

\newtheorem{proposition}{Proposition}[section]
\theoremstyle{definition}
\newtheorem{definition}{Definition}[section]
\newtheorem{assumption}{Assumption}[section]
\theoremstyle{remark}
\newtheorem{remark}{Remark}[section]

\newcommand{\X}{\mathcal{X}}
\newcommand{\Y}{\mathcal{Y}}
\renewcommand{\H}{\mathcal{H}}
\newcommand{\D}{\mathcal{D}}

\renewcommand{\S}{\mathcal{S}}
\newcommand{\T}{\mathcal{T}}
\newcommand{\Dtrain}{\D_{\textrm{train}}}
\newcommand{\Pos}{\mathcal{P}}
\newcommand{\Neg}{\mathcal{N}}
\newcommand{\R}{\mathbb{R}}
\newcommand{\N}{\mathbb{N}}
\newcommand{\x}{\boldsymbol{x}}
\newcommand{\y}{\boldsymbol{y}}
\newcommand{\h}{\boldsymbol{h}}
\newcommand{\f}{\boldsymbol{f}}
\newcommand{\g}{\boldsymbol{g}}
\renewcommand{\l}{\boldsymbol{l}}
\newcommand{\w}{\boldsymbol{w}}
\newcommand{\grad}{\boldsymbol{\nabla}}
\newcommand{\TCAV}{\textrm{TCAV}}
\newcommand{\TCAR}{\textrm{TCAR}}
\newcommand{\card}[1]{\left| #1 \right|}

\renewcommand{\a}{\boldsymbol{a}}

\newcommand{\norm}[2][\X]{\left\| #2 \right\|_{#1}}
\newcommand{\takeaway}[1]{\textcolor{purple}{\textbf{Take-away #1:}}}
\newcommand*\circled[1]{\tikz[baseline=(char.base)]{
		\node[shape=circle,draw,inner sep=.3pt] (char) {#1};}}
\newcommand{\isometry}{\boldsymbol{\tau}}
\newcommand{\isometryarrow}{\overset{\isometry}{\mapsto}}
\renewcommand{\r}{\boldsymbol{r}}
\newcommand{\Jacobian}{\boldsymbol{J}}
\newcommand{\Lagrangian}{\mathcal{L}}
\newcommand{\dualpha}{\boldsymbol{\alpha}}
\newcommand{\cmark}{\ding{51}}%
\newcommand{\xmark}{\transparent{0.4}\ding{55}}%

\DeclareMathOperator{\softmax}{Softmax}
\DeclareMathOperator{\idg}{IDG}
\DeclareMathOperator{\edg}{EDG}

\begin{document}
	
	\maketitle
	
	\begin{abstract}
		Concept-based explanations permit to understand the predictions of a deep neural network (DNN) through the lens of concepts specified by users. Existing methods assume that the examples illustrating a concept are mapped in a fixed direction of the DNN's latent space. When this holds true, the concept can be represented by a concept activation vector (CAV) pointing in that direction. In this work, we propose to relax this assumption by allowing concept examples to be scattered across different clusters in the DNN's latent space. Each concept is then represented by a region of the DNN's latent space that includes these clusters and that we call concept activation region (CAR). To formalize this idea, we introduce an extension of the CAV formalism that is based on the kernel trick and support vector classifiers. This CAR formalism yields global concept-based explanations and local concept-based feature importance. We prove that CAR explanations built with radial kernels are invariant under latent space isometries. In this way, CAR assigns the same explanations to latent spaces that have the same geometry. We further demonstrate empirically that CARs offer (1) more accurate descriptions of how concepts are scattered in the DNN's latent space; (2) global explanations that are closer to human concept annotations and (3) concept-based feature importance that meaningfully relate concepts with each other. Finally, we use CARs to show that DNNs can autonomously rediscover known scientific concepts, such as the prostate cancer grading system.

	\end{abstract}
	
	\section{Introduction} \label{sec:introduction}
	Deep learning models are both useful and challenging. Their utility is reflected in their increasing contributions to sophisticated tasks such as natural language processing~\cite{Young2018, Chowdhary2020}, computer vision~\cite{Szeliski2010, Voulodimos2018} and scientific discovery~\cite{Jumper2021, Davies2021, Kirkpatrick2021, Degrave2022}. Their challenging nature can be attributed to their inherent complexity. State of the art deep models typically contain millions to billions parameters and, hence, appear as \emph{black-boxes} to human users. The opacity of black-box models make it difficult to: anticipate how models will perform at deployment~\cite{Lipton2016}; reliably distil knowledge from the models~\cite{Doshi-Velez2017} and earn the trust of stakeholders in high-stakes domains~\cite{Ching2018,Tonekaboni2019,Langer2021}. With the aim of increasing the transparency of black-box models, the field of explainable AI (XAI) developed~\cite{Adadi2018, Das2020, Tjoa2020, BarredoArrieta2020}. We can broadly divide XAI methods in 2 categories: \circled{1} Methods that restrict the model's architecture to enable explanations. Examples include attention models that motivate their predictions by highlighting features they pay attention to~\cite{Choi2016} and prototype-based models that motivate their predictions by highlighting relevant examples from their training set~\cite{Chen2018}. \circled{2} \emph{Post-hoc} methods that can be used in a plug-in fashion to provide explanation for a pre-trained model. Examples include \emph{feature importance methods} (also known as feature attribution or saliency methods) that highlight features the model is sensitive to~\cite{Binder2016,Ribeiro2016,Fong2017,Shrikumar2017,Lundberg2017,Sundararajan2017,Crabbe2021Dynamask}; \emph{example importance methods} that identify influential training examples~\cite{Koh2017, Ghorbani2019Shap, Pruthi2020} and \emph{hybrid methods} combining the two previous approaches~\cite{Crabbe2021Simplex}. In this work, we focus on a different type of explanation methods known as \emph{concept-based} explanations. Let us now summarize the relevant literature to contextualize our own contribution.    

\textbf{Related work.} Concept-based explanations were first formalized with the \emph{concept activation vector} (CAV) formalism~\cite{Kim2017, Schrouff2021}. Given a concept specified by the user (e.g. stripes in an image), linear classifiers are used as {probes}~\cite{Alain2016} to assess whether a deep model's representation space separates examples where a concept  is present (\emph{concept positive examples}) from examples where a concept is absent (\emph{concept negative examples}). A CAV is then extracted from the linear classifier's weights. With this CAV, it is possible to provide post-hoc explanations such as the sensitivity of a model's prediction to the presence/absence of a concept. It goes without saying that several concepts are needed to explain the prediction of a deep model. To formalize this idea, existing works have used concept basis decomposition~\cite{Zhou2018} and sufficient statistics~\cite{Yeh2020}. Although most applications of CAV involve image data, we note that the formalism has been successfully applied to time-series data~\cite{Kusters2020,Mincu2020}. While concepts are typically specified by the user, early works in computer vision have been undertaken to discover concepts in the form of meaningful image segmentations~\cite{Ghorbani2019}. The main criticism against the CAV formalism is that it requires concept positive examples to be linearly separable from concept negative examples~\cite{Chen2020,Soni2020}. This is because linear separability of concept sets is a restrictive criterion that the model is not explicitly trained to fulfil. To address this issue, works like \emph{concept whitening transformations}~\cite{Chen2020} and \emph{concept bottleneck models}~\cite{Koh2020, Barbiero2021, Georgiev2021} propose to do away with the post-hoc nature of concept-based explanations. These methods introduce new neural network architectures that permit to train the models with concept labels. We stress that this requires the set of concepts to be specified \emph{before} training the model. This assumes that we know what concepts are relevant for a model to solve a downstream task \emph{a-priori}. This is not the case whenever we train a model to solve a task for which little or no knowledge is available. In this setup, it seems more appropriate to train a model for the task first and, then, perform a post-hoc analysis of the model to determine the concepts that were relevant in providing a solution. Furthermore, recent concerns have emerged regarding the reliability of interpretations provided by these altered model architectures~\cite{Margeloiu2021,Mahinpei2021}. 

\begin{figure}
	%\vspace{-.8cm}
	\centering
	\includegraphics[width=.8\textwidth]{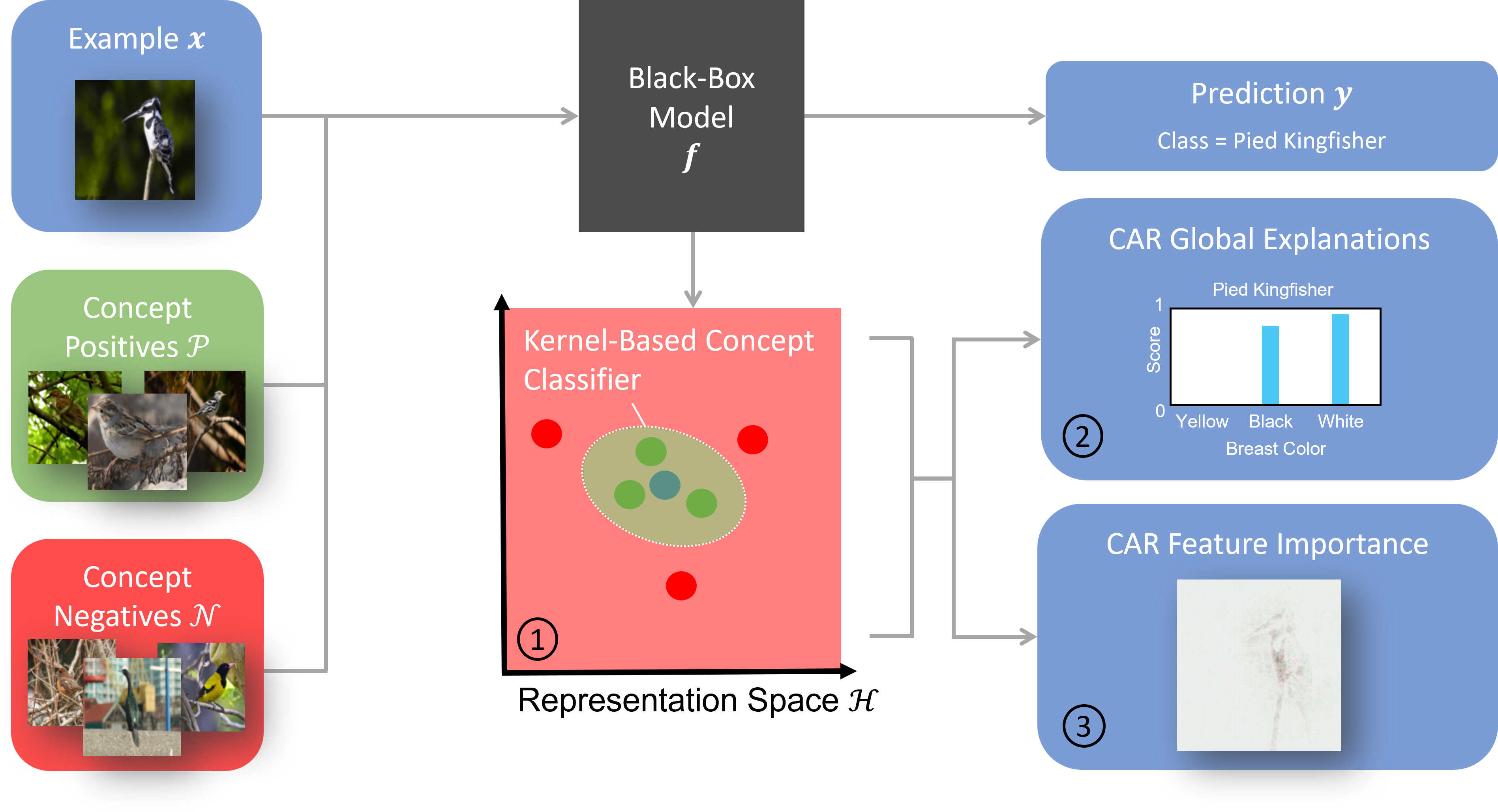}
	\caption{Illustration of the CAR framework. To define a concept, a set of positive and negative examples are fed to the model. In this illustration, we consider a neural network predicting the species of a bird based on a picture. The concept we illustrate is \emph{white breast}. Concept positive and negative images respectively exhibit birds with and without white breasts. CARs aim at explaining the prediction for a given example $\x$. \textcircled{1}~All the examples are mapped in the model's representation space $\H$. Unlike CAVs~\cite{Kim2017}, CARs rely on kernel-based support vector classifiers (SVC) and do not require the positive and negative sets to be linearly separable in $\H$. \textcircled{2}~CAR uses the SVC to output TCAR scores that indicate how classes are related to concepts. \textcircled{3}~The CAR formalism can be used in combination with any feature importance method for deep neural networks. This permits to create concept-specific saliency maps.}
	\label{fig:main_figure}
\end{figure}

\textbf{Contributions.} In this work, our purpose is to retain the flexible post-hoc nature of concept-based explanations without assuming that the concept sets are linearly separable. To that aim, we introduce \emph{concept activation regions} (CARs), an extension of the CAV framework illustrated in Figure~\ref{fig:main_figure}. \textbf{\circled{1}~Generalized formalism.}~As a substitute to linear separability, we propose in Section~\ref{subsec:car_prelim} to adapt the smoothness assumption from semi-supervised learning~\cite{Chapelle2006}. Intuitively, this more general assumption only requires positive and negative examples to be scattered across distinct clusters in the model's representation space. In practice, this generalization is implemented by substituting CAV's linear classifiers by  kernel-based support vector classifiers. We demonstrate that choosing radial kernels leads to CAR classifiers that are invariant under isometries of the latent space. This permits to assign identical explanations to latent spaces characterized by the same geometry. Moreover, we show in Section~\ref{subsec:car_acc_validation} that our CAR classifiers yield a substantially more accurate description of how concepts are distributed in deep model's representation spaces. \textbf{\circled{2}~Better global explanations.}~With the nonlinear decision boundaries of our support vector classifiers, there is no obvious way to adapt the notion of concept activation vectors. Since CAV's concept importance (TCAV score) is computed with these vectors, we need an alternative approach. In Section~\ref{subsec:car_concept_detection}, we propose to define concept importance by building on our smoothness assumption. Concretely, a concept is important for a given example if the model's representation for this example lies in a cluster of concept positive representations. With this characterization, we define TCAR scores that are the CAR equivalent of TCAV scores. In Section~\ref{subsec:car_global_validation}, we demonstrate that TCAR scores lead to global explanations that are more consistent with concept annotations provided by humans. \textbf{\circled{3}~Concept-based feature importance.}~In Section~\ref{subsec:car_concept_features}, we argue that our CAR formalism permits to assign concept-specific feature importance scores for each example fed to the neural network. We verify empirically in Section~\ref{subsec:car_feature_importance_validation} that those feature importance scores reflect meaningful concept associations. Finally, we illustrate in Section~\ref{subsec:seer_use_case} how these contributions permit to establish that deep models implicitly discover known scientific concepts. 

	\section{Concept Activation Regions (CARs)} \label{sec:car}
	\subsection{Preliminaries}  \label{subsec:car_prelim}
We assume a typical supervised setting where each sample is represented by a couple $(\x , \y)$ with input features $\x \in \X \subseteq \R^{d_X}$ and a label $\y \in \Y \subseteq \R^{d_Y}$, where $d_X, d_Y \in \N^*$ are respectively the dimensions of the feature (or input) and label (or output) spaces~\footnote{Note that we use bold symbols for vectors.}. In order to predict the labels from the features, we are given a deep neural network (DNN) $\f = \l \circ \g : \X \rightarrow \Y$, where $\g: \X \rightarrow \H$ is a feature extractor that maps features $\x \in \X$ to latent representations $\h = \g(\x) \in \H \subseteq \R^{d_H}$ and $\l : \H \rightarrow \Y$ maps latent representations $\h \in \H$ to labels $\y = \l(\h) \in \Y$. The representation $\h = \g(\x)$ typically corresponds to the output from one of the DNN's hidden layer. We assume that the model $\f$ was obtained by fitting a training set $\Dtrain \subset \X \times \Y$. Our purpose is to \emph{understand} how the neural network representation induced by $\g$ allows the model $\f$ to solve the prediction task. To that end, we shall use concept-based explanations. Those explanations rely on a set of $C \in \N^*$ concepts indexed~\footnote{It is understood that there exists a dictionary between the integer concept identifier and the concept's name. In this way, if the first concept of interest is \emph{stripes}, we can write $c=1$ and $c = \textrm{Stripes}$.} by $c \in [C]$, where $[C]$ denotes the set of positive integers between $1$ and $C$. Each concept $c \in [C]$ is defined by the user through a set of $N^{c} \in \N^*$ positive examples $\Pos^c = \left\{ \x^{c, n} \mid n \in [N^{c}] \right\}$ and a set of $N^{\neg c}$ negative examples $\Neg^c = \left\{ \x^{\neg c, n} \mid n \in [N^{\neg c}] \right\}$. The concept $c \in [C]$ is present in the positive examples $\Pos^c$ and absent from the negative examples $\Neg^c$. For instance, if we work in a computer vision context and the concept of interest is $c = \textrm{Stripes}$, then $\Pos^c$ contains images with stripes (e.g. zebra images) and $\Neg^c$ contains images without stripes (e.g. cow images). We note that non-binary concepts (e.g. a colour) can be one-hot encoded to fit this formulation. In this paper, we use balanced sets: $N^c = N^{\neg c}$.

\textbf{CAV formalism.} Let us start by summarizing the CAV formalism~\cite{Kim2017}. This formalism relies on the crucial assumption that the sets $\g(\Pos^c) \subset \H$ and $\g(\Neg^c) \subset \H$ are linearly separable in $\H$. Hence, it is possible to find a vector $\w^c \in \R^{d_H}$ and a bias term $b^c \in \R$ such that $(\w^c)^\intercal \h + b^c > 0$ for all $\h \in \g(\Pos^c)$ and, conversely, $(\w^c)^\intercal \h + b^c < 0$ for all $\h \in \g(\Neg^c)$. The vector $\w^c$ is called the \emph{concept activation vector} (abbreviated CAV) associated to the concept $c \in [C]$. Geometrically, this vector is normal to the hyperplane separating $\g(\Pos^c)$ from $\g(\Neg^c)$ in the latent space $\H$. Intuitively, this vector points in a direction of the latent space where the presence of the concept $c \in [C]$ increases. When $f_k(\x) = l_k \circ \g(\x)$ corresponds to the predicted probability of class $k \in [d_Y]$ for $\x \in \X$, this insight allows us to define a class-wise conceptual sensitivity. Indeed, the directional derivative $S^c_k(\x) \equiv (\w^c)^\intercal \grad_{\h} l_k \left[ \g(\x) \right]$ measures the extent to which the probability of class $k$ varies in the direction of the CAV. If the sensitivity is positive $S^c_k(\x) > 0$, this indicates that the presence of concept $c$ increases the belief of the model $\f$ that $k$ is the correct class for example $\x$. Note that CAV sensitivities $S^c_k(\x)$ can be generalized to aggregate the gradients between $\g(\x)$ and a baseline $\bar{\h} \in \H$~\cite{Schrouff2021}. Given a dataset of examples split by class $\D = \bigsqcup_{k=1}^{d_Y} \D_k$, where $\D_k$ contains only examples of class $k \in [d_Y]$, it is possible to summarize the overall sensitivity of class $k$ to concept $c$ by computing the score $\TCAV^c_k = \nicefrac{\card{ \left\{ \x \in \D_k \mid S^c_k(\x) > 0 \right\}}}{\card{\D}}$, where $\card{\cdot}$ denotes the set cardinality.

\begin{wrapfigure}{r}{.3\textwidth}
	\vspace{-.4cm}
	\centering
	\includegraphics[width=.3\textwidth]{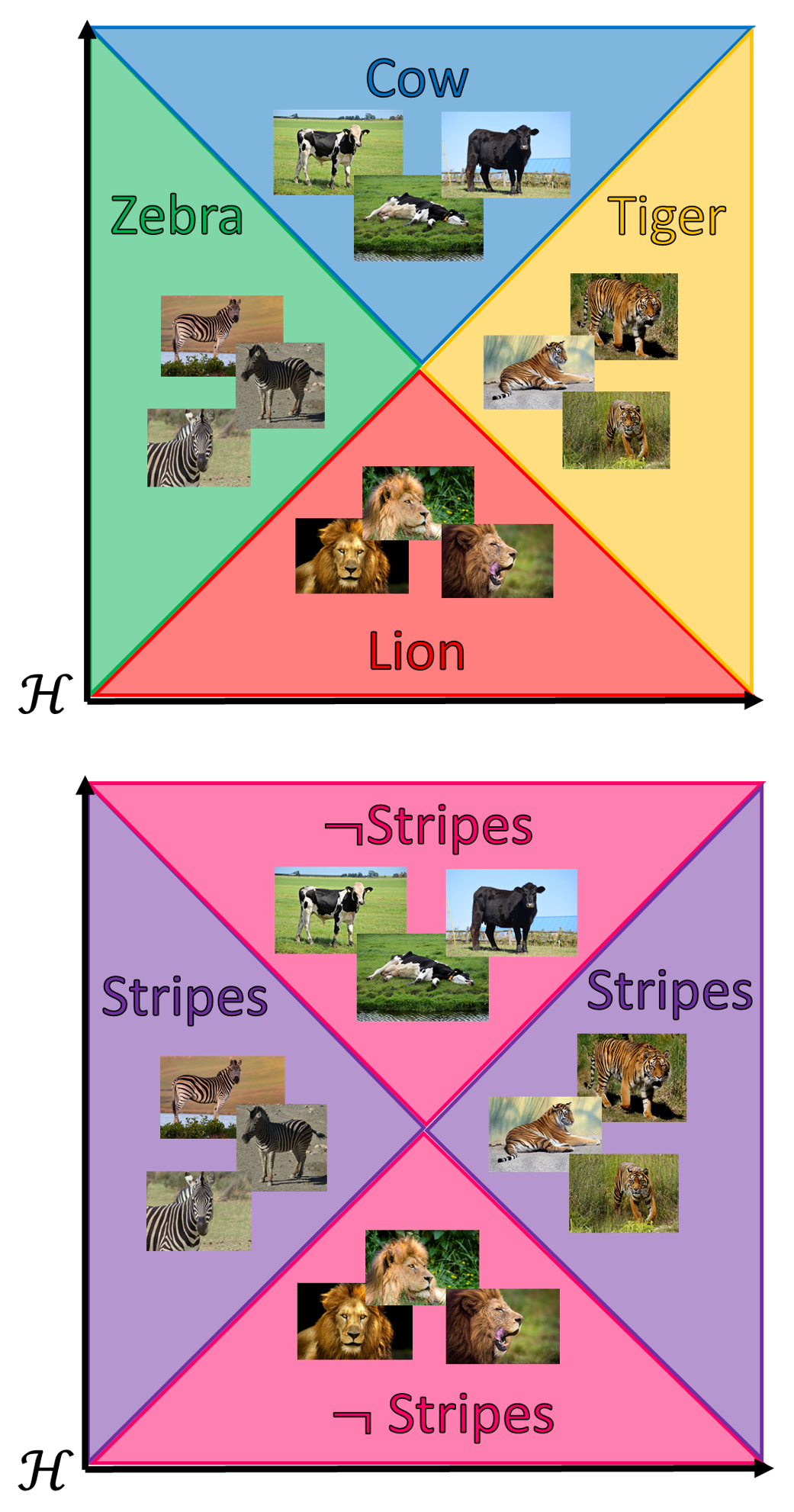}
	\caption{Classes are linearly separable (top) but concept sets are not (bottom).}
	\label{fig:concept_linear_separability}
\end{wrapfigure}
\textbf{The limitations of linear separability.} The linear separability of concept negatives and positives is central in the aforementioned formalism. Indeed, the existence of a separating hyperplane in $\H$ is necessary to define CAVs, which in turn are required to define concept sensitivity and TCAV scores. This assumption has been criticized in the literature~\cite{Chen2020, Soni2020}. The main criticism is the following: there is no reason to expect the representation map $\g$ to linearly separate concept positive and negatives since these concept-related labels are not known by the model. Let us consider a concrete example to stress that generic classifiers should not be expected to linearly separate concepts even if they linearly separate classes. Consider a classifier $\f = \l \circ \g$, where for each class $k \in [d_Y]$: $l_k(\h) = \softmax(\boldsymbol{\alpha}_k^\intercal \h + \beta_k)$ with $\boldsymbol{\alpha}_k \in \R^{d_H}$ and $\beta_k \in \R$. This parametrization is typically realized when the representation space $\H$ corresponds to the penultimate layer of a neural network. For concreteness, we consider a computer vision setting with the classes $k \in \{ \textrm{Tiger}, \textrm{Lion}, \textrm{Cow}, \textrm{Zebra} \}$. We note that the decision boundary in $\H$ for any pair of class is a hyperplane. For instance, the decision boundary between the classes lion and tiger is parametrized by $l_{\textrm{Lion}}(\h) = l_{\textrm{Tiger}}(\h)$, which corresponds to the equation of a hyperplane in $\H$: $(\boldsymbol{\alpha}_{\textrm{Lion}} - \boldsymbol{\alpha}_{\textrm{Tiger}})^\intercal \h + (\beta_{\textrm{Lion}} - \beta_{\textrm{Tiger}}) = 0$. In this setting, any successful classifier $\f$ requires a representation map $\g$ that linearly separates the classes in the latent space $\H$. We now consider the concept $c = \textrm{Stripes}$. In this case, we can expect examples from the classes tiger and zebra in $\Pos^{\textrm{Stripes}}$ (as both tigers and zebras have stripes) and examples from the classes lion and cow in $\Neg^{\textrm{Stripes}}$ (as neither lions nor cows have stripes). As illustrated in Figure~\ref{fig:concept_linear_separability}, it is therefore perfectly possible to have $\g(\Pos^{\textrm{Stripes}})$ and $\g(\Neg^{\textrm{Stripes}})$ that are not linearly separable in spite of the classes linear separability. With this example, we emphasize the subtle distinction between classes and concept linear separability. While the former is expected and holds experimentally~\cite{Alain2016}, this is not the case for the latter.

\textbf{A better assumption.} Although the representations from Figure~\ref{fig:concept_linear_separability} do not linearly separate the concept sets, we note that concept positive and negative examples are scattered across distinct clusters. In this way, a model that produces these representations appears to make a difference between presence and absence of the concept. From this angle, we could consider that the concept is well encoded in the representation space geometry. To formalize this more general notion of concept set separability, we adapt the smoothness assumption originally formulated in semi-supervised learning~\cite{Chapelle2006}.

\begin{assumption}[Concept Smoothness] \label{assumption:smoothness}
	A concept $c \in [C]$ is encoded in the latent space $\H$ if $\H$ is \emph{smooth} with respect to the concept. This means that we can separate $\H = \H^c \bigsqcup \H^{\neg c}$ into a \emph{concept activation region} (CAR) $\H^c$ where the concept $c$ is mostly present (i.e. $\card{\g(\Pos^c) \bigcap \H^c} \gg \card{\g(\Neg^c) \bigcap \H^c}$) and a region $\H^{\neg c}$ where the concept $c$ is mostly absent (i.e. $\card{\g(\Neg^c) \bigcap \H^{\neg c}} \gg\card{\g(\Pos^c) \bigcap \H^{\neg c}}$). If two points $\h_1, \h_2 \in \H$ in a high-density region of the latent space are close to each other, then we should have $\h_1 , \h_2 \in \H^c$ or $\h_1 , \h_2 \in \H^{\neg c}$. 
\end{assumption}
\begin{remark}
	We can easily see that linear separability is trivially included in this assumption: it corresponds to the case where the CAR $\H^c$ and $\H^{\neg c}$ are separated by a hyperplane. Conversely, our assumption does not require the CAR $\H^c$ and $\H^{\neg c}$ to be separated by a hyperplane for a concept to be relevant, as illustrated in Figure~\ref{fig:concept_smoothness}. 
\end{remark}

\begin{figure} 
	%\vspace{-.8cm}
	\centering
	\includegraphics[width=.7\textwidth]{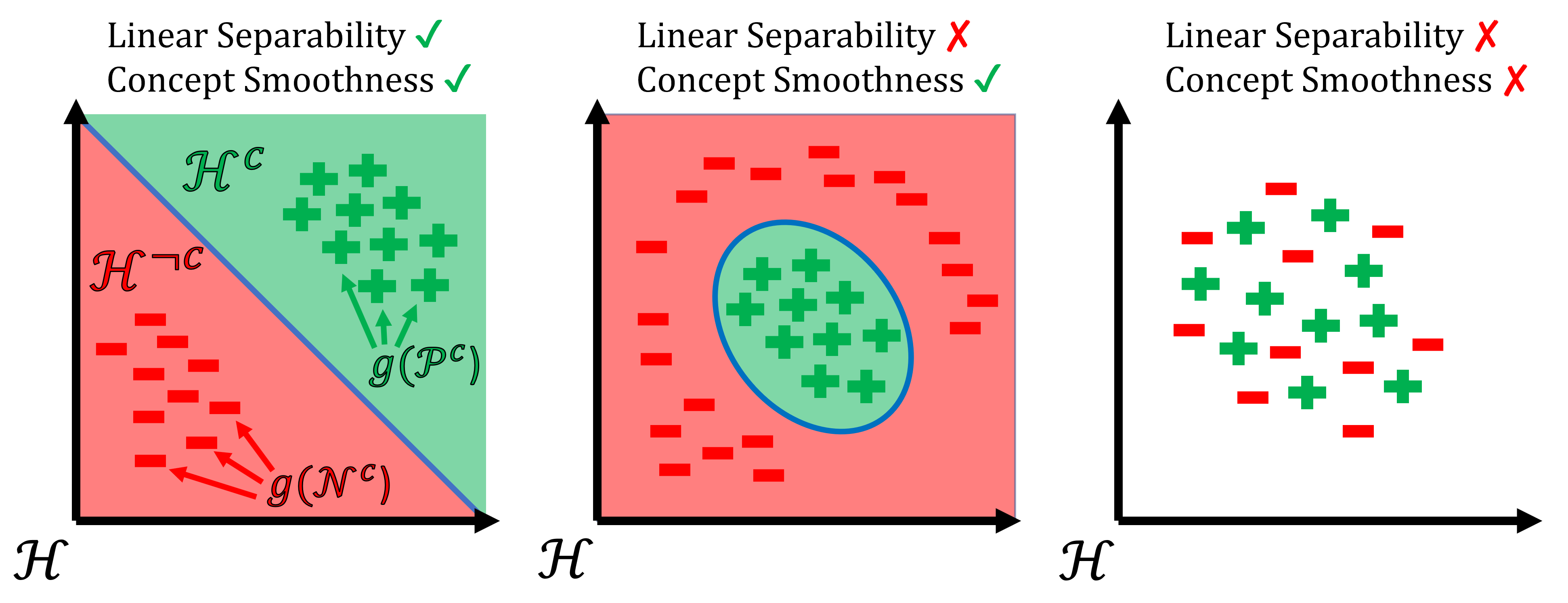}
	\caption{Linear separability implies concept smoothness but not the opposite.}
	\label{fig:concept_smoothness}
	%\vspace{-.3cm}
\end{figure}

There are two crucial components in the previous assumption that need to be detailed: what do we mean by \emph{density} and how do we extract a CAR $\H^c$ from $\H$. 

\subsection{Detecting Concepts} \label{subsec:car_concept_detection}
\textbf{Concept Density.} Let us start by detailing the notion of density that we use. We first note that the notion of proximity in $\H$ is crucial in Assumption~\ref{assumption:smoothness}. It is conventionally formalized through a kernel function $\kappa : \H^2 \rightarrow \R^+$~\cite{Hofman2008}. These functions are such that the proximity between $\h_1, \h_2 \in \H$ increases with $\kappa(\h_1, \h_2)$. Similarly, the proximity between $\h$ and a discrete set $\S \subset \H$ of representations increases with $\sum_{\h' \in \S} \kappa(\h, \h')$. This last sum can be interpreted as the density of examples from $\S$ at a point $\h \in \H$. In our setup, it is natural to define a density relative to each concept $c \in [C]$ by using the two sets $\Pos^c$ and $\Neg^c$. This motivates the following definition.

\begin{definition}[Concept Density] \label{def:concept_density}
	Let $\kappa : \H^2 \rightarrow \R^+$ be a kernel function. For each concept $c \in [C]$, we assume that we have a positive set $\Pos^c = \left\{ \x^{c, n} \mid n \in [N^{c}] \right\}$ and a negative set $\Neg^c = \left\{ \x^{\neg c, n} \mid n \in [N^{c}] \right\}$. We define the concept density as a function $\rho^c : \H \rightarrow \R$ such that
	\begin{align*}
		&\rho^c(\h) = \rho^{\Pos^c}(\h) - \rho^{\Neg^c}(\h), \\
		&\rho^{\Pos^c}(\h) = \frac{1}{N^c} \sum_{n=1}^{N^c} \kappa\left[ \h , \g(\x^{c, n})\right], \hspace{.5cm} \rho^{\Neg^c}(\h) = \frac{1}{N^c} \sum_{n=1}^{N^c} \kappa\left[ \h , \g(\x^{\neg c, n})\right].
	\end{align*}
	The concept density for an example $\x \in \X$ can similarly be defined as $\rho^c[\g(\x)]$.
\end{definition}

\begin{remark}
	This density function is not necessarily positive. Indeed, we have assigned a positive contribution for examples from $\Pos^c$ and a negative one for those of $\Neg^c$. The idea is that $\rho^c(\h) > 0$ whenever the density of $\g(\Pos^c)$ is higher around $\h$. Conversely, $\rho^c(\h) < 0$ whenever the density of $\g(\Neg^c)$ is higher around $\h$. Finally,  $\rho^c(\h) \approx 0$ if $\h$ is isolated from positive and negative examples or if the density of  $\g(\Pos^c)$ balances the density of $\g(\Neg^c)$ around $\h$.     
\end{remark}

\textbf{Concept Activation Regions.} With this definition, it is tempting to use the density $\rho^c$ to define the regions $\H^c$ and $\H^{\neg c}$. A natural choice would be to define the CAR as the positive density region $\H^c = (\rho^c)^{-1}(\R^+)$ and its complementary as the negative density region $\H^{\neg c} = (\rho^c)^{-1}(\R^-)$. This corresponds to using a Parzen window classifier for the concept~\cite{Parzen1962}. An obvious limitation of this approach is that each evaluation of the density $\rho^c$ scales linearly with the number $N^c$ of concept examples. If the size of concept sets is large, it is possible to obtain a sparse version of this Parzen window classifier with a \emph{support vector classifier} (SVC)~\cite{Benett1992, Cortes1995} $s^c_{\kappa} : \H \rightarrow \{0,1\}$. This SVC is fitted to discriminate the concept sets $\g(\Pos^c)$ and $\g(\Neg^c)$. We can then define the CARs as $\H^c = (s^c_{\kappa})^{-1}(1)$ and $\H^{\neg c} = (s^c_{\kappa})^{-1}(0)$. More details can be found in Appendix~\ref{appendix:car_classifiers}.

\textbf{Global explanations.} Our CAR formalism permits to extend those local (i.e. sample-wise) considerations globally. Let us start by defining the equivalent of TCAV scores described in Section~\ref{subsec:car_prelim}. This score aims at understanding how the model relates classes with concepts. We define the TCAR score as the fraction of examples that have class $k$ and whose representation lies in the CAR $\H^c$: $\TCAR^c_k = \nicefrac{\card{\g(\D_k) \bigcap \H^c }}{\card{\D_k}}$. Note that $\TCAR^c_k \in [0,1]$, where $0$ corresponds to no overlap and $1$ to a full overlap. In Appendix~\ref{appendix:tcar}, we extend this approach to measure the overlap between two concepts.

\textbf{Latent space isometries invariance.} In many applications such as clustering~\cite{Rokach2005} or data visualization~\cite{Vandermaaten2008}, the only relevant geometrical information of the representation space $\H$ is the distance $\norm[\H]{\h_1 - \h_2}$ between every pair of points $(\h_1, \h_2) \in \H^2$. Informally, we say that two representation spaces are isometric if they assign the same distance to each pair of points. In the aforementioned applications, two isometric representation spaces are therefore indistinguishable from one another. Since concept-based explanations similarly describe the representation space geometry, one might require similar invariance to hold in this context. In Appendix~\ref{appendix:isometry_invariance}, we show that our CAR formalism provides such guarantee if $\kappa$ is a radial kernel~\cite{Berlinet2011}. To the best of our knowledge, this type of analysis has not been performed in the context of concept-based explanation methods. We believe that future works in this domain would greatly benefit from this type of insight.

\subsection{Concepts and Features} \label{subsec:car_concept_features}
Not all features are relevant to identify a concept. As an example, let us consider the concept $c = \textrm{Stripes}$ in a computer vision setting. If the image $\x$ represents a zebra, only some small portion of the image will exhibit stripes (namely the body of the zebra). Therefore, if we build a saliency map for the concept $c$, we expect only this part of the image to be relevant for the identification of this concept. Existing works to produce concept-level saliency maps relying on the CAV formalism exist in the literature~\cite{Zhou2018, Brocki2019}. In the absence of CAVs, we need an alternative approach. We now describe how any feature importance method can be used in conjunction with CARs. A generic feature importance method assigns a score $a_i(f, \x) \in \R$ to each feature $i \in [d_X]$ for a \emph{scalar} model $f: \X \rightarrow \R$ to make a prediction $f(\x)$~\cite{Covert2020}. Since our purpose is to measure the relevance of each feature in identifying a concept $c \in [C]$, we compute the feature importance for the concept density: $a_i(\rho^c \circ \g, \x)$. Note that some specific feature importance methods, such as Integrated Gradients~\cite{Sundararajan2017} or Gradient Shap~\cite{Lundberg2017}, require the explained model to be differentiable. This is the case whenever the kernel $\kappa$ and the latent representation $\g$ are differentiable. We note that the support vector classifier $s_{\kappa}^c$ cannot be used in this context due to its non-differentiability. In Appendix~\ref{appendix:car_feature_importance}, we show that our CAR-based feature importance can be endowed with a useful completeness propriety that relates the importance scores $a_i(\rho^c \circ \g, \x)$ to the concept density $\rho^c \circ \g(\x)$.

	\section{Experiments} \label{sec:experiments}  
	The code to reproduce all the experiments from this section is available at \url{https://github.com/JonathanCrabbe/CARs} and  \url{https://github.com/vanderschaarlab/CARs}.

\subsection{Empirical Evaluation}
Our purpose is to empirically validate the formalism described in the previous section. We have several independent components to evaluate: \circled{1}~the concept classifier used to detect the CARs $\H^c$, \circled{2}~the global explanations induced by the TCAR values and \circled{3}~the feature importance scores induced by the concept densities $\rho^c$.

\textbf{Datasets.} We perform our experiments on 3 datasets. \circled{1}~The MNIST dataset~\cite{LeCun1998} consists of $28 \times 28$ grayscale images, each representing a digit. We train a convolutional neural network (CNN) with 2 layers to identify the digit of each image. \circled{2}~The MIT-BIH Electrocardiogram (ECG) dataset~\cite{Goldberger2000, Moody2001} consists of univariate time series with $187$ time steps, each representing a heartbeat cycle. We train a CNN with 3 layers to determine whether each heartbeat is normal or abnormal.  \circled{3}~The Caltech-UCSD Birds-200 (CUB) dataset~\cite{Branson2011} consists of coloured images of various sizes, each representing a bird from one of the $200$ species present in the dataset. We fine-tune an Inceptionv3 neural network~\cite{Szegedy2015} to identify the species each bird belongs to among the $200$ possible choices.

\textbf{Concepts.} For each dataset, we study the models through the lens of several well-defined concepts that are provided by human annotations. \circled{1}~For MNIST: we use $C=4$ concepts that correspond to simple geometrical attributes of the images: \emph{Loop} (positive images include a loop), \emph{Vertical/Horizontal Line} (positive images include a vertical/horizontal line) and \emph{Curvature} (positive images contain segments that are not straight lines). \circled{2}~For ECG: we use $C=4$ concepts defined by cardiologists to better characterize abnormal heartbeats~\cite{Kachuee2018}: \emph{Premature Ventricular}, \emph{Supraventricular}, \emph{Fusion Beats} and \emph{Unknown}. The exact definition for each of these concepts is beyond the scope of this paper. We simply note that annotations for those concepts are  available in the ECG dataset.  \circled{3} For CUB: we use $C = 112$ concepts that correspond to visual attributes of the birds (e.g. their size, the colour of their wings, etc.). We use the same procedure as~\cite{Koh2020} to extract those concepts from the CUB dataset. We stress that, in each case, we selected concepts that can be unambiguously associated to the classes that are predicted by the models. Hence, it is reasonable to expect those concepts to be salient for the models. In each case, we sample the positive and negative sets $\Pos^c$ and $\Neg^c$ from the model's training sets. For more details on the concepts and the models, please refer to Appendix~\ref{appendix:empirical_evaluation}.

\subsubsection{Accuracy of concept activation regions} \label{subsec:car_acc_validation}

\begin{figure}
	%\vspace{-.8cm}
	\centering
	\begin{subfigure}[b]{0.32\textwidth} 
		\centering
		\includegraphics[width=\textwidth]{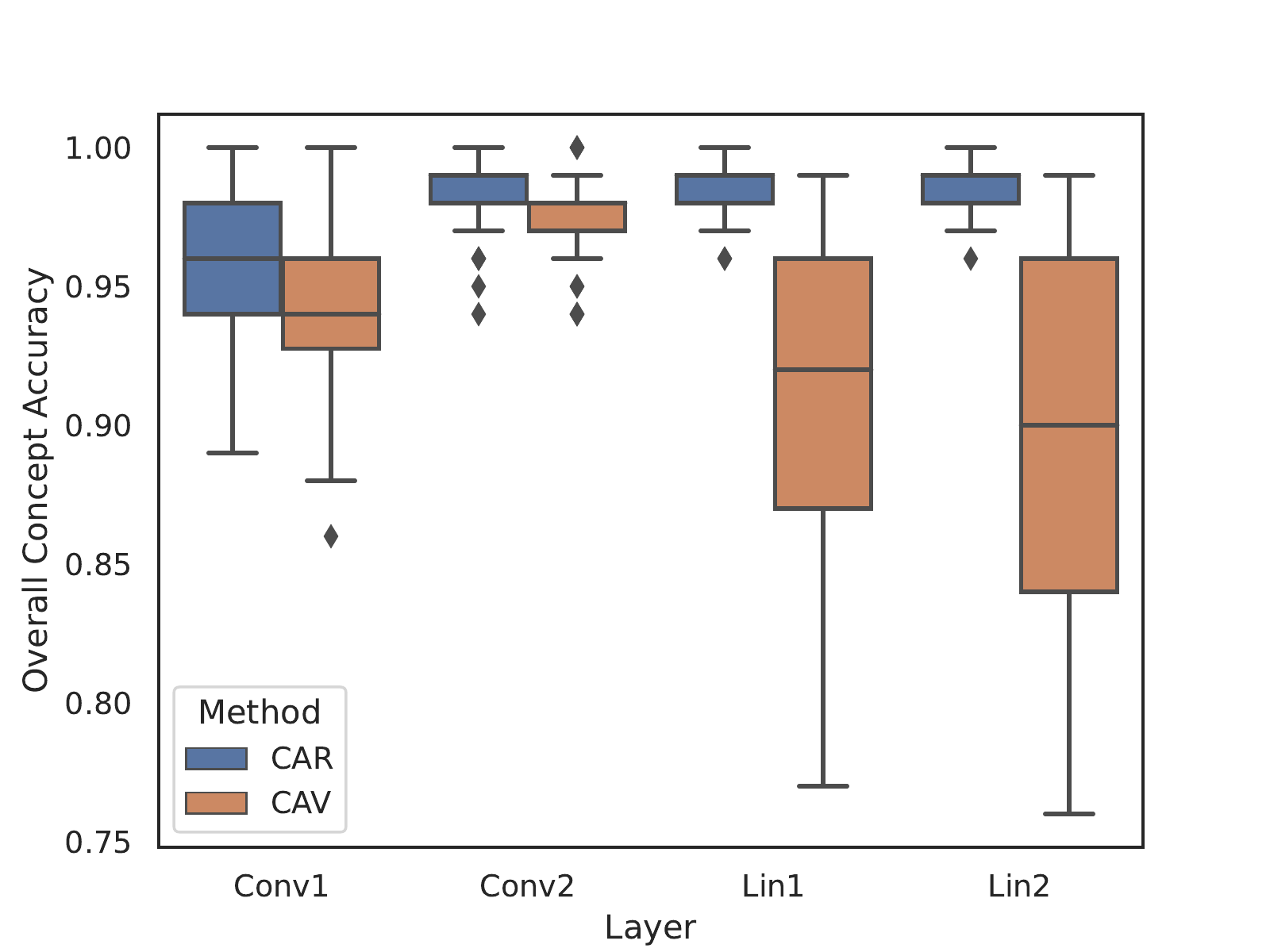}
		\caption{MNIST}
	\end{subfigure}
	\begin{subfigure}[b]{0.32\textwidth} 
		\centering
		\includegraphics[width=\textwidth]{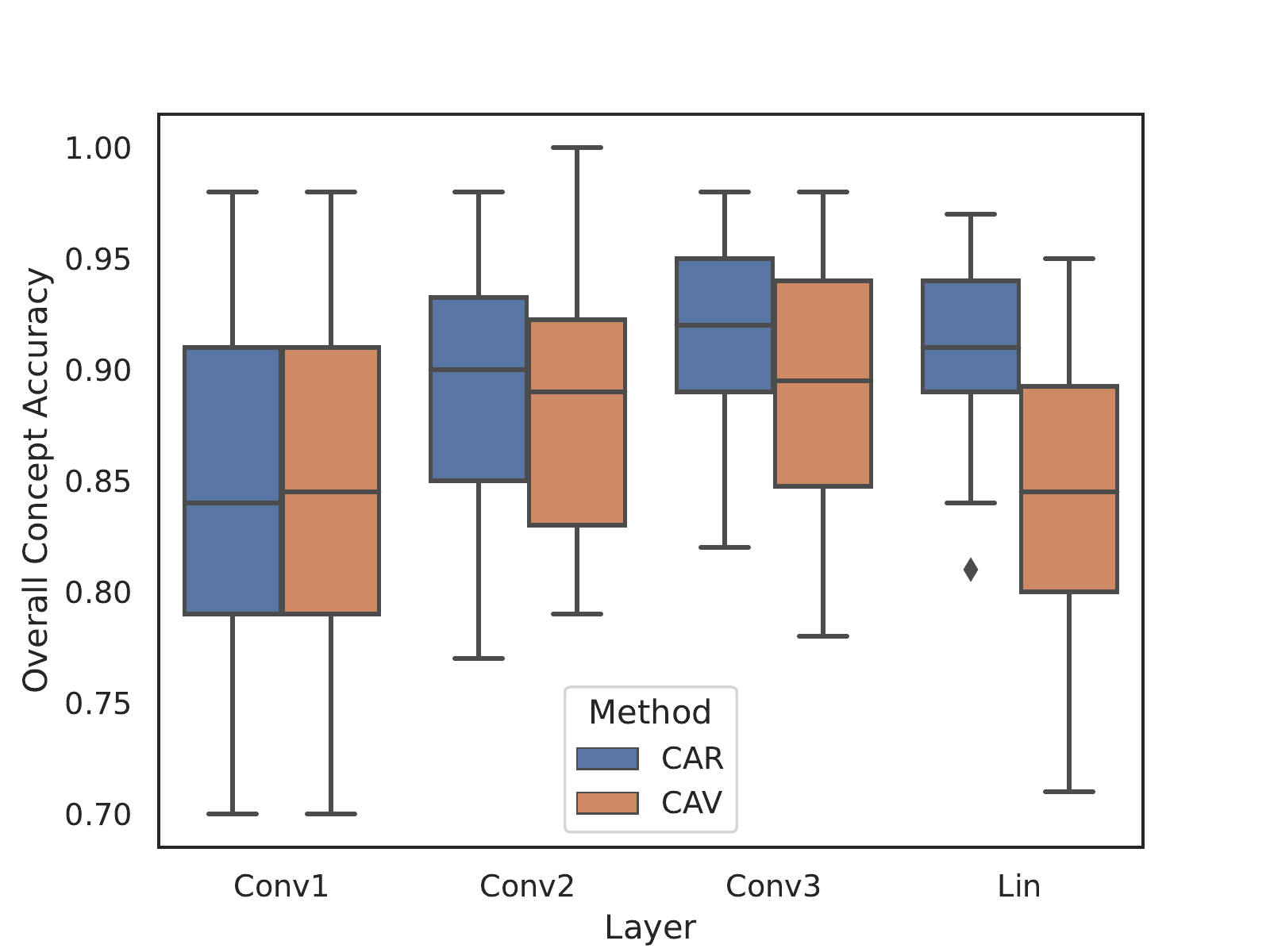}
		\caption{ECG}
	\end{subfigure}
	\begin{subfigure}[b]{0.32\textwidth} 
		\centering
		\includegraphics[width=\textwidth]{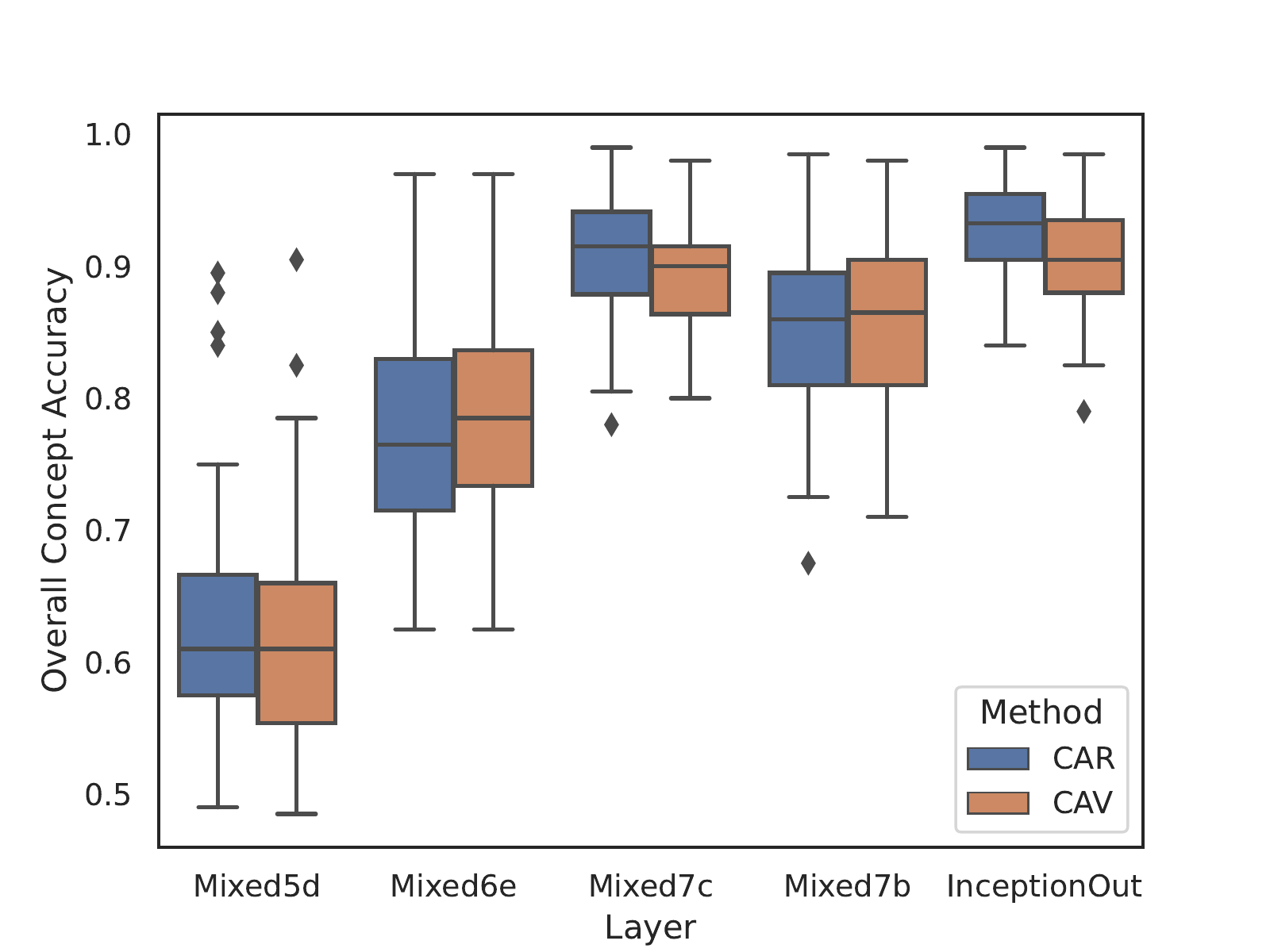}
		\caption{CUB}
		\label{fig:overall_concept_acc:cub}
	\end{subfigure}
	\caption{Overall accuracy of concept classifiers.}
	\label{fig:overall_concept_acc}
\end{figure}

\textbf{Methodology.} The purpose of this experiment is to assess if the concept regions $\H^c$ identified by our CAR classifier generalize well to unseen examples. Each of the models described above are endowed with several representation spaces (one per hidden layer). For several of those latent spaces, we fit our CAR classifier (SVC with radial basis function kernel) to discriminate the concept sets $\Pos^c , \Neg^c $ for each concept $c \in [C]$. These two sets have a size $N^c = 200$ and are sampled from the model's training set. The classifier is then evaluated by computing its accuracy on a holdout balanced concept set $\T^c$ of size 100 sampled from the model's testing set. For MNIST and ECG, we repeat this experiment 10 times for each concept and let the sets $\Pos^c , \Neg^c, \T^c$ vary on each run. For comparison, we perform the same experiment with a linear CAV classifier as a benchmark. We report the overall (all the concepts together) accuracy in Figure~\ref{fig:overall_concept_acc}.

 \textbf{Analysis.}  The CAR classifier substantially outperforms the CAV classifier. Note that this advantage is even more striking in the representation spaces associated to the deeper (last) DNN layers. This can be better understood through the lens of Cover's theorem~\cite{Cover1965}: the linear separation underlying CAV is usually easier to achieve in higher dimensional spaces, which corresponds to the shallower (first) DNN layers in this case. When the dimension $d_H$ of the latent space becomes comparable to the size $N^c$ of the concept sets, linear separation often fails to maintain a high accuracy. By contrast, the SVC underlying CARs manage to maintain high accuracy through the more flexible notion of concept smoothness defined in Assumption~\ref{assumption:smoothness}. This suggests that concepts can be well encoded in the geometry of the latent space even when accurate linear separability is not possible.  We also note that the accuracy CAR classifiers seems to increase with the representation's depth. This is consistent with the behaviour of class probes~\cite{Alain2016}. 

\textbf{Statistical significance.} The statistical significance of the concept classifiers is evaluated with the permutation test from~\cite{Schrouff2021}. All of them are statistically significant with p-value $ < .05$ except for some concepts classifiers that are fitted with the layers \emph{Mixed5d} and \emph{Mixed6e} of the CUB Inceptionv3 model. We note that those classifiers do not generalize well in Figure~\ref{fig:overall_concept_acc:cub}. This suggests that deeper networks are required to identify more challenging concepts correctly.  

\takeaway{1} CAR classifiers better capture how concepts are spread across representation spaces.

\subsubsection{Consistency of global explanations} \label{subsec:car_global_validation}

\begin{figure}
	%\vspace{-.3cm}
	\centering
	\begin{subfigure}[b]{0.32\textwidth} 
		\centering
		\includegraphics[width=\textwidth]{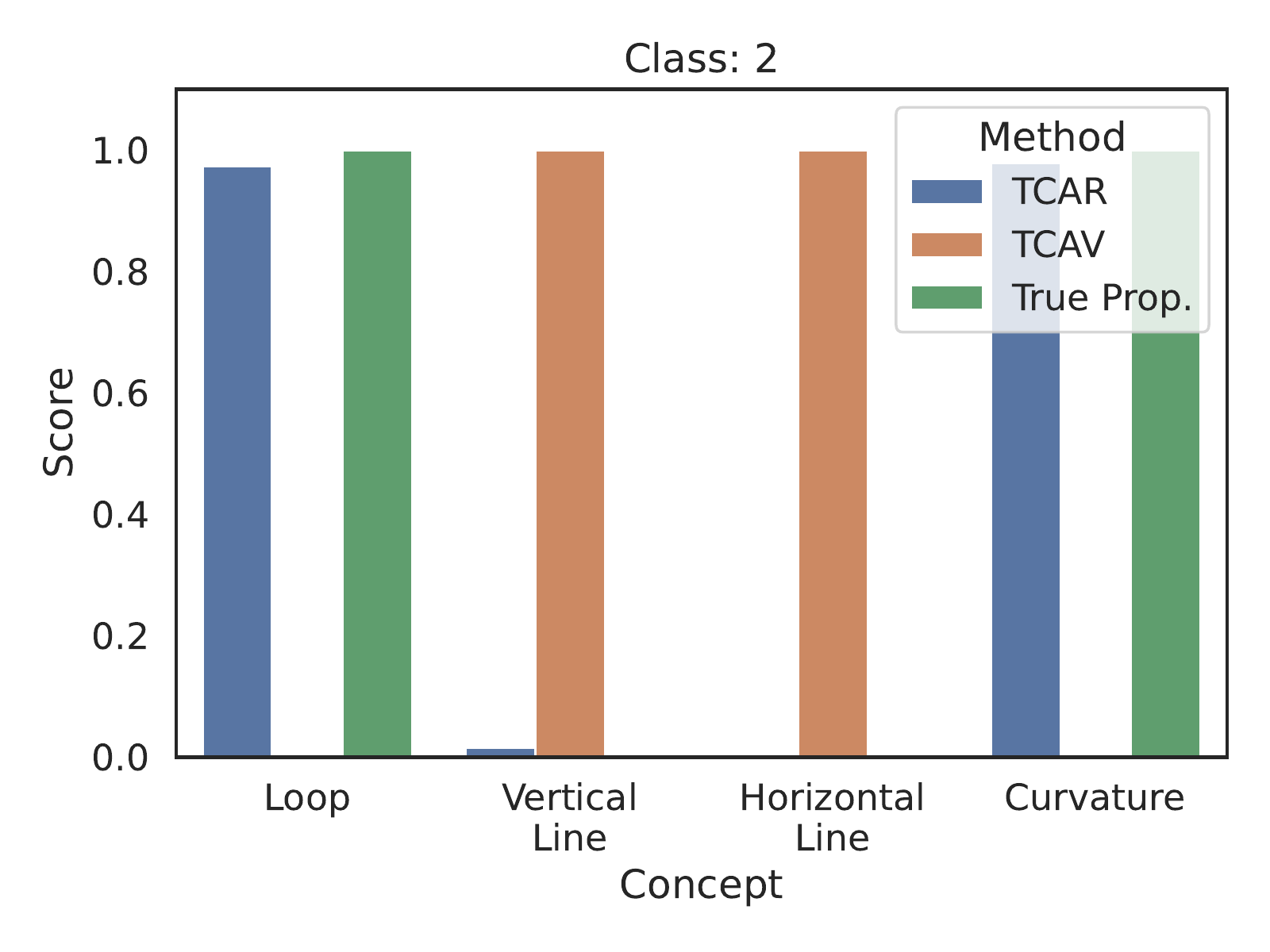}
		\caption{MNIST Digit 2}
	\end{subfigure}	
	\begin{subfigure}[b]{0.32\textwidth} 
		\centering
		\includegraphics[width=\textwidth]{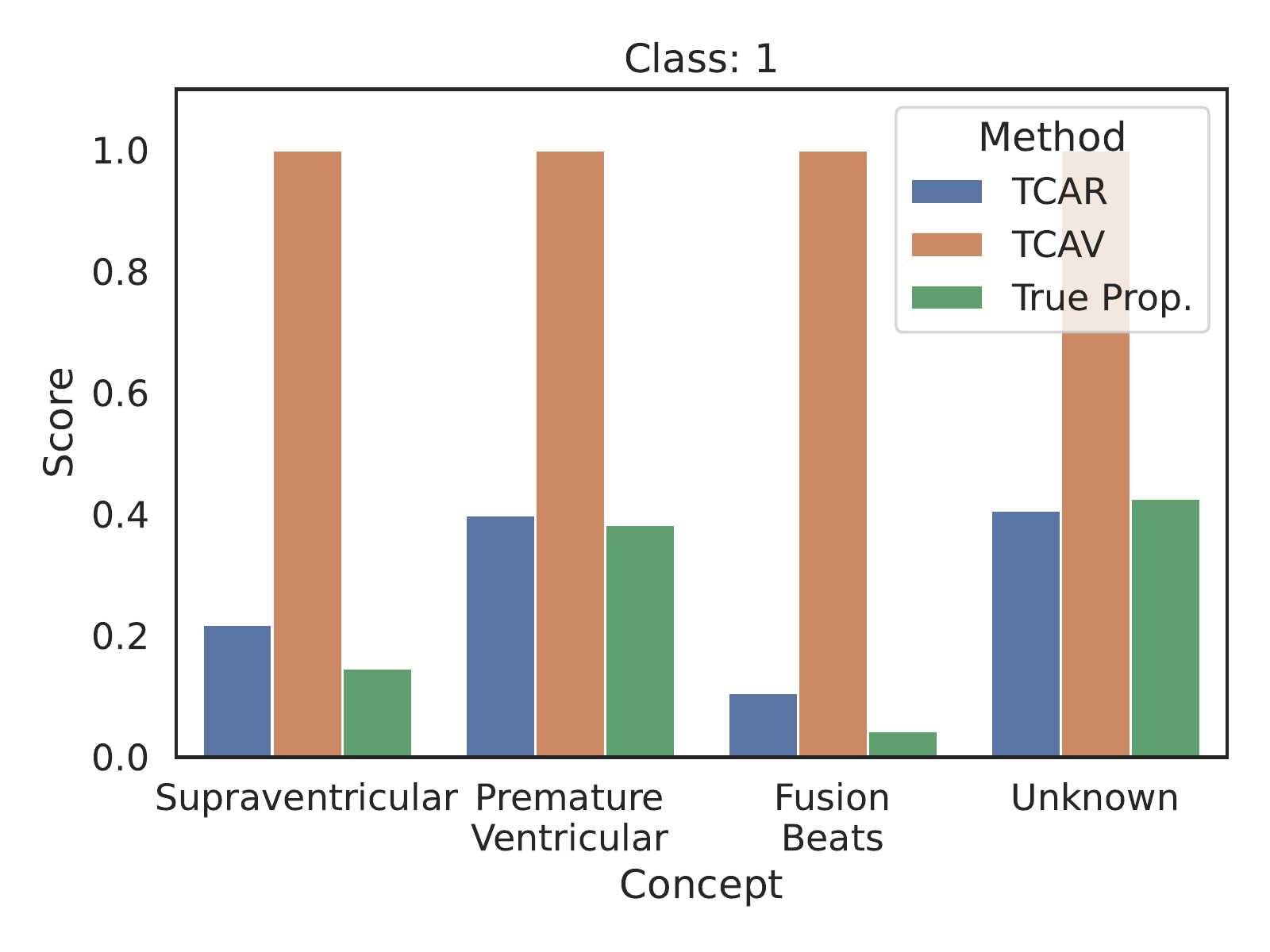}
		\caption{ECG Abnormal Heartbeat}
	\end{subfigure}
	\begin{subfigure}[b]{0.32\textwidth} 
		\centering
		\includegraphics[width=\textwidth]{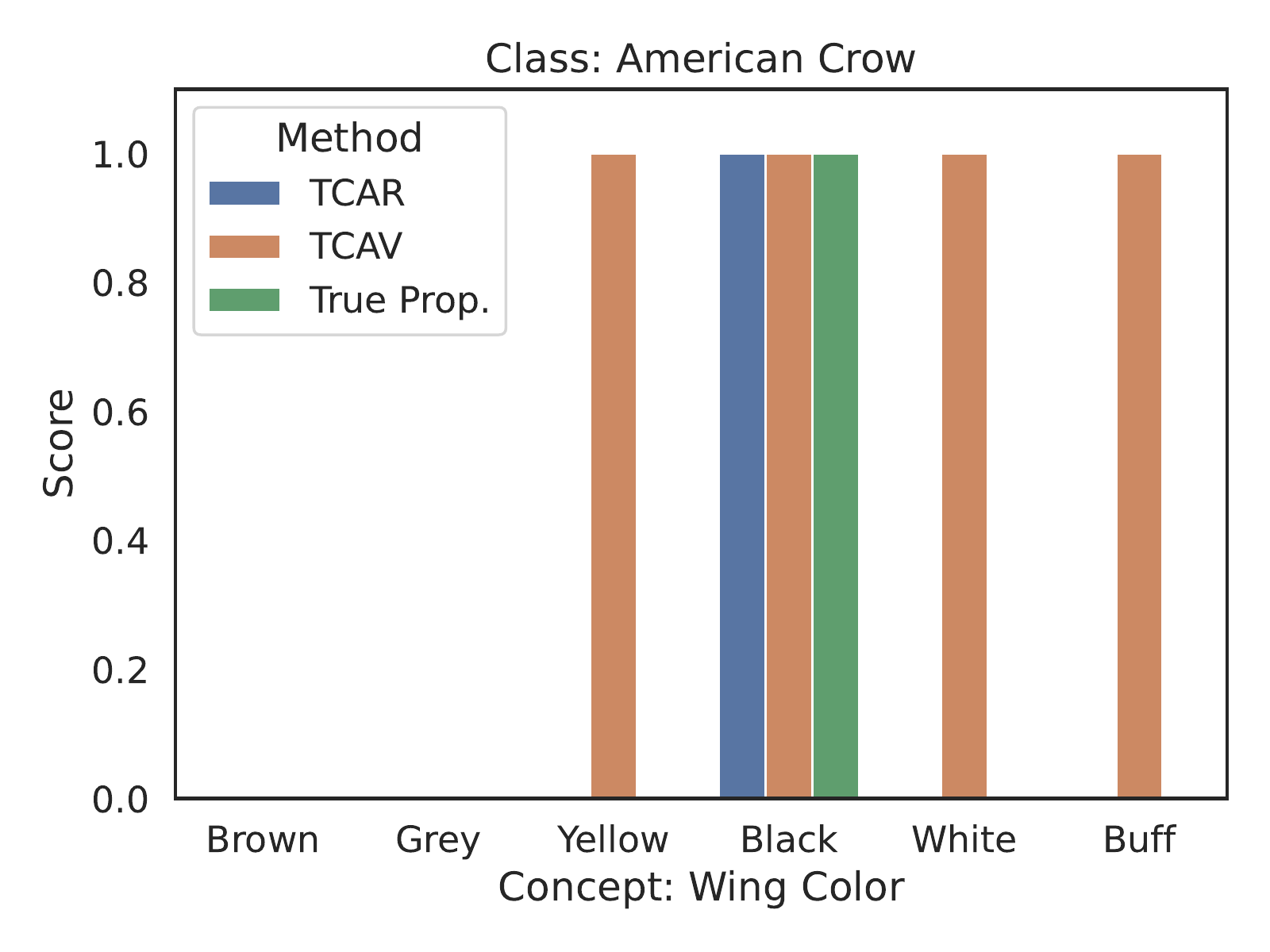}
		\caption{CUB American Crow}
	\end{subfigure}
	\caption{Examples of global concept-based explanations.}
	\label{fig:global_explanation}
	%\vspace{-.5cm}
\end{figure}

\textbf{Methodology.} The purpose of this experiment is to assess if the aforementioned concept classifiers create meaningful global association between classes and concepts. For each model, we now focus our study on the penultimate layer.  We compute the TCAV and TCAR scores for each $(\mathrm{class}, \mathrm{concept})$ pair over the whole testing set. Ideally, these scores should be correlated to the true proportion (computed with the human concept annotations) of examples from the class that exhibit the concept. To that aim, we report the Pearson correlation $r$ between each score and this true proportion in Table~\ref{tab:global_explanations}. To illustrate those global explanations, we also provide the scores for one class of each dataset in Figure~\ref{fig:global_explanation} (for CUB, we restrict to wing colour concepts, other examples are reported in Appendix~\ref{appendix:further_examples}).

\begin{wraptable}{r}{.5\textwidth}
	\caption{Evaluation of global explanations.}
	\label{tab:global_explanations}
	\centering
	\resizebox{.5\textwidth}{!}{
	\begin{tabular}{ccc}
		\toprule
		Dataset     & $r(\mathrm{TCAR}, \mathrm{True Prop.})$     & $r(\mathrm{TCAV}, \mathrm{True Prop.})$ \\
		\midrule
		MNIST & \textbf{1.00} & .60 \\
		ECG & \textbf{.99} & .74 \\
		CUB & \textbf{.86} & .52 \\ 
		\bottomrule
	\end{tabular}}
\end{wraptable}
\textbf{Analysis.} The TCAR scores better correlate with the true presence of concepts. This difference can be understood by looking at the examples from Figure~\ref{fig:global_explanation}. We note that TCAV scores tend to predict nonexistent associations (e.g. yellow wings for American crows) and miss existing associations (e.g. curvature for digit 2). For ECG, we note that TCAR does capture the fact that some concepts (like fusion beats) are less represented within the class, while TCAV does not. In all of these cases, TCAV explanations might give the impression that the model did not learn meaningful class-concept associations. The TCAR analysis leads to the opposite conclusion. Since we have established that TCAR is built upon more accurate concept classifiers, it seems that models indeed learn concepts as intended, in spite of what TCAV explanations suggest.  

\takeaway{2} TCAR scores more faithfully reflect the true association between classes and concepts.

\subsubsection{Coherency of concept-based feature importance} \label{subsec:car_feature_importance_validation}

\begin{figure}
	%\vspace{-.7cm}
	\centering
	\begin{subfigure}[b]{0.32\textwidth} 
		\centering
		\includegraphics[width=\textwidth]{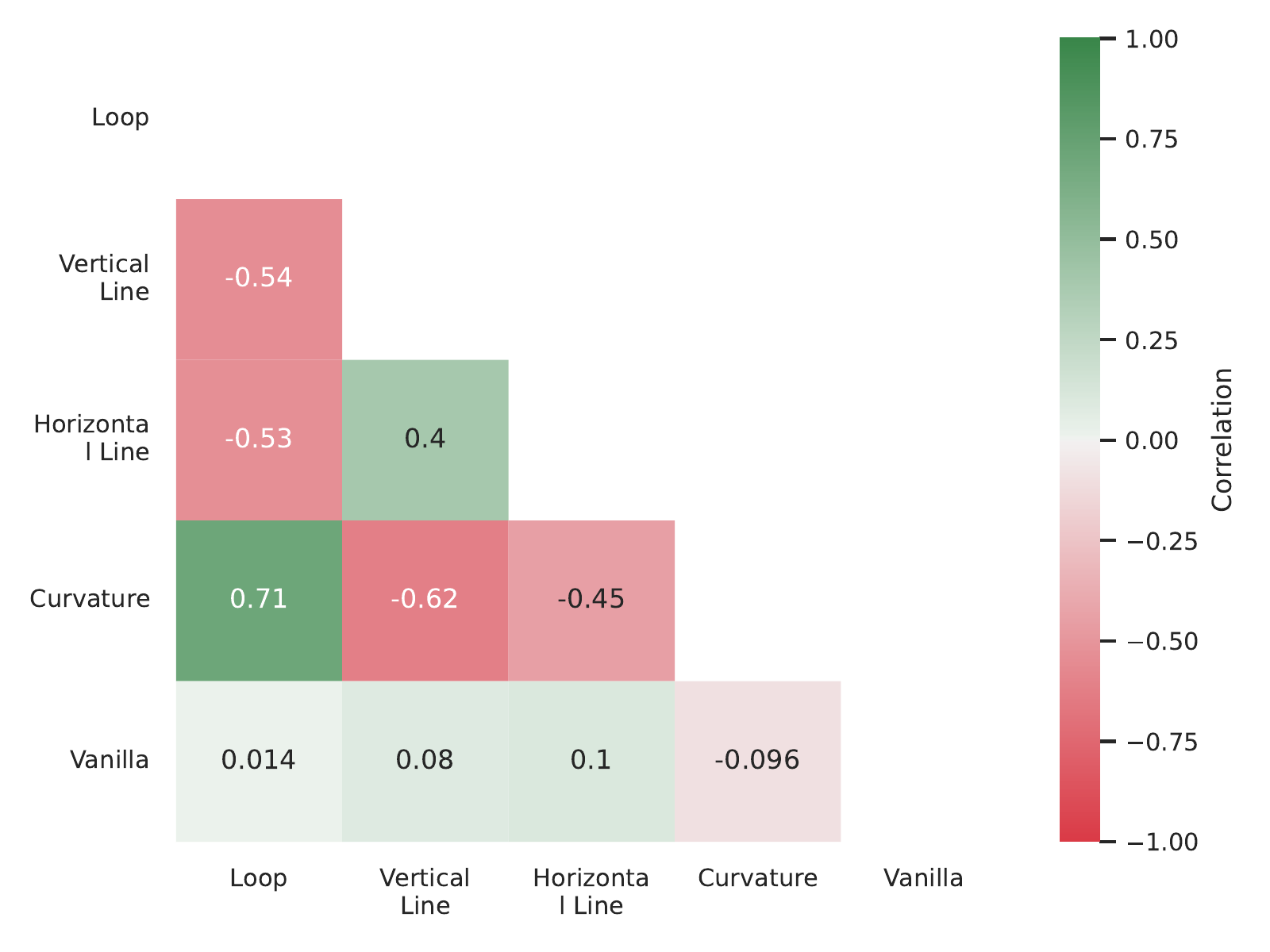}
		\caption{MNIST}
	\end{subfigure}	
	\begin{subfigure}[b]{0.32\textwidth} 
		\centering
		\includegraphics[width=\textwidth]{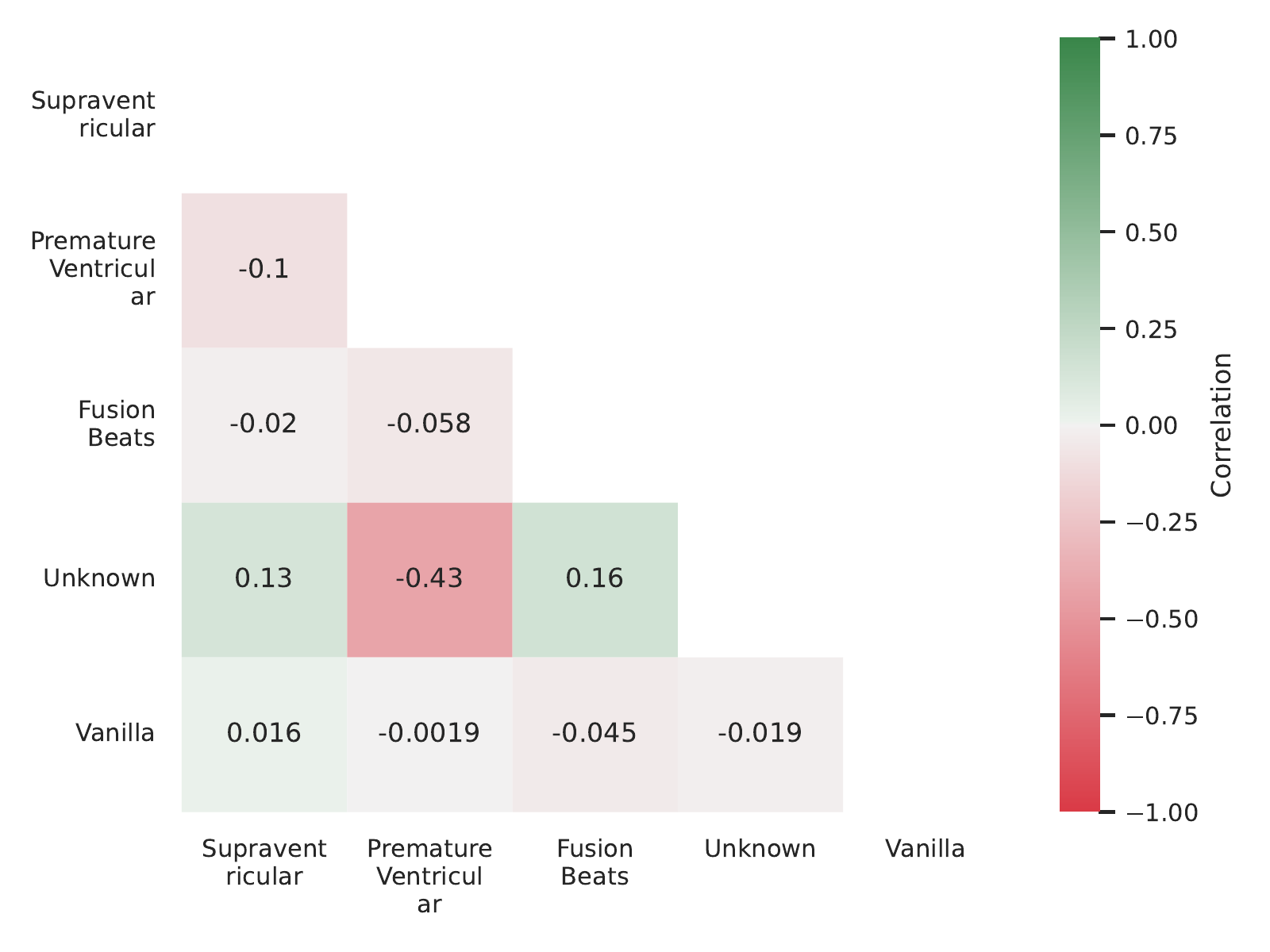}
		\caption{ECG}
	\end{subfigure}
	\begin{subfigure}[b]{0.32\textwidth} 
	\centering
	\includegraphics[width=\textwidth]{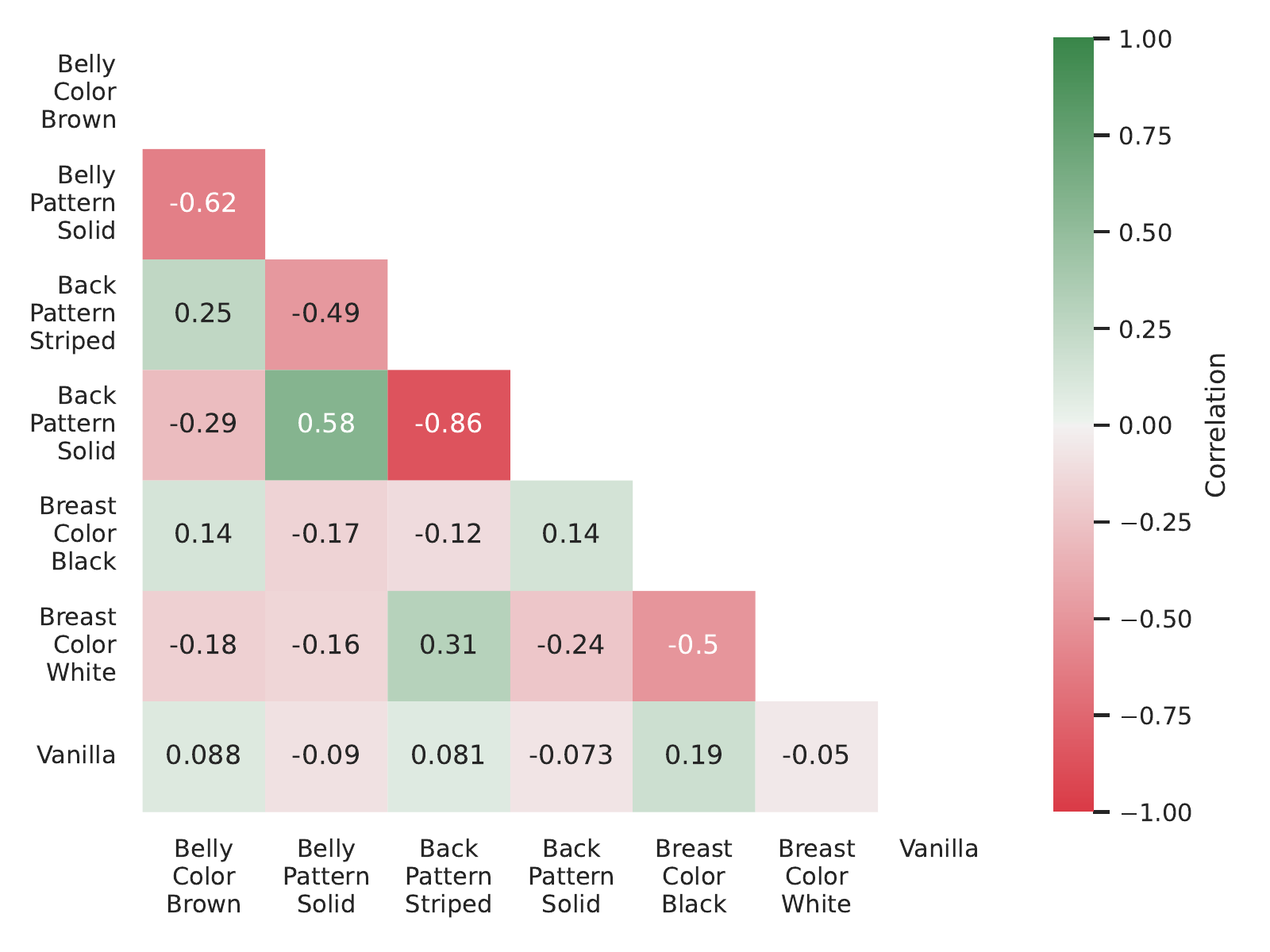}
	\caption{CUB}
	\end{subfigure}
	\caption{Correlations between concept feature importance.}
	\label{fig:attr_corr}
\end{figure}

\textbf{Methodology.} Focusing again on the model's penultimate layer, we are now interested in feature importance. The evaluation of feature importance methods is a notoriously difficult problem~\cite{Kindermans2019}. To perform an analysis in line with our concept-based explanations, we propose to check that our concept-based feature importance meaningfully captures features that are concept-specific. We analyse this through 2 desiderata: \circled{1} concept-based feature importance is not equivalent to the model (vanilla) feature importance and \circled{2} only concepts that can be identified with similar features should yield correlated feature importance. To evaluate these desiderata empirically, we compute our CAR-based Integrated Gradients $a_i(\rho^c \circ \g , \x)$ for each concept $c \in [C]$ and for each example $\x \in \D_{\mathrm{test}}$ from the test set. Similarly, we compute the vanilla Integrated Gradients $a_i(\f, \x)$ for the model~\footnote{The vanilla Integrated Gradients are computed for the model estimation of the true class probability.}. We quantitatively compare these various feature importance scores by computing their Pearson correlation $r$ as in~\cite{LeMeur2013, Crabbe2022}. We report the results in Figure~\ref{fig:attr_corr} (again, we selected a subset of 6 concepts for CUB).
 
\textbf{Analysis.} First, we observe that CAR-based feature importance correlates weakly with vanilla feature importance ($\card{r}<.25$ for all datasets). This confirms that desideratum~\circled{1} is fulfilled. Then, we note that most of the CAR-based feature importance scores are decorrelated or weakly correlated with each other. The counterexamples that we observe indeed correspond to concepts that can be identified with similar features. A first example is the positive correlation between the loop and the curvature concepts for MNIST ($r=.71$), both concepts are generally associated to curved symbols. Another example is the negative correlation between striped and solid back patterns for CUB birds ($r=-.86$), those concepts are mutually exclusive and identified by inspecting the bird's back. Those examples support that CAR-based feature importance fulfils desideratum~\circled{2}.

\takeaway{3} CAR-based feature importance is concept-specific and captures concept associations.

\subsection{Use Case: Machine Learning Model Rediscovering Known Medical Concepts} \label{subsec:seer_use_case}
We will now describe a use case of the CAR formalism introduced in this paper. We stress that the literature already contains numerous use cases of CAV concept-based explanations, especially in the medical setting~\cite{Graziani2019, Clough2019, Mincu2020}. Since CARs generalize CAVs, it goes without saying that they apply to these use cases. Rather than repeating existing usage of concept-based explanations, we discuss an alternative use case motivated by recent trends in machine learning. With the successes of deep models in various scientific domains~\cite{Jumper2021, Davies2021, Kirkpatrick2021, Degrave2022}, we witness an increasing overlap between scientific discovery and machine learning. While this new trend opens up fascinating opportunities, it comes with a set of new challenges. The evaluation of machine learning models is arguably one of the most important of these challenges. In a scientific context, the canonical machine learning approach to validate models (out-of-sample generalization) is likely to be insufficient. Beyond generalization on unseen data, the scientific validity of a model requires consistency with established scientific knowledge~\cite{Roscher2020}. We propose to illustrate how our CARs can be used in this context.

\textbf{Dataset.} We use the data collected with the Surveillance, Epidemiology, and End Results (SEER) Program. The dataset~\cite{Adamo2017} contains a US population-based cohort of 171,942 men diagnosed with non-metastatic prostate cancer between Jan 1, 2000, and Dec 31, 2016. Each patient is described by age, the results of a prostate-specific antigen blood test (PSA), the clinical stage of its tumour and two Gleason scores (primary and secondary). Each patient also has a label that indicates if they died because of their prostate cancer. We train a multilayer perceptron (MLP) to predict the patient's mortality on $90 \%$ of the data and test on the remaining $10 \%$. For a more detailed description of the data and the model, please refer to Appendix \ref{appendix:use_case}.
 
\textbf{Concepts.} Doctors use an established grading system to predict how likely the cancer is to spread~\cite{Gordetsky2016}. This system assigns to each patient a grade between 1 and 5. The probability that the cancer spreads increases with this grade. It can be computed from the two Gleason scores (more details in Appendix~\ref{appendix:use_case}). We can consider each of these 5 grades as a concept. We note that this grade is not explicitly part of the input features of the MLP that we trained. Our purpose is to assess if the MLP implicitly discovered those grades in order to predict the patient's mortality. If this happens to be the case, this would demonstrate that the MLP is in-line with existing medical knowledge.  

\begin{figure}
	\centering
	\begin{subfigure}[b]{0.49\textwidth} 
		\centering
		\includegraphics[width=\textwidth]{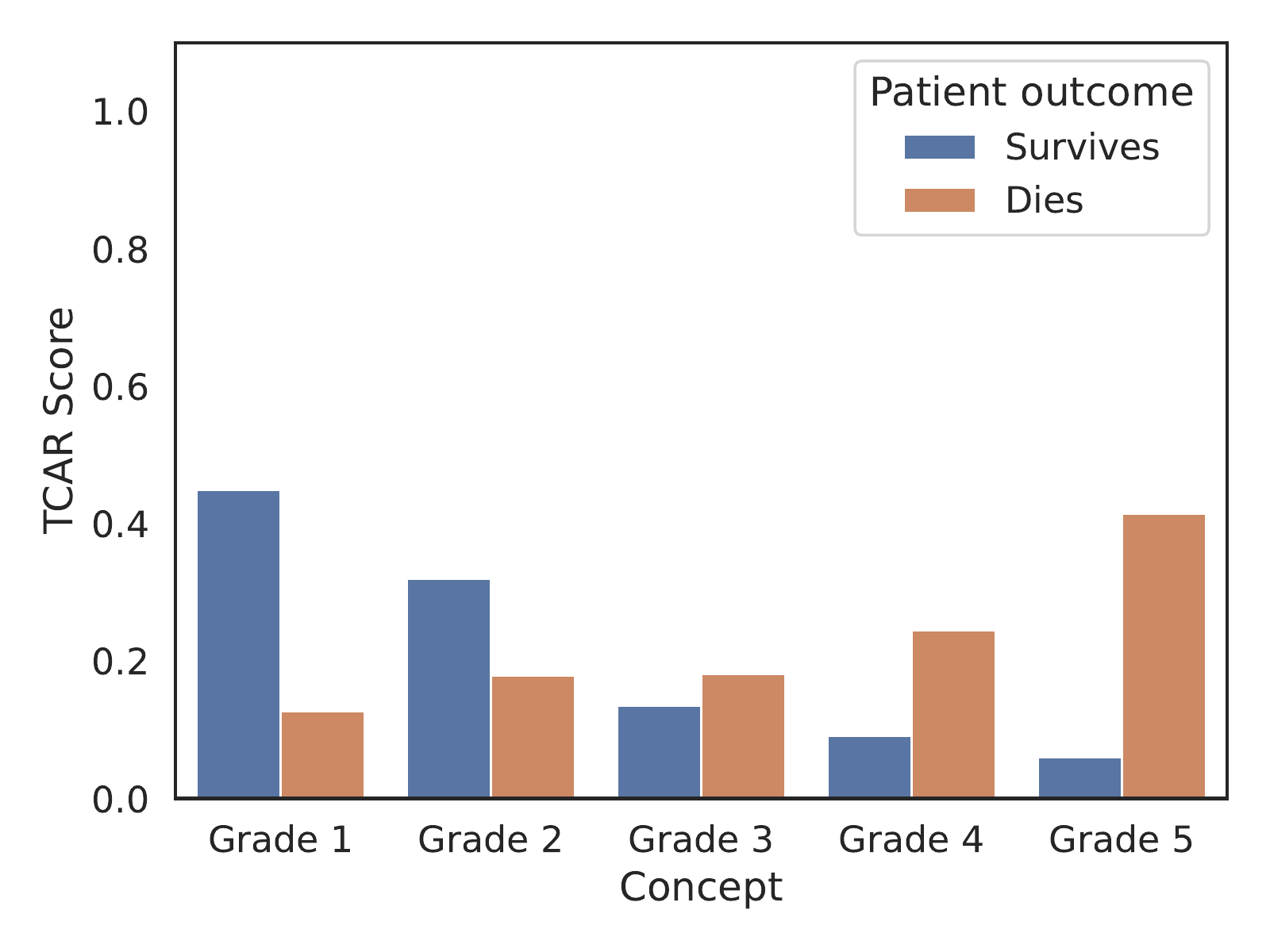}
		\caption{TCAR grades vs mortality.} \label{fig:seer_global}
	\end{subfigure}	
	\begin{subfigure}[b]{0.49\textwidth} 
		\centering
		\includegraphics[width=\textwidth]{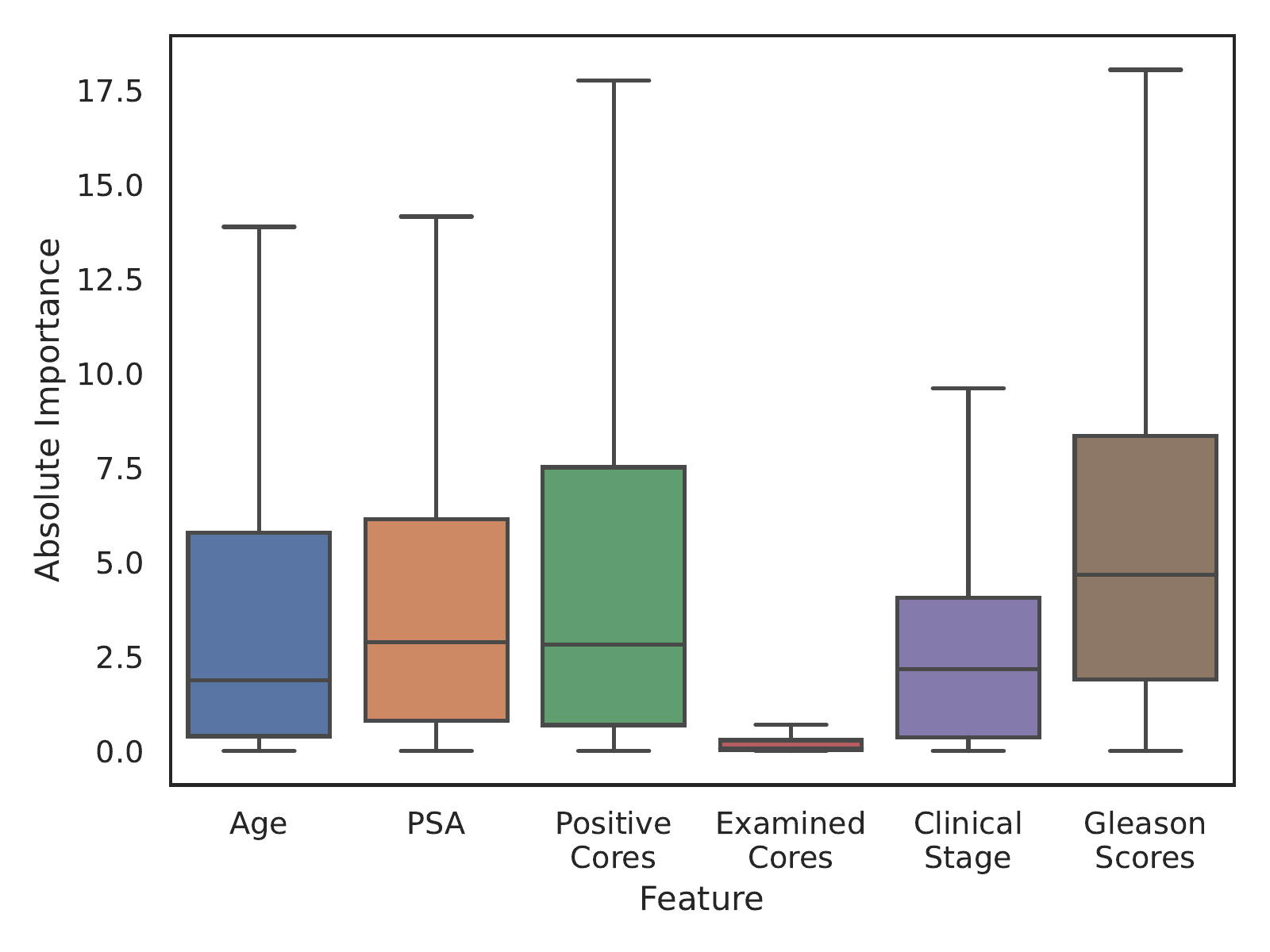}
		\caption{Grades feature importance.} \label{fig:seer_feature_importance}
	\end{subfigure}
	\caption{CAR explanations with the prostate cancer grades as concepts.}
\end{figure}

\textbf{Methodology.} As in the previous section, we are going to use the output of the MLP's penultimate layer as a representation space. We fit a CAR classifier (SVC with linear kernel) to discriminate the concepts sets $\Pos^c, \Neg^c$ for each grade $c \in [5]$. Both of those sets contain $N^c = 250$ patients sampled from the training set. We then proceed to several verifications. \circled{1} We determine if the grades have been learned by measuring the accuracy of each CAR classifier on a holdout balanced concept set $\T^c$ of size 100 sampled from the model's testing set. \circled{2} We check that each grade is associated to the appropriate outcome by reporting the TCAR scores between each pair $(\mathrm{mortality}, \mathrm{grade})$ in Figure~\ref{fig:seer_global}. \circled{3} We verify that grades are identified with the right features by plotting a summary of the absolute feature importance scores $\card{a_i(\rho^c \circ \g, \x)}$ on the test set $\x \in \D_{\mathrm{test}}$ in Figure~\ref{fig:seer_feature_importance}.

\textbf{Analysis.} Let us summarize the findings for each of the above points. \circled{1}~All of the CAR classifiers generalize well on the test set (their accuracy ranges from $90\%$ to $100\%$). This strongly suggests that the MLP implicitly separates patients with different grades in representation space. \circled{2}~Figure~\ref{fig:seer_global} demonstrates that the model associates higher grades with higher mortality. This is in line with the clinical interpretation of the grades. \circled{3}~Figure~\ref{fig:seer_feature_importance} suggests that the Gleason scores constitute the most important features overall (highest quantiles) for the model to discriminate between grades. This is consistent with the fact that grades are computed with the Gleason scores. We conclude that the MLP implicitly identifies the patient's grades with high accuracy and in a way that is consistent with the medical literature.  

\takeaway{4} The CAR formalism can reliably support scientific evaluation of a model.

	\section{Conclusion} \label{sec:conclusion}
	We introduced Concept Activation Regions, a new framework to
relax the linear separability assumption underlying the Concept Activation Vector formalism. We showed that our
framework guarantees crucial properties, such as invariance with respect to latent symmetries. Through extensive validation on several datasets, we verified that \circled{1}~Concept Activation Regions better capture the distribution of concepts across the model's representation space, \circled{2}~The resulting global explanations are more consistent with human annotations and \circled{3}~Concept Activation Regions permit to define concept-specific feature importance that is consistent with human intuition. Finally, through a use case involving prostate cancer data, we show the neural network can implicitly rediscover known scientific concepts, such as the prostate cancer grading system.

Many important points that were not covered in the main paper can be found in the appendices. In Appendix~\ref{appendix:car_classifiers}, we discuss how to tune the various hyperparameters of the CAR classifiers. In Appendix~\ref{appendix:tcar}, we explain how to generalize concept activation vectors to nonlinear decision boundaries. In Appendix~\ref
{appendix:robustness}, we show that the CAR explanations are robust to adversarial perturbations and background shifts. In Appendix~\ref{appendix:unsupervised_concepts}, we demonstrate that CAR explanations can be used to relate abstract concepts discovered by self-explaining neural networks with human concepts. Finally, Appendix~\ref{appendix:car_nlp} illustrates how CAR explanations allow us to probe language models. 
 
We believe that our extended concept explainability framework opens up many interesting avenues for future work. A first one would be to probe state of the art neural networks with an approach similar to Section~\ref{sec:experiments}. In particular, it would be interesting to analyse if improving model performance is associated with a better encoding of human concepts. A second one would be to analyze how concept discovery~\cite{Ghorbani2019} can benefit from our generalized notion of concept activation. A more fine-grained characterization of a model's latent space is likely to improve the surfaced concept. A third one, as suggested
by Section~\ref{subsec:seer_use_case}, would be to use concept-based explanations to make a better scientific assessment of neural networks. Indeed, consistency with well-established knowledge is crucial for a scientific model to be accepted.

	\begin{ack}
		The authors are grateful to Fergus Imrie, Yangming Li and
the 3 anonymous NeurIPS reviewers for their useful comments
on an earlier version of the manuscript. Jonathan Crabbé is
funded by Aviva and Mihaela van der Schaar by the Office
of Naval Research (ONR), NSF 172251.
	\end{ack}
	
	\nocite{*}
	\bibliographystyle{unsrt}
	\bibliography{main.bib}
	\clearpage
	
	%%%%%%%%%%%%%%%%%%%%%%%%%%%%%%%%%%%%%%%%%%%%%%%%%%%%%%%%%%%%
	\section*{Checklist}
	
	\begin{enumerate}

		\item For all authors...
		\begin{enumerate}
			\item Do the main claims made in the abstract and introduction accurately reflect the paper's contributions and scope?
			\answerYes{All the claims from the abstract and Section~\ref{sec:introduction} are verified empirically in Section~\ref{sec:experiments}. The theoretical claims are demonstrated in Appendices~\ref{appendix:car_feature_importance}~and~\ref{appendix:isometry_invariance}.}
			\item Did you describe the limitations of your work?
			\answerYes{The assumption for our concept-based explanations to be valid are clearly stated in Assumption~\ref{assumption:smoothness}. Our method only applies to neural networks, which is stated in Section~\ref{sec:car}.}
			\item Did you discuss any potential negative societal impacts of your work?
			\answerYes{Our work improves the transparency of deep neural networks. This has many beneficial societal impacts, as described in Section~\ref{sec:introduction}. We do not see any potential negative societal impact to this approach. We have carefully reviewed the points from the ethical guideline and none of the mentioned negative impact seems to apply to our work.}  
			\item Have you read the ethics review guidelines and ensured that your paper conforms to them?
			\answerYes{We have read the ethics review guidelines and confirm that our paper conforms to them.}
		\end{enumerate}

		\item If you are including theoretical results...
		\begin{enumerate}
			\item Did you state the full set of assumptions of all theoretical results?
			\answerYes{The assumptions in Propositions~\ref{prop:completeness}~and~\ref{prop:isometry_invariance} are clearly stated in Appendices~\ref{appendix:car_feature_importance}~and~\ref{appendix:isometry_invariance} respectively.}
			\item Did you include complete proofs of all theoretical results?
			\answerYes{The proofs for Propositions~\ref{prop:completeness}~and~\ref{prop:isometry_invariance} are given in Appendices~\ref{appendix:car_feature_importance}~and~\ref{appendix:isometry_invariance} respectively.}
		\end{enumerate}

		\item If you ran experiments...
		\begin{enumerate}
			\item Did you include the code, data, and instructions needed to reproduce the main experimental results (either in the supplemental material or as a URL)?
			\answerYes{The full code is available at \url{https://github.com/JonathanCrabbe/CARs} and  \url{https://github.com/vanderschaarlab/CARs}. The implementation of our method closely follows Algorithms~\ref{alg:car_classifier},~\ref{alg:tcar_class}~and~\ref{alg:car_feature_importance} in the appendices. All the details to reproduce the experimental results are given in Appendices~\ref{appendix:empirical_evaluation}~and~\ref{appendix:use_case}.}
			\item Did you specify all the training details (e.g., data splits, hyperparameters, how they were chosen)?
			\answerYes{All the training details are specified in Appendices~\ref{appendix:empirical_evaluation}~and~\ref{appendix:use_case}.}
			\item Did you report error bars (e.g., with respect to the random seed after running experiments multiple times)?
			\answerYes{Figure~\ref{fig:overall_concept_acc} measures the accuracy of our method over several runs. All the runs are aggregated together in the form of a box-plot.}
			\item Did you include the total amount of compute and the type of resources used (e.g., type of GPUs, internal cluster, or cloud provider)?
			\answerYes{Our computing resources are described in Appendices~\ref{appendix:empirical_evaluation}~and~\ref{appendix:use_case}.} 
		\end{enumerate}

		\item If you are using existing assets (e.g., code, data, models) or curating/releasing new assets...
		\begin{enumerate}
			\item If your work uses existing assets, did you cite the creators?
			\answerYes{We have included a citation for the MNIST dataset~\cite{LeCun1998}, for the ECG dataset~\cite{Goldberger2000}, for the CUB dataset~\cite{Branson2011}, for the SEER dataset~\cite{Adamo2017} and for the InceptionV3 model~\cite{Szegedy2015}.}
			\item Did you mention the license of the assets?
			\answerYes{The licenses are mentioned in Appendices~\ref{appendix:empirical_evaluation}~and~\ref{appendix:use_case}.}
			\item Did you include any new assets either in the supplemental material or as a URL?
			\answerNA{We have used only existing assets.}
			\item Did you discuss whether and how consent was obtained from people whose data you're using/curating?
			\answerNA{All the datasets that we use are publicly available.}
			\item Did you discuss whether the data you are using/curating contains personally identifiable information or offensive content?
			\answerNA{The medical datasets we use are public and have been de-identified.}
		\end{enumerate}

		\item If you used crowdsourcing or conducted research with human subjects...
		\begin{enumerate}
			\item Did you include the full text of instructions given to participants and screenshots, if applicable?
			\answerNA{}
			\item Did you describe any potential participant risks, with links to Institutional Review Board (IRB) approvals, if applicable?
			\answerNA{}
			\item Did you include the estimated hourly wage paid to participants and the total amount spent on participant compensation?
			\answerNA{}
		\end{enumerate}

	\end{enumerate}

	%%%%%%%%%%%%%%%%%%%%%%%%%%%%%%%%%%%%%%%%%%%%%%%%%%%%%%%%%%%%
	
	\clearpage
	\appendix
	
	\section{CAR Classifiers} 
	\label{appendix:car_classifiers}
	In this appendix, we provide some details about our CAR classifiers.

\textbf{Implementation.} To determine the concept activation regions, we fit a SVC $s^c_{\kappa}$ for each concept $c \in [C]$. This process is described in Algorithm~\ref{alg:car_classifier}. 

\begin{algorithm}
	\setstretch{1.35}
	\caption{Fit CAR Classifier}\label{alg:car_classifier}
	\KwIn{Neural network $\f = \l \circ \g : \X \rightarrow \H \rightarrow \Y$, kernel function $\kappa : \H^2 \rightarrow \R^+$, concept positives $\Pos^c \subset \X$, concept negatives $\Neg^c \subset \X$}
	\KwOut{CAR SVC Classifier $s^c_{\kappa} : \H \rightarrow \{0,1\}$}
	$s^c_{\kappa} \gets \mathrm{SVC}(\kappa)$ \Comment*[r]{Initialize SVC classifier} 
	$\D \gets \left\{ \left(\g(\x), \mathbb{I}_{\Pos^c}(\x) \right) \mid \x \in \Pos^c \bigsqcup \Neg^c \right\}$ \Comment*[r]{Assemble SVC training set}
	$(s^c_{\kappa}).\mathrm{fit}(\D)$ \Comment*[r]{Fit SVC classifier}
	\KwRet{$s^c_{\kappa}$}
\end{algorithm}

where $\mathbb{I}_{\Pos^c}$ is the indicator function on the set $\Pos^c$. Our implementation leverages the SVC class from scikit-learn~\cite{Scikit-learn}. We use the default hyperparameters for each CAR classifier that appears in our experiments. The resulting classifier $s^c_{\kappa}$ can then be used to assess if a test point $\x \in \X$ has a representation $\g(\x) \in \H$ that falls in the CAR $\H^c$ or not:
\begin{align*}
	\g(\x) \in \left\{
	\begin{array}{ll}
		\H^c & \mathrm{if} \ s^c_{\kappa} \circ \g (\x) = 1 \\
		\H^{\neg c} & \mathrm{if} \ s^c_{\kappa} \circ \g (\x) = 0. \\
	\end{array}
	\right. 
\end{align*}

\begin{figure}[b]
	\centering
	\includegraphics[width=.6\textwidth]{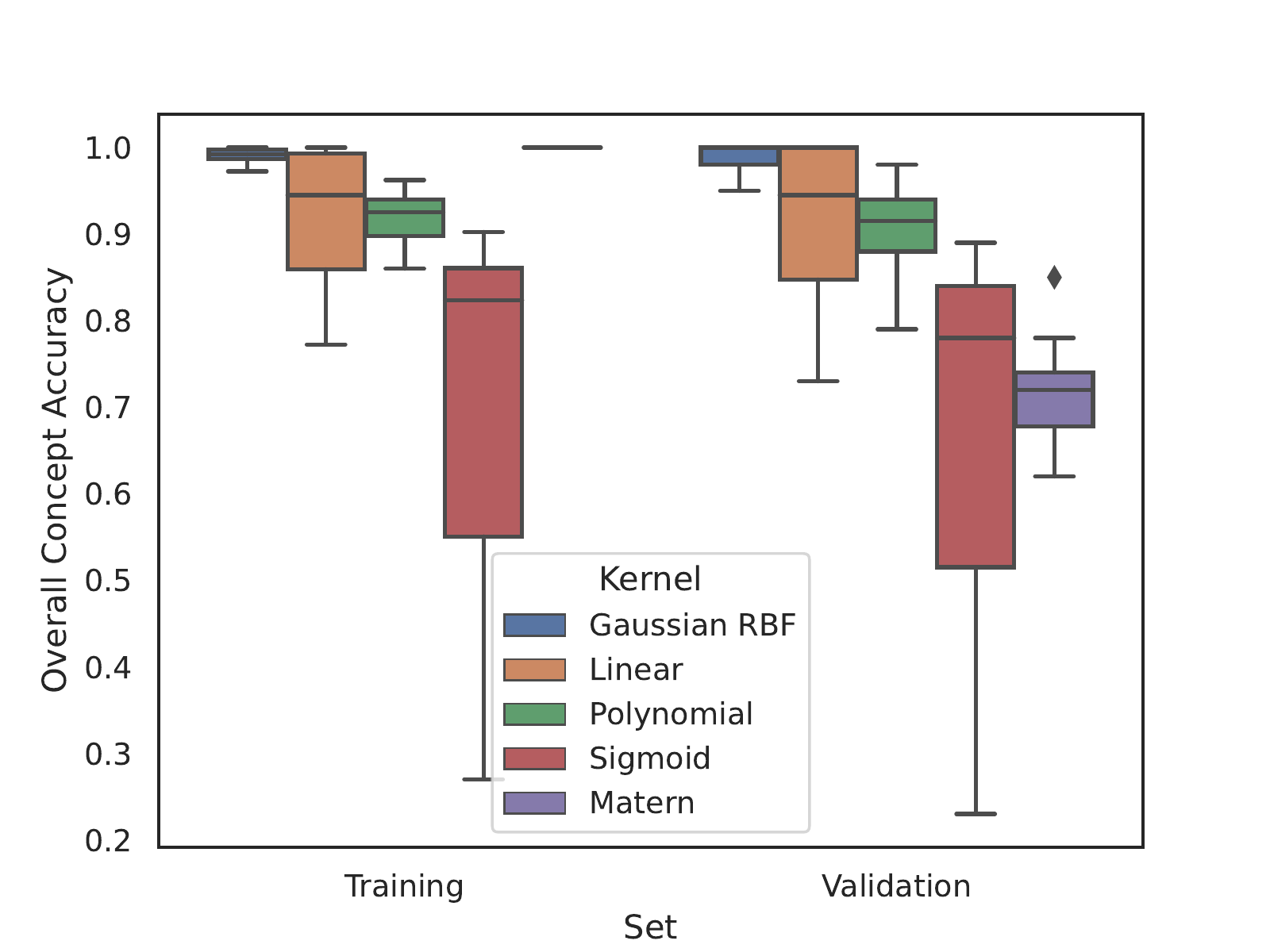}
	\caption{Accuracy of CAR classifiers for MNIST concepts with different kernels.}
	\label{fig:mnist_kernel_sensitivity}
\end{figure}

\textbf{Choosing the kernel.} Algorithm~\ref{alg:car_classifier} requires the user to specify a kernel function $\kappa$. To select this kernel, a good approach is to train several SVCs $s^c_{\kappa}$ with different kernels $\kappa$ and see how each SVC generalizes on a validation set. We illustrate this process with MNIST in Figure~\ref{fig:mnist_kernel_sensitivity}. We note that the Gaussian RBF kernel outperforms the other kernels on the validation set, although the Matern kernel achieves perfect accuracy on the training set. This highlights the importance of evaluating a CAR classifier on a held-out dataset to make sure that the related CAR offers a good description of how concepts are distributed in the latent space $\H$. In our experiments, we found that Gaussian RBF kernels are often the most interesting option.  

\begin{figure}
	\centering
	\includegraphics[width=.6\textwidth]{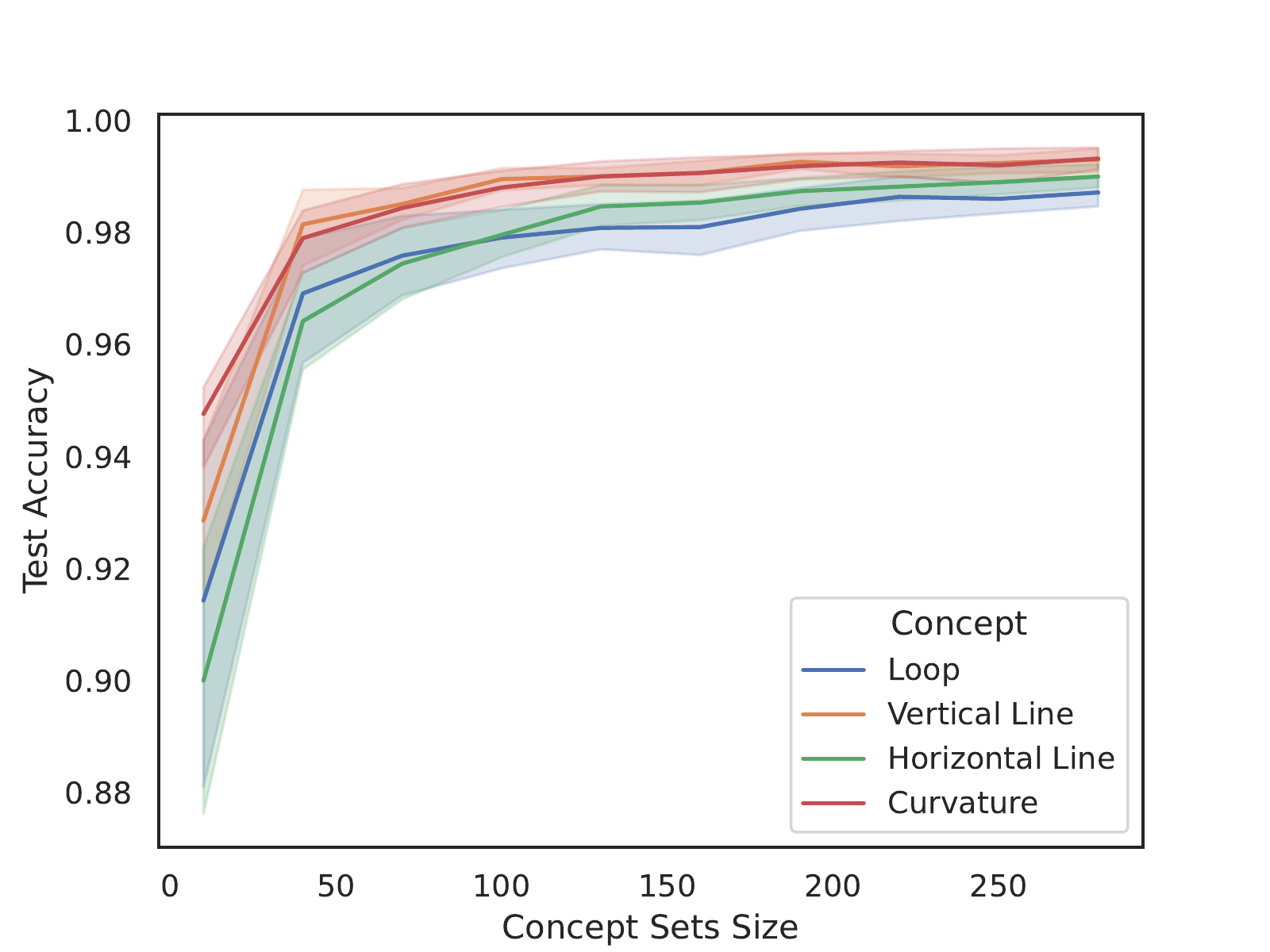}
	\caption{Accuracy of CAR classifiers for MNIST concepts for various concept set size $N^c$ (average and $95\%$ confidence intervals based on 10 runs).}
	\label{fig:mnist_concept_size_impact}
\end{figure}

\textbf{Concept set size.} Algorithm~\ref{alg:car_classifier} requires the user to specify concepts sets $\Pos^c$ and $\Neg^c$. What size $N^c = \card{\Pos^c} = \card{\Neg^c}$ should we choose for these concept sets? It seems logical that larger concepts sets are more likely to yield more accurate CARs. To study this experimentally, we propose to fit several concept classifiers for MNIST by varying the size $N^c$ of their training concept sets $\Pos^c$ and $\Neg^c$. We report the results in Figure~\ref{fig:mnist_concept_size_impact}. As we can see, the curves flatten above $N^c > 200$. Increasing the concept sets size beyond this point does not improve the accuracy of the resulting CAR classifier. We recommend to acquire concept examples until the performance of the CAR classifier stabilizes. In our experiments, we found that $N^c = 200$ examples is often sufficient to obtain accurate CAR classifiers. 

\textbf{Tuning hyperparameters.} In the case where the user desires a CAR classifier that generalizes as well as possible, tuning these hyperparameters might be useful. We propose to tune the kernel type, kernel width and error penalty of our CAR classifiers $s^c_{\kappa}$ for each concept $c \in [C]$ by using Bayesian optimization
and a validation concept set:

\begin{enumerate}
	\item Randomly sample the hyperparameters from an initial prior distribution $\theta_h \sim P_{\mathrm{prior}}$.
	\item Split the concept sets $\mathcal{P}^c, \mathcal{N}^c$ into training concept sets $\mathcal{P}^c_{\mathrm{train}}, \mathcal{N}^c_{\mathrm{train}}$ and validation concept sets $\mathcal{P}^c_{\mathrm{val}}, \mathcal{N}^c_{\mathrm{val}}$.
	\item For the current value $\theta_h$ of the hyperparameters, fit a model $s^c_{\kappa}$ to discriminate
	the training concept sets $\mathcal{P}^c_{\mathrm{train}}, \mathcal{N}^c_{\mathrm{train}}$.
	\item Measure the accuracy $\mathrm{ACC}_{\mathrm{val}} = \frac{\sum_{x \in \mathcal{P}^c_{\mathrm{val}}} \boldsymbol{1}(s^c_{\kappa}\circ \ g(x)=1) \ + \
		\sum_{x \in \mathcal{N}^c_{\mathrm{val}}} \boldsymbol{1}(s^c_{\kappa}\circ \ g(x)=0)}{|\mathcal{P}^c_{\mathrm{val}} \ \cup \ \mathcal{N}^c_{\mathrm{val}}|}$.
	\item Update the current hyperparameters $\theta_h$ based on $\mathrm{ACC}_{\mathrm{val}}$
	using Bayesian optimization (Optuna in our case).
	\item Repeat 3-5 for a predetermined number of trials.
\end{enumerate}

We applied this process to the CAR accuracy experiment (same setup as in Section 3.1.1 of the main paper) to tune the CAR classifiers for the CUB concepts. Interestingly, we noticed no improvement
with respect to the CAR classifiers reported in the main paper: tuned and standard CAR classifier have an average accuracy of $(93 \pm .2) \%$ for the penultimate Inception layer.
This suggests that the accuracy of CAR classifiers is not heavily dependant on hyperparameters
in this case.

	\section{TCAR Global Explanations}  
	\label{appendix:tcar}
	In this appendix, we provide some details about the TCAR scores. 

\textbf{Implementation.} When the CAR classifiers are available, they permit to compute TCAR scores through Algorithm~\ref{alg:tcar_class}.

\begin{algorithm}
	\setstretch{1.35}
	\caption{Compute Class-Concept TCAR Score}\label{alg:tcar_class}
	\KwIn{Neural network $\f = \l \circ \g : \X \rightarrow \H \rightarrow \Y$, CAR classifier $s^c_{\kappa} : \H \rightarrow \{0,1\}$, set of examples $\D_k \subset \X$ of a given class $k \in [d_Y]$}
	\KwOut{TCAR score $\TCAR^c_k \in [0,1]$}
	$\mathrm{count} \gets 0$ \Comment*[r]{Initialize concept positive count}
	\For{$\x \in \D_k$}{
		\If{$s^c_{\kappa} \circ \g (\x) = 1$}{
			$\mathrm{count} \gets \mathrm{count} + 1$ \Comment*[r]{Increment count if example in the CAR}
		}
	}
	$\TCAR^c_{k} \gets \nicefrac{\mathrm{count}}{\card{\D_k}} $ \Comment*[r]{Normalize count to get TCAR}
	\KwRet{$\TCAR^c_{k}$}
\end{algorithm}

\textbf{TCAR between concepts.} Up until now, we have discussed TCAR scores that indicate how models relate classes to concepts. It is possible to define a similar score to estimate how models relate two concepts with each other. Given a set $\D \subset \X$ of examples, we define the TCAR score associated to the concepts $c_1, c_2 \in [C]$ as the ratio $$\TCAR^{c_1, c_2} = \frac{\card{\g(\D) \bigcap \H^{c_1} \bigcap \H^{c_2} }}{\card{\g(\D) \bigcap (\H^{c_1} \bigcup \H^{c_2})}}.$$ Again, $\TCAR = 0$ corresponds to no overlap and $\TCAR = 1$ describes a perfect overlap. We note that the concept-concept TCAR score can be interpreted as a Jaccard index between the sets $\g(\D) \bigcap \H^{c_1}$ and $\g(\D) \bigcap \H^{c_2}$. This score is symmetric with respect to the concepts: $\TCAR^{c_1, c_2} = \TCAR^{c_2, c_1}$. The computation of this score is done as in Algorithm~\ref{alg:tcar_concept}.

\begin{algorithm}
	\setstretch{1.35}
	\caption{Compute Concept-Concept TCAR Score}\label{alg:tcar_concept}
	\KwIn{Neural network $\f = \l \circ \g : \X \rightarrow \H \rightarrow \Y$, CAR classifiers $s^{c_1}_{\kappa} : \H \rightarrow \{0,1\}$ and $s^{c_2}_{\kappa} : \H \rightarrow \{0,1\}$, set of examples $\D \subset \X$}
	\KwOut{TCAR score $\TCAR^{c_1 , c_2} \in [0,1]$}
	$\mathrm{numerator} \gets 0$ \Comment*[r]{Initialize score numerator}
	$\mathrm{denominator} \gets 0$ \Comment*[r]{Initialize score denominator}
	\For{$\x \in \D$}{
		\If{$s^{c_1}_{\kappa} \circ \g (\x) = 1$ and $s^{c_2}_{\kappa} \circ \g (\x) = 1$}{
			$\mathrm{numerator} \gets \mathrm{numerator} + 1$ \Comment*[r]{Increment if example in both CARs}
		}
		\If{$s^{c_1}_{\kappa} \circ \g (\x) = 1$ or $s^{c_2}_{\kappa} \circ \g (\x) = 1$}{
			$\mathrm{denominator} \gets \mathrm{denominator} + 1$ \Comment*[r]{Increment if example in one CAR}
		}
	}
	$\TCAR^{c_1 , c_2} \gets \nicefrac{\mathrm{numerator}}{\mathrm{denominator}} $ \Comment*[r]{Assemble the TCAR score}
	\KwRet{$\TCAR^{c_1 , c_2}$}
\end{algorithm}

\textbf{TCAR between MNIST concepts.} As an illustration, we compute concept-concept TCAR scores for the MNIST concepts and report the results in Figure~\ref{fig:mnist_tcar_inter_concept}. We see that the model relates concepts that tend to appear together (e.g. curvature and loop). Hence, concept-concept TCAR scores can serve as a proxy for the concept semantics encoded in a model's representation space. We note the similarity with the correlation between concept-based feature importance illustrated in Figure~\ref{fig:attr_corr}. The main difference is that concept-concept TCAR scores do not explicitly refer to input features.

\begin{figure}[b]
	\centering
	\includegraphics[width=.6\textwidth]{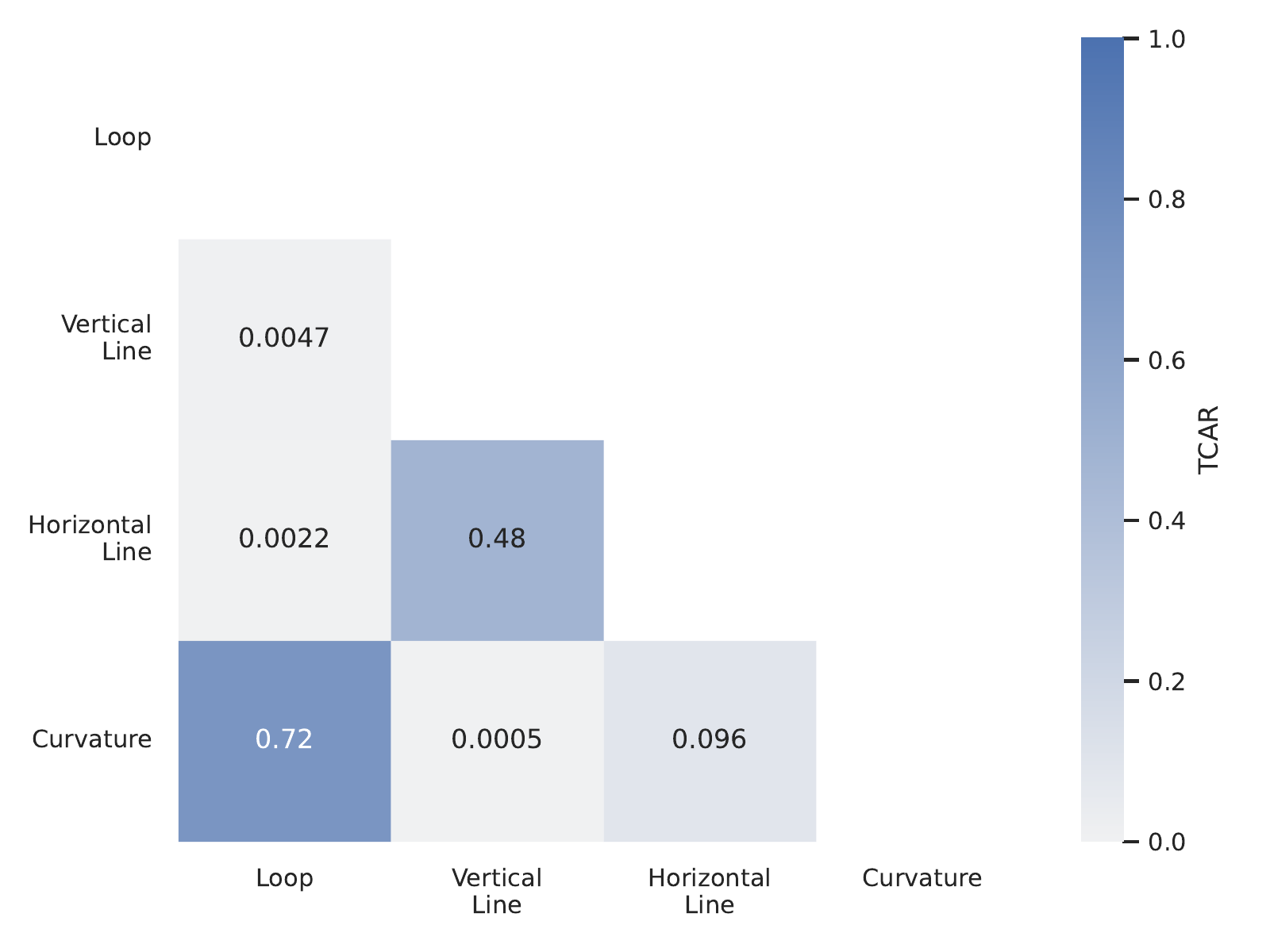}
	\caption{TCAR score between MNIST concepts.}
	\label{fig:mnist_tcar_inter_concept}
\end{figure}

\textbf{Generalizing concept sensitivity.} In our formalism, it is perfectly possible to define a \emph{local} concept activation vector through the concept density $\rho^c : \mathcal{H} \rightarrow \mathbb{R}^+$ defined in Definition~\ref{def:concept_density}. Indeed, the vector $\nabla_{\h}\rho^c[\h] \in \mathcal{H}$ points in the direction of the representation space $\mathcal{H}$ where the concept density (and hence the presence of the concept) increases. Hence, this vector can be interpreted as a \emph{local} concept activation vector. Note that this vector becomes global whenever we parametrize the concept density $\rho^c$ with a linear kernel $\kappa(\h_1, \h_2) = \h_1^{\intercal} \h_2$. Equipped with this generalized notion of concept activation vector, we can also generalize the CAV concept sensitivity $S^c_k$ by replacing the CAV $w^c$ by $\nabla_{\h}\rho^c[\h]$ for the representation $\h = \g(\x)$ of the input $\x \in \mathcal{X}$:  

\begin{align*}
	S^c_k(\x) \equiv (\nabla_{\h} \rho^c [\g(\x)])^{\intercal} (\nabla_{\h} l_k [\g(\x)]).
\end{align*}

In this way, all the interpretation provided by the CAV formalism are also available in the CAR formalism. 
	
	\section{CAR Feature Importance} \label{appendix:car_feature_importance}
	In this appendix, we provide some details about our concept-based feature importance.

\textbf{Implementation.} Concept-based feature importance uses the concept densities $\rho^c$ to attribute an importance score to each feature in order to confirm/reject the presence of a concept $c \in [C]$ for a given example $\x \in \X$. This process is described in Algorithm~\ref{alg:car_feature_importance}. We stress that features importance scores are computed with \emph{concept densities} and not with SVC classifiers. The reason for this is that SVCs are non-differentiable functions of the input and are therefore incompatible with gradient-based attribution methods such as Integrated-Gradients~\cite{Sundararajan2017}. One could argue that a sparse version of the density $\rho^c$ could be used by only allowing support vectors from the SVC to contribute. Our first implementation was relying on this approach. Unfortunately, this often led to vanishing importance scores. A possible explanation is the following: Gaussian radial basis function kernels $\kappa(\h_1 , \h_2) = \exp \left[ - (\gamma \norm[\H]{\h_1 - \h_2})^2\right]$ decay very quickly as $\norm[\H]{\h_1 - \h_2}$ increases. Hence, unless the example $\x$ we wish to explain has a representation $\g(\x)$ located near a support vector of the SVC classifier, the sparse density (and its gradients) vanishes. On the other hand, using the density from Definition~\ref{def:concept_density} allows us to incorporate the representations of all concept examples. This makes it more likely that the representation $\g(\x)$ is located near (at least) one of the representation that contributes to the density. This permits to solve the vanishing gradient problem that led to vanishing attributions. It goes without saying that this solution comes at the expense of having to compute a density that scales linearly with the size $N^c$ of the concept sets. In our implementation, we make this tractable by restricting to small concept sets ($N^c \leq 250$) and by saving the concept sets representations $\g(\Pos^c)$ and $\g(\Neg^c)$ to limit the number of queries to the model. We use Captum's implementation of feature importance methods~\cite{Kokhlikyan2020}. When using radial basis function kernels, we tune the kernel width $\gamma$ with  Optuna~\cite{Optuna} with $1,000$ trials so that the Parzen window classifier $\mathbb{I}_{\R^+} \circ \rho^c$, where $\mathbb{I}_{\R^+}$ denotes the indicator function on $\R^+$, accurately discriminates between positives representations $\g(\Pos^c)$ and negative representations $\g(\Neg^c)$.

\begin{algorithm}
	\setstretch{1.35}
	\caption{Compute concept-based feature importance}\label{alg:car_feature_importance}
	\KwIn{Neural network $\f = \l \circ \g : \X \rightarrow \H \rightarrow \Y$, kernel function $\kappa: \H^2 \rightarrow \R^+$, concept positives $\Pos^c \subset \X$, concept negatives $\Neg^c \subset \X$, feature importance method $\a : \mathcal{A}(\R^{\X}) \times \X \rightarrow \R^{d_X}$, example $\x \in \X$}
	\KwOut{Feature importance vector $\a \in \R^{d_X}$ for the example $\x$}
	$\rho^c \gets \mathrm{Density}(\g, \kappa, \Pos^c , \Neg^c)$ \Comment*[r]{Initialize concept density as Def.~\ref{def:concept_density}} 
	$\a \gets \a(\rho^c \circ \g, \x)$ \Comment*[r]{Compute feature importance vector}
	\KwRet{$\a$}
\end{algorithm}
where $\mathcal{A}(\R^{\X})$ denotes the hypothesis set of scalar functions on the input space $\X$. In this case, this corresponds to the set of neural networks with input space $\X$ and scalar output.

\textbf{Completeness.} Completeness is a crucial property of some feature importance methods. It guarantees that the feature importance scores can be used to reconstruct the model prediction. To make this more precise, let $f: \X \rightarrow \R$ be a model with \emph{scalar} output and let $\x \in \X$ be an example we wish to explain. Feature importance methods assign a score $a_i(f, \x)$ to each feature $i \in [d_X]$. This score reflects the sensitivity of the model prediction $f(\x)$ with respect to the component $x_i$ of $\x$. Completeness is fulfilled whenever the sum of this importance scores equals the model prediction up to a constant baseline $b \in \R$: $\sum_{i=1}^{d_X} a_i(f, \x) = f(\x) - b$. The baseline varies from one method to the other. For instance, Lime~\cite{Ribeiro2016} uses a vanishing baseline $b = 0$, SHAP~\cite{Lundberg2017} uses the average prediction $b = \mathbb{E}_{\boldsymbol{X}} \left[f(\boldsymbol{X})\right]$ and Integrated Gradients~\cite{Sundararajan2017} use a baseline prediction $b = f(\bar{\x})$ for some baseline input $\bar{\x}$. With completeness, the importance scores $a_i(f, \x)$ are given a natural interpretation: their sign indicates if the features tend to increase/decrease the prediction $f(\x)$ and their absolute value indicates how important this effect is. When combined with CAR concept densities $\rho^c$, this interpretation is even more insightful.  

\begin{proposition}[CAR Completeness] \label{prop:completeness}
	Consider a neural network decomposed as $\f = \l \circ \g$, where $\g : \X \rightarrow \H$ is a feature extractor mapping the input space $\X$ to the representation space $\H$ and $\l$ is a label function that maps the representation space $\H$ to the label set $\Y$. Let $\rho^c : \H \rightarrow \R^+$ be a concept density for some concept $c \in [C]$, defined as in Definition~\ref{def:concept_density}. Let $\a : \mathcal{A}(\R^{\X}) \times \X \rightarrow \R$ be a feature importance method satisfying the completeness property: $\sum_{i=1}^{d_X} a_i(f, \x) = f(\x) - b$ for some baseline $b \in \R$ and for all $\x \in \X$. Then, we have the following completeness property for the concept-based density:
	\begin{align*}
		\sum_{i=1}^{d_X} a_i(\rho^c \circ \g , \x) = \rho^c\circ\g(\x) - b
	\end{align*}
\end{proposition}
\begin{remark}
	With this property, we can interpret features $i \in [d_X]$ with $a_i(\rho^c \circ \g , \x) > 0$ as those that tend to increase the concept density. This means that those features are important for the feature extractor $\g$ to map the example in a region of the representation space $\H$ where the concept is present. Hence, those are features that are important to identify a given concept $c \in [C]$. Conversely, features $i \in [d_X]$ with $a_i(\rho^c \circ \g , \x) < 0$ tend to decrease the concept density and therefore brings the example $\x$ in a region of the representation space $\H$ where the concept is absent. These features can therefore be interpreted as important to reject the presence of a given concept $c \in [C]$. 
\end{remark}
\begin{proof}
	We simply note that $\rho^c \circ \g$ is a scalar function as $\rho^c \circ \g (\x) \in \R^+$ for all $\x \in \X$. Hence, the proposition immediately follows by applying the completeness property to $f = \rho^c \circ \g$.
\end{proof}

\textbf{Input baseline choice.} The choice of baseline input $\bar{\x}$ has a notable effect on feature importance methods~\cite{Sturmfels2020}. What constitutes a good input baseline is problem dependant. Intuitively, $\bar{\x}$ should correspond to an input $\x \in \X$ where no information is present~\cite{Covert2020}. Let us now explain how this information removal is achieved with the datasets that we use in our experiments. \circled{1} For MNIST, we chose a black image $\bar{\x} = \boldsymbol{0}$ as an input baseline. This is because MNIST images have a black background an the information comes from white pixels that represent the digits. \circled{2} For ECG, we chose a constant $\bar{\x} = \boldsymbol{0}$ time series as an input baseline. This is because ECG time series are normalized ($x_t \in  [0,1]$ for all $t \in [d_X]$) and the pulse information comes from time steps where the time series is non-vanishing $x_t \neq 0$. \circled{3} For CUB, choosing a baseline is more complicated. The reason for this is that the background colour changes from one image to the other and rarely corresponds to a single colour. To address this, we proceed as in the literature~\cite{Fong2017} and select a baseline $\bar{\x}$ that corresponds to a blurred version of the image we wish to explain: $\bar{\x}(\x) = \boldsymbol{G}_{\sigma} \otimes \x$, where $\boldsymbol{G}_{\sigma}$ is a Gaussian filter of width $\sigma$ and $\otimes$ denotes the convolution operation. We note that this baseline depends on which input $\x \in \X$ we want to explain. In our implementation, we use $\sigma = 50$ to have images that are significantly blurred. \circled{4} For SEER, we chose a constant vector $\bar{\x} = \boldsymbol{0}$ as an input baseline. This is because all continuous features are standardized and all categorical features are one-hot encoded. 
	
	\section{CAR Latent Isometry Invariance} \label{appendix:isometry_invariance}
	In this appendix, we prove that CAR explanations are invariant under isometries of the latent space when built with a radial kernel. Let us first rigorously define the notion of isometry between two vector spaces.

\begin{definition}[Isometry] \label{def:isometry}
	Let $(\H, \norm[\H]{\cdot})$ and  $(\H', \norm[\H']{\cdot})$ be two normed vector spaces. An \emph{isometry} from $\H$ to $\H'$ is a map $\isometry: \H \rightarrow \H'$ such that for all $\h_1, \h_2 \in \H$:
	\begin{align*}
		\norm[\H']{\isometry(\h_1) - \isometry(\h_2)} = \norm[\H]{\h_1 - \h_2}.
	\end{align*}
We say that the two spaces $(\H, \norm[\H]{\cdot})$ and  $(\H', \norm[\H']{\cdot})$ are \emph{isometric} if there exists a bijective isometry $\isometry$ from $\H$ to $\H'$.
\end{definition}  

An explanation method is invariant to latent space isometries if applying a bijective isometry $\isometry$ to the model's latent space $\H$ does not affect the explanations produced by the method. To make this more formal, we write the model as $\f = \l \circ \g = \l \circ \isometry^{-1} \circ \isometry \circ \g$, where $\isometry^{-1}$ is the inverse of the bijective isometry $\isometry$. In this setup, we could produce explanations with CARs by making the following replacements for the feature extractor and the label map: $\g \isometryarrow \isometry \circ \g$ and $\l \isometryarrow \l \circ \isometry^{-1}$. The explanations are defined to be invariant to latent space isometries if they are unaffected by this replacement. It is legitimate to expect this since the previous replacement leads to the same model $\f \isometryarrow \f$ and substitutes the latent space by an isometric latent space $\H \isometryarrow \H' =  \isometry(\H)$.

Let us now discuss the isometry invariance of CAR explanations. First, we recall that CARs are defined through a kernel $\kappa$. Not all kernels lead to isometry invariant CAR explanations. We will show that it holds for a family of kernels known as radial kernels.

\begin{definition}[Radial Kernel] \label{def:radial_kernel}
	A \emph{radial kernel} is a kernel function $\kappa: \H^2 \rightarrow \R^+$ that can be written as
	\begin{align}
		\kappa(\h_1 , \h_2) = \chi\left(\norm[\H]{\h_1 - \h_2}\right),
	\end{align}	
	with a function $\chi: \R^+ \rightarrow \R^+$.
\end{definition}

A typical example of radial kernel is the Gaussian radial basis function kernel (RBF) that corresponds to $\chi(x)=\exp \left[- (\gamma x)^2\right]$ for some $\gamma \in \R^+$. We are now ready to state and prove the isometry invariance property of CAR explanations. 

\begin{proposition}[CAR Isometry Invariance] \label{prop:isometry_invariance}
	Consider a neural network decomposed as $\f = \l \circ \g$, where $\g : \X \rightarrow \H$ is a feature extractor mapping the input space $\X$ to the representation space $\H$ and $\l$ is a label function that maps the representation space $\H$ to the label set $\Y$. Let $\kappa: \H^2 \rightarrow \R^+$ be a radial kernel $\kappa(\h_1, \h_2) = \chi(\norm[\H]{\h_1 - \h_2})$ that we use to define CAR's concept density in Definition~\ref{def:concept_density}. Let $\isometry : \H \rightarrow \H'$ be a bijective isometry between the normed spaces $(\H, \norm[\H]{\cdot})$ and $(\H', \norm[\H']{\cdot})$. All the explanations outputted by CAR remain the same if we transform the latent space with the isometry $\isometry$ by making the following replacements: $\g \isometryarrow \isometry \circ \g$ and $\l \isometryarrow \l \circ \isometry^{-1}$.
\end{proposition}
\begin{proof}
	For each concept $c \in [C]$, we note that CAR explanations exclusively rely on the concept density $\rho^c$ from Definition~\ref{def:concept_density} and the associated support vector classifier $s_{\kappa}^c$. Hence, it is sufficient to show that these two functions are invariant under isometry. 
	
	\textbf{Concept Density.} We start with the concept density $\rho^c$. We note that the kernel function $\kappa$ can be applied to vectors from $\H'$ as $\kappa'(\h'_1, \h'_2) \equiv \chi(\norm[\H']{\h'_1 - \h'_2})$ for all $(\h'_1, \h'_2) \in \H'$. Applying CAR in the latent space $\H'$ isometric to $\H$ corresponds to using this kernel to compute an alternative density $\rho'^c$. Let us fix an example $\x \in \X$. Under the isometry, the concept density for this example transforms as $\rho^c \circ \g (\x) \isometryarrow \rho'^c \circ \isometry \circ \g (\x)$. Let us show that this is in fact an invariance:
	\begin{align*}
		\rho'^c \circ \isometry \circ \g (\x) &\overset{\mathrm{Def.}~\ref{def:concept_density}}{=} \sum_{n=1}^{N^c} \kappa' \left[\isometry\circ \g(\x) , \isometry\circ \g(\x^{c , n}) \right] - \kappa' \left[\isometry\circ \g(\x) , \isometry\circ \g(\x^{\neg c , n}) \right]\\
		&\overset{\mathrm{Def.~\ref{def:radial_kernel}}}{=} \sum_{n=1}^{N^c} \chi\left(\norm[\H']{\isometry\circ \g(\x) - \isometry\circ \g(\x^{c , n})}\right) - \chi\left(\norm[\H']{\isometry\circ \g(\x) - \isometry\circ \g(\x^{\neg c , n})}\right) \\
		&\overset{\mathrm{Def.~\ref{def:isometry}}}{=} \sum_{n=1}^{N^c} \chi\left(\norm[\H]{ \g(\x) -  \g(\x^{c , n})}\right) - \chi\left(\norm[\H]{ \g(\x) -  \g(\x^{\neg c , n})}\right) \\
		&\overset{\mathrm{Def.~\ref{def:radial_kernel}}}{=} \sum_{n=1}^{N^c} \kappa \left[\g(\x) , \g(\x^{c , n}) \right] - \kappa \left[ \g(\x) , \g(\x^{\neg c , n}) \right]\\
		&\overset{\mathrm{Def.~\ref{def:concept_density}}}{=} \rho^c \circ \g (\x).
	\end{align*}
We deduce that the concept density is invariant under isometry: $\rho^c \circ \g (\x) \isometryarrow \rho^c \circ \g (\x)$. 

\textbf{SVC.} The proof is more involved for the SVC $s^c_{\kappa}$. We assume that the reader is familiar with the standard theory of SVC. If this is not the case, please refer e.g. to Chapter~7 of \cite{Bishop2006}. For the sake of notation, we will abbreviate $\g(\x^{c, n})$ and $\g(\x^{\neg c, n})$ by  $\h^{c, n}$ and $\h^{\neg c, n}$ respectively. Similarly, we abbreviate $\isometry \circ \g(\x^{c, n})$ and $\isometry \circ \g(\x^{\neg c, n})$ by  $\h'^{c, n}$ and $\h'^{\neg c, n}$ respectively. The SVC concept classifier can be written as
\begin{align} \label{equ:svc}
	s^c_{\kappa}(\h) = \mathbb{I}_{\R^+} \left( \sum_{n=1}^{N^c} \alpha^{c,n} \kappa[\h , \h^{c,n}] - \alpha^{\neg c,n} \kappa[\h , \h^{\neg c,n}] + \beta \right),
\end{align}
where $\mathbb{I}_{\R^+}$ is the indicator function on $\R^+$, the real numbers $\alpha^{c,n}, \alpha^{\neg c, n}$ maximize the objective
\begin{align} \label{equ:svc_lagrangian}
	\Lagrangian(\dualpha^c , \dualpha^{\neg c}) =& \sum_{n=1}^{N^c} \alpha^{c, n} + \alpha^{\neg c, n} - \frac{1}{2}  \sum_{n=1}^{N^c} \sum_{m=1}^{N^c}  \alpha^{c, n} \alpha^{c, m} \kappa\left[\h^{c, n}, \h^{c, m}\right]  \nonumber \\
	& + \alpha^{\neg c, n} \alpha^{\neg c, m} \kappa\left[\h^{\neg c, n}, \h^{\neg c, m}\right] - 2 \alpha^{c, n} \alpha^{\neg c, m} \kappa\left[\h^{c, n}, \h^{ \neg c, m}\right]
\end{align}
under the constraints
\begin{align*}
	&\alpha^{c, n} \geq 0 \hspace{.5cm} \alpha^{\neg c, n} \geq 0 \hspace{.5cm} \forall n \in [N^c],\\
	&\sum_{n=1}^{N^c} \alpha^{c, n} - \alpha^{\neg c, n} = 0.
\end{align*} 
Finally, the bias term can be written as 
\begin{align} \label{equ:svc_bias}
	\beta =& \frac{1}{\card{\S^c} + \card{\S^{\neg c}}}  \left(  \card{\S^c} - \card{\S^{\neg c}}  - \sum_{n \in \S^c} \sum_{m \in \S^c} \alpha^{c,m} \kappa \left[ \h^{c, n}, \h^{c, m} \right]     \right.   \nonumber \\
	& + \sum_{n \in \S^c} \sum_{m \in \S^{\neg c}} \alpha^{\neg c,m} \kappa\left[ \h^{c, n}, \h^{\neg c, m} \right] - \sum_{n \in \S^{\neg c}} \sum_{m \in \S^c} \alpha^{c,m} \kappa\left[ \h^{\neg c, n}, \h^{c, m} \right] \nonumber \\
	 & \left.   
     + \sum_{n \in \S^{\neg c}} \sum_{m \in \S^{\neg c}} \alpha^{\neg c,m} \kappa\left[ \h^{\neg c, n}, \h^{\neg c, m} \right] \right),
\end{align}
where $\S^c = \{ n \in [N^c] \mid \alpha^{c ,n} \neq 0\}$ and $\S^{\neg c} = \{ n \in [N^c] \mid \alpha^{\neg c ,n} \neq 0\}$
are the indices of the support vectors. Under isometry, the SVC classification for a latent vector $\h \in \H$ transforms as $s^c_{\kappa}(\h) \isometryarrow s'^c_{\kappa'}(\h')$ with $\h' = \isometry(\h)$ and
\begin{align} \label{equ:svc'}
	s'^c_{\kappa'}(\h') = \mathbb{I}_{\R^+} \left( \sum_{n=1}^{N^c} \underbrace{\alpha'^{c,n}}_{\circled{2}} \underbrace{\kappa'[\h' , \h'^{c,n}]}_{\circled{1}} - \underbrace{\alpha'^{\neg c,n}}_{\circled{2}} \underbrace{\kappa'[\h' , \h'^{\neg c,n}]}_{\circled{1}} + \underbrace{\beta'}_{\circled{3}}\right).
\end{align}
We will now show that $s^c_{\kappa}(\h) = s'^c_{\kappa'}(\h')$ so that the SVC is invariant under isometries. We proceed in 3 steps.

\circled{1} We show that $\kappa' \left[\h'_1, \h'_2\right] = \kappa \left[\h_1, \h_2\right]$ for any $(\h_1 , \h_2) \in \H^2$ and $\h'_1 = \isometry(\h_1), \h'_2 = \isometry(\h_2)$:
\begin{align*}
	\kappa' \left[\h'_1, \h'_2\right] &\overset{\mathrm{Def.~\ref{def:radial_kernel}}}{=} \chi \left(\norm[\H']{\isometry(\h_1) - \isometry(\h_2)}\right)\\
	&\overset{\mathrm{Def.~\ref{def:isometry}}}{=} \chi \left(\norm[\H]{\h_1 - \h_2}\right)\\
	&\overset{\mathrm{Def.~\ref{def:radial_kernel}}}{=} \kappa \left[\h_1, \h_2\right].
\end{align*}
By injecting this in \eqref{equ:svc'}, we are able to make the following replacements: $\kappa' \left[\h', \h'^{c,n}\right] = \kappa \left[\h, \h^{c,n}\right]$ and $\kappa' \left[\h', \h'^{\neg c,n}\right] = \kappa \left[\h, \h^{\neg c,n}\right]$ for all $n \in [N^c]$.

\circled{2} We show that $\alpha'^{c,n} = \alpha^{c,n}$ and $\alpha'^{\neg c,n} = \alpha^{\neg c,n}$ for all $n \in [N^c]$. To that aim, we note that $\dualpha'^c$ and $\dualpha'^{\neg c}$ maximize the objective
\begin{align*} 
	\Lagrangian'(\dualpha'^c , \dualpha'^{\neg c}) =& \sum_{n=1}^{N^c} \alpha'^{c, n} + \alpha'^{\neg c, n} - \frac{1}{2}  \sum_{n=1}^{N^c} \sum_{m=1}^{N^c}  \alpha'^{c, n} \alpha'^{c, m} \kappa'\left[\h'^{c, n}, \h'^{c, m}\right]   \\
	  &+ \alpha'^{\neg c, n} \alpha'^{\neg c, m} \kappa'\left[\h'^{\neg c, n}, \h'^{\neg c, m}\right] - 2 \alpha'^{c, n} \alpha'^{\neg c, m} \kappa'\left[\h'^{c, n}, \h'^{ \neg c, m}\right] \\
	\overset{\circled{1}}{=}& \sum_{n=1}^{N^c} \alpha'^{c, n} + \alpha'^{\neg c, n} - \frac{1}{2}  \sum_{n=1}^{N^c} \sum_{m=1}^{N^c} \alpha'^{c, n} \alpha'^{c, m} \kappa\left[\h^{c, n}, \h^{c, m}\right]  \\
	 & + \alpha'^{\neg c, n} \alpha'^{\neg c, m} \kappa\left[\h^{\neg c, n}, \h^{\neg c, m}\right] - 2 \alpha'^{c, n} \alpha'^{\neg c, m} \kappa\left[\h^{c, n}, \h^{ \neg c, m}\right] \\
	\overset{\eqref{equ:svc_lagrangian}}{=}& \Lagrangian(\dualpha'^c , \dualpha'^{\neg c}). \hspace{2.7cm}  
\end{align*}
Since this objective is identical to the one from the original SVC and the constraints are unaffected by the isometry, we deduce that the solution to this convex optimization problem is identical to the solution of \eqref{equ:svc_lagrangian}. Hence, we have that $\alpha'^{c,n} = \alpha^{c,n}$ and $\alpha'^{\neg c,n} = \alpha^{\neg c,n}$ for all $n \in [N^c]$. Again, we can make these replacements in \eqref{equ:svc'}.

\circled{3} We show that $\beta' = \beta$. First, we note that the support vector indices are invariant under isometry: $\S'^c = \{ n \in [N^c] \mid \alpha'^{c ,n} \neq 0\} = \{ n \in [N^c] \mid \alpha^{c ,n} \neq 0\} = \S^c$ and similarly $\S'^{\neg c} = \S^{\neg c}$. Hence we have
\begin{align*} 
	\beta' = \hspace{.2cm} &  \frac{1}{\card{\S^c} + \card{\S^{\neg c}}}  \left(  \card{\S^c} - \card{\S^{\neg c}}  - \sum_{n \in \S^c} \sum_{m \in \S^c} \alpha'^{c,m} \kappa' \left[ \h'^{c, n}, \h'^{c, m} \right]     \right.   \nonumber \\
	& + \sum_{n \in \S^c} \sum_{m \in \S^{\neg c}} \alpha'^{\neg c,m} \kappa'\left[ \h'^{c, n}, \h'^{\neg c, m} \right] - \sum_{n \in \S^{\neg c}} \sum_{m \in \S^c} \alpha'^{c,m} \kappa'\left[ \h'^{\neg c, n}, \h'^{c, m} \right] \nonumber \\
	& \left.   
	+ \sum_{n \in \S^{\neg c}} \sum_{m \in \S^{\neg c}} \alpha'^{\neg c,m} \kappa'\left[ \h'^{\neg c, n}, \h'^{\neg c, m} \right] \right)\\
	\overset{\circled{1} \& \circled{2}}{=}& \frac{1}{\card{\S^c} + \card{\S^{\neg c}}}  \left(  \card{\S^c} - \card{\S^{\neg c}}  - \sum_{n \in \S^c} \sum_{m \in \S^c} \alpha^{c,m} \kappa \left[ \h^{c, n}, \h^{c, m} \right]     \right.   \nonumber \\
	& + \sum_{n \in \S^c} \sum_{m \in \S^{\neg c}} \alpha^{\neg c,m} \kappa\left[ \h^{c, n}, \h^{\neg c, m} \right] - \sum_{n \in \S^{\neg c}} \sum_{m \in \S^c} \alpha^{c,m} \kappa \left[ \h^{\neg c, n}, \h^{c, m} \right] \nonumber \\
	& \left.   
	+ \sum_{n \in \S^{\neg c}} \sum_{m \in \S^{\neg c}} \alpha^{\neg c,m} \kappa\left[ \h^{\neg c, n}, \h^{\neg c, m} \right] \right)\\
	= \hspace{.4cm} &  \beta.
\end{align*}
By making the replacements from points \circled{1}, \circled{2} and \circled{3} in \eqref{equ:svc'}, we deduce that $s^c_{\kappa}(\h) = s'^c_{\kappa'}(\h')$.
\end{proof}

\iffalse
Let us now show that the CAV formalism does not verify the previous invariance property. We note that CAV explanations are built upon the concept sensitivity $S^c_k(\x) = (\w^c)^\intercal \grad_{\h} l_k \left[ \g(\x) \right]$. It is sufficient to find a counter-example to show that the invariance property does not hold for CAV. We consider a translation: $\h' = \isometry(\h) = \h + \r$. When performing the latent space replacement, we get $\grad_{\h}  l_k \left[ \g(\x) \right] \isometryarrow \grad_{\h'}  l_k \left[ \isometry^{-1} \circ \isometry \circ \g(\x) \right] = \grad_{\h'}  l_k \left[ \g(\x) \right]$. By using the chain rule, we have that $\grad_{\h'} = \Jacobian_{\isometry} \grad_{\h}$, where $\Jacobian_{\isometry}$ is the Jacobian associated to the translation $\isometry$. Since $\isometry$ is a translation, the Jacobian $\Jacobian_{\tau}$ is the identity matrix. Hence, we get the following invariance $\grad_{\h}  l_k \left[ \g(\x) \right] \isometryarrow \grad_{\h}  l_k \left[ \g(\x) \right]$. On the other hand, the CAV vector is not necessarily invariant under translation $\w^c \isometryarrow \w'^c$ with possibly $\w'^c \neq \w^c$ (see Figure~\hl{TODO} for instance). We deduce that the sensitivity $S^c_k(\x)$ is not invariant by isometry: $S^c_k(\x) \isometryarrow S'^{c}_k(\x) =  (\w'^c)^\intercal \grad_{\h}  l_k \left[ \g(\x) \right] $ with possibly $S'^{c}_k(\x) \neq S^c_k(\x)$. Since TCAV explanations are built upon these sensitivities, we deduce that CAV explanations are not invariant under isometry. 
\fi
	
	\section{Empirical Evaluation} \label{appendix:empirical_evaluation}
	This appendix provides useful details to reproduce the empirical evaluation from Section~\ref{sec:experiments}.

\textbf{Computing Resources.} All the empirical evaluations were run on a single machine equipped with a 18-Core Intel Core i9-10980XE CPU and a NVIDIA RTX A4000 GPU. The machine runs on Python 3.9~\cite{Python} and  Pytorch 1.10.2~\cite{Pytorch}. 

\textbf{Dataset licenses.} The MNIST dataset dataset is made available under the terms of the Creative Commons Attribution-Share Alike 3.0 License. The ECG dataset dataset is made available under the terms of the Open Data Commons Attribution License v1.0. The CUB dataset is made available for non-commercial research and educational purposes.

\textbf{Models.} The detailed architecture of the models are provided in Tables~\ref{tab:mnist_model},~\ref{tab:ecg_model}~and~\ref{tab:cub_model}. The InceptionV3 architecture is the same as in the literature~\cite{Szegedy2015}. We use its official Pytorch implementation.

\begin{table}[b]
	\caption{MNIST Model}
	\label{tab:mnist_model}
	\centering
	\begin{adjustbox}{width=\columnwidth}
	\begin{tabular}{llll}
		\toprule
		\textbf{Block Name} & \textbf{Layer Type} & \textbf{Hyperparameters} & \textbf{Activation}\\ 
		\midrule
		\multirowcell{3}{Conv1} & Conv2d &  Input Channels = 1, Output Channels = 16, Kernel Size = 5, Stride = 1, Padding = 0 & ReLU\\
		& Dropout2d &  p = 0.2 & \\
		& MaxPool2d &  Kernel Size = 2, Stride = 2 & \\
		\midrule
		\multirowcell{3}{Conv2} & Conv2d &  Input Channels = 16, Output Channels = 32, Kernel Size = 5, Stride = 1, Padding = 0 & ReLU\\
		& Dropout2d &  p = 0.2 & \\
		& MaxPool2d &  Kernel Size = 2, Stride = 2 & \\
		\midrule
		\multirowcell{2}{Lin1} & Flatten &   & \\
		& Linear &  Input Features = 512, Output Features = 10 & \\
		\midrule
		\multirowcell{2}{Lin2} & Dropout & p = 0.2  & \\
		& Linear &  Input Features = 10, Output Features = 5 & \\
		\midrule
		& Dropout & p = 0.2  & \\
		& Linear & Input Features = 5, Output Features = 10  & \\
		\bottomrule
	\end{tabular}
	\end{adjustbox}
\end{table}

\begin{table}
	\caption{ECG Model}
	\label{tab:ecg_model}
	\centering
	\begin{adjustbox}{width=\columnwidth}
		\begin{tabular}{llll}
			\toprule
			\textbf{Block Name} & \textbf{Layer Type} & \textbf{Hyperparameters} & \textbf{Activation}\\ 
			\midrule
			\multirowcell{2}{Conv1} & Conv1d &  Input Channels = 1, Output Channels = 16, Kernel Size = 3, Stride = 1, Padding = 1 & \\
			& MaxPool1d &  Kernel Size = 2 & \\
			\midrule
			\multirowcell{2}{Conv2} & Conv1d &  Input Channels = 16, Output Channels = 64, Kernel Size = 3, Stride = 1, Padding = 1 & \\
			& MaxPool1d &  Kernel Size = 2 & \\
			\midrule
			\multirowcell{2}{Conv3} & Conv1d &  Input Channels = 64, Output Channels = 128, Kernel Size = 3, Stride = 1, Padding = 1 & \\
			& MaxPool1d &  Kernel Size = 2 & \\
			\midrule
			\multirowcell{2}{Lin} & Flatten &   & \\
			& Linear &  Input Features = 2944, Output Features = 32 & \\
			\midrule
			& Leaky ReLU & Negative Slope = $10^{-2}$ & Leaky ReLU\\
			& Linear & Input Features = 32, Output Features = 2  & \\
			\bottomrule
		\end{tabular}
	\end{adjustbox}
\end{table}

\begin{table}
	\caption{CUB Model}
	\label{tab:cub_model}
	\centering
		\begin{tabular}{llll}
			\toprule
			\textbf{Block Name} & \textbf{Layer Type} & \textbf{Hyperparameters} & \textbf{Activation}\\  
			\midrule 
			InceptionOut & InceptionV3~\cite{Szegedy2015} & Pretrained = True & \\
			\midrule
			& Linear & Input Features = 2048, Output Features = 200 & \\
			\bottomrule
		\end{tabular}
\end{table}

\textbf{Data Split.} All the datasets are naturally split in training and testing data. In the ECG dataset, the different types of abnormal heartbeats are imbalanced (e.g. the fusion beats constitute only $0.7\%$ of the training set). Hence, we create a synthetic training set with balanced concepts using SMOTE~\cite{Chawla2011}. As in \cite{Koh2020}, the CUB dataset is augmented by using random crops and random horizontal flips. 

\textbf{Model Fitting.} In fitting each model, we use the test set as a validation set since our purpose is not to obtain the models with the best generalization but simply models that perform well on a set of examples we wish to explain (here the examples from the test set). All the models are trained to minimize the cross-entropy between their prediction and the true labels.  The hyperparameters are as follows. \circled{1} For MNIST we use a Adam optimizer with batches of 120 examples, a learning rate of $10^{-3}$, a weight decay of $10^{-5}$ for $50$ epochs with patience $10$. \circled{2} For ECG we use a Adam optimizer with batches of 300 examples, a learning rate of $10^{-3}$, a weight decay of $10^{-5}$ for $50$ epochs  with patience $10$. \circled{3} For CUB, we use a stochastic gradient descent optimizer with batches of 64 examples, a learning rate of $10^{-3}$, a weight decay of $4 \cdot 10^{-5}$ for $1,000$ epochs  with patience $50$.

\textbf{Concepts.} The concept mapping between MNIST classes and concepts is provided in Table~\ref{tab:mnist_concepts}. For the ECG and the CUB datasets, the presence/absence of a concept for each example is readily available in the dataset.

\begin{table}
	\caption{MNIST Concepts}
	\label{tab:mnist_concepts}
	\centering
	\begin{tabular}{c|cccc}
		\hline
		\backslashbox{Class}{Concept} & Loop & Vertical Line & Horizontal Line & Curvature  \\  \hline 
		0 & \cmark & \xmark & \xmark & \cmark \\
		1 & \xmark & \cmark & \xmark & \xmark \\
		2 & \cmark & \xmark & \xmark & \cmark \\
		3 & \xmark & \xmark & \xmark & \cmark \\
		4 & \xmark & \cmark & \cmark & \xmark \\
		5 & \xmark & \xmark & \cmark & \cmark \\
		6 & \cmark & \xmark & \xmark & \cmark \\
		7 & \xmark & \cmark & \cmark & \xmark \\
		8 & \cmark & \xmark & \xmark & \cmark \\
		9 & \cmark & \xmark & \xmark & \cmark \\
		\bottomrule
	\end{tabular}
\end{table}

\textbf{Concept classifiers.} All the concept classifiers are implemented with scikit-learn~\cite{Scikit-learn}. For CAR classifiers, we fit a SVC with Gaussian RBF kernel and default hyperparameters from scikit-learn. For CAV classifiers, we fit a linear classifier with a stochastic gradient descent optimizer with learning rate $10^{-2}$ and a tolerance of $10^{-3}$ for $1,000$ epochs and the remaining default hyperparameters from scikit-learn.

\textbf{Statistical significance.} The statistical significance test from Section~\ref{subsec:car_acc_validation} is performed with the scikit-learn~\cite{Scikit-learn} implementation of the permutation test. For MNIST and ECG, we consider 100 permutations per concept. For CUB, this test is more expensive since the latent spaces are high-dimensional. We consider only 25 permutations per concept in that case.

\textbf{Concept-based feature importance.} In the experiment from Section~\ref{subsec:car_concept_features}, we use Captum's implementation~\cite{Kokhlikyan2020} of Integrated Gradients~\cite{Sundararajan2017} with default parameters. In the case of CUB, storing the feature importance scores for each concept and for the whole test set requires a prohibitive amount of memory. To avoid this problem, we select $C=6$ concepts and subsample 50 positive and 50 negative examples per concept from the test set. This corresponds to a set of $600$ examples. We compute the feature importance for these examples only. 
    
\textbf{Alternative architecture.} We extended the analysis of Section~\ref{subsec:car_acc_validation} to a ResNet-50 architecture. We fine-tune the ResNet model on the CUB dataset and reproduced the experiment from Section~\ref{subsec:car_acc_validation} with this new architecture. In particular, we fit a CAR and a CAV classifier on the penultimate layer of the ResNet. We then measure the accuracy averaged over the  CUB concepts. This results in $(89 \pm 1) \%$ accuracy for CAR classifiers and $(87 \pm 1) \%$ accuracy for CAV classifiers. We deduce that CAR classifiers are highly accurate to identify concepts in the penultimate ResNet layer. As in the main paper, we observe that CAR classifiers outperform CAV classifiers, although the gap is smaller than for the Inception-V3 neural network. We deduce that our CAR formalism extends beyond the architectures explored in the paper and we hope that CAR will become widely used to interpret any more architectures.

	\section{Use Case} 
	\label{appendix:use_case}
	This appendix provides useful details to reproduce the use case from Section~\ref{subsec:seer_use_case}.

\textbf{Computing Resources.} The use case was run on a single machine equipped with a 18-Core Intel Core i9-10980XE CPU and a NVIDIA RTX A4000 GPU. The machine runs on Python 3.9~\cite{Python} and  Pytorch 1.10.2~\cite{Pytorch}. 

\textbf{Dataset license.} The SEER dataset is made available under the terms of the SEER Research Data Use Agreement.

\begin{table}
	\caption{SEER Model}
	\label{tab:seer_model}
	\centering
		\begin{tabular}{lll}
			\toprule
			\textbf{Layer Type} & \textbf{Hyperparameters} & \textbf{Activation}\\ 
			\midrule
			Linear & Input Features = 21, Output Features = 400 & ReLU \\
			Dropout & p = 0.3 & \\
			Linear & Input Features = 400, Output Features = 100 & ReLU \\
			Dropout & p = 0.3 & \\
			Linear & Input Features = 100, Output Features = 2 & \\
			\bottomrule
		\end{tabular}
\end{table}

\textbf{Model.} The detailed architecture of the model is provided in Table~\ref{tab:seer_model}.

\textbf{Data split.} We randomly split the whole SEER dataset into a training set ($90 \%$ of the data) and a test set (the remaining $10 \%$). Since patients with a death outcome are in minority (less than $3 \%$), we oversample them to obtain a balanced training set.

\textbf{Model fitting.} In fitting the model, we use the test set as a validation set since our purpose is not to obtain the model with the best generalization but simply a model that performs well on a set of examples we wish to explain (here the examples from the test set). The model is trained to minimize the cross-entropy between its prediction and the true labels.  We use a Adam optimizer with batches of 500 examples, a learning rate of $10^{-3}$, a weight decay of $10^{-5}$ for $500$ epochs with patience $50$.

\textbf{Concepts.} The concepts correspond to prostate cancer grades. Those grades can be computed from the Gleason score as follows~\cite{Gordetsky2016}:

\begin{align*}
	\mathrm{Grade(Gleason_1 , Gleason_2)} = \left\{
	\begin{array}{ll}
		1 & \mathrm{if \ Gleason_1 + Gleason_2 \leq 6} \\
		2 & \mathrm{if \ Gleason_1 = 3 \ \wedge \ Gleason_2 = 4} \\
		3 & \mathrm{if \ Gleason_1 = 4 \ \wedge \ Gleason_2 = 3} \\
		4 & \mathrm{if \ Gleason_1 + Gleason_2 = 8} \\
		5 & \mathrm{if \ Gleason_1 + Gleason_2 \geq 9.} \\
	\end{array}
	\right.
\end{align*} 
It goes without saying that the model is trained with the Gleason scores only, not with the grades. 

\textbf{Concept classifiers.} We fit a SVC with linear kernel and default hyperparameters from scikit-learn.

\textbf{Concept-based feature importance.} We use Captum's implementation~\cite{Kokhlikyan2020} of Integrated Gradients~\cite{Sundararajan2017} with default parameters.

	\section{Explanation Robustness}
	\label{appendix:robustness}
	\textbf{Adversarial perturbations.} We perform an experiment to evaluate the robustness
of CAR explanations with respect to adversarial perturbations.
In this experiment, we work with the MNIST dataset in the same setting as
the experiment from Section \ref{subsec:car_global_validation}. We train a CAR concept classifier for each MNIST concept $c \in [C]$. We use the CAR classifier to output TCAR scores relating the concept $c$ with each class $k \in [d_Y]$. As in the main paper, since the ground-truth association between concepts and classes is known (e.g. the class corresponding
to digit 8 will always have the concept loop), we can compute the correlation $r(\mathrm{TCAR}, \mathrm{TrueProp})$ between
our TCAR score and the ground-truth proportion of examples that exhibit the concept. In this experiment, the correlation is evaluated on a test set $\mathcal{D}_{\mathrm{test}} = \mathcal{D}_{\mathrm{adv}} \ \sqcup \mathcal{D}_{\mathrm{orig}}$ that contains adversarial
test examples $\mathcal{D}_{\mathrm{adv}}$ and original test examples $\mathcal{D}_{\mathrm{orig}}$. Each adversarial MNIST image $\x_{\mathrm{adv}} \in \mathcal{D}_{\mathrm{adv}}$ is constructed by finding a small (w.r.t. the $\| \cdot \|_{\infty}$ norm) perturbation $\boldsymbol{\epsilon} \in \mathbb{R}^{d_X}$ around an original test image $\x \in \mathcal{X}$ that maximizes the prediction shift for the model $\f : \mathcal{X} \rightarrow \mathcal{Y}$:

$$\boldsymbol{\epsilon} = \arg \max_{\tilde{\boldsymbol{\epsilon}} \in \mathbb{R}^{d_X}} \mathrm{Cross Entropy}[\f(\x), \f(\x + \tilde{\boldsymbol{\epsilon}})] \ s.t. \ \| \tilde{\boldsymbol{\epsilon}} \|_{\infty} < .1$$

The adversarial image is then defined as $\x_{\mathrm{adv}} \equiv \x + \boldsymbol{\epsilon}$. We measure the correlation $r(\mathrm{TCAR}, \mathrm{TrueProp})$ by varying the proportion $\frac{|\mathcal{D}_{\mathrm{adv}}|}{|\mathcal{D}_{\mathrm{test}}|}$ of adversarial examples in the test set. The results are reported in Table~\ref{tab:mnist_adversarial}.

\begin{table}
	\caption{MNIST Adversarial Perturbation Sensitivity}
	\label{tab:mnist_adversarial}
	\centering
	\begin{tabular}{cc}
		\hline
		Adversarial \% & $r(\mathrm{TCAR}, \mathrm{TrueProp})$ \\  \hline 
		              0 & .99 \\
		              5 & .99 \\
		              10 & .99 \\
		              20 & .99 \\
		              50 & .97 \\
		              70 & .95 \\
		             100 & .90 \\
		\bottomrule
	\end{tabular}
\end{table}

We observe that the TCAR scores keep a high correlation with the true proportion of examples that exhibit the concept even when all the test examples are adversarially perturbed. We conclude that TCAR explanations are robust with respect to adversarial perturbations in this setting.

\textbf{Background shift.} For completeness, we have also adapted the background shift robustness experiment in Section 7 from \cite{Koh2017}. We use CAR to explain the predictions of our Inception-V3 model trained on the original CUB training set. The explanations are made on test images where the background has been replaced. As \cite{Koh2017}, we use the segmentation of the CUB dataset to isolate the bird on each image. The rest of the image is replaced by a random background sampled from the \emph{Place365} dataset~\cite{Zhou2018b}. This results in a test set $\mathcal{D}_{\mathrm{test}}$ with a background shift with respect to the training set. By following the approach from Section 3.1.2 of our paper, we measure the correlation $r(\mathrm{TCAR}, \mathrm{TrueProp})$ between the TCAR score and the true proportion of examples in the class that exhibit the concept for each $(\mathrm{class}, \mathrm{concept})$ pair. We measured a correlation of $r(\mathrm{TCAR}, \mathrm{TrueProp}) = .82$ in the background-shifted test set. This is close to the correlation for the original test set reported in the main paper, which suggests that CAR explanations are robust with respect to background shifts. Note that this correlation is still better than the one obtained with TCAV on the original test set.
	
	\section{Using CAR to Understand Unsupervised Concepts}
	\label{appendix:unsupervised_concepts}
	Our CAR formalism adapts to a wide variety of neural network architectures. In this appendix, we use CAR to analyze the concepts discovered by a self explaining neural network (SENN) trained on the MNIST dataset. As in \cite{Alvarez-Melis2018}, we use a SENN of the form

$$f(\x) = \sum_{s=1}^S \theta_s (\x) \cdot g_s(\x),$$

Where $g_s(\x)$ and $\theta_s(\x)$ are respectively the activation and the relevance of the synthetic concept $s \in [S]$ discovered by the SENN model. We follow the same training process as~\cite{Alvarez-Melis2018}. This yields a set of $S = 5$ concepts explaining the predictions made by the SENN $f : \mathcal{X} \rightarrow \mathcal{Y}$.

We use our CAR formalism to study how the synthetic concepts $s \in [S]$ discovered by the SENN are related to the concepts $c \in \{ \mathrm{Loop}, \mathrm{Vertical \ Line}, \mathrm{Horizontal \ Line}, \mathrm{Curvature}  \}$ introduced in our paper. With our formalism, the relevance of a concept $c$ for a given prediction $\x \mapsto f(\x)$ is measured by the concept density $\rho^c \circ \g(\x)$. To analyze the relationship between the SENN concept $s$ and the concept $c$, we can therefore compute the correlation of their relevance:

$$r(s, c) = \mathrm{corr}_{\textbf{X} \sim P_{\mathrm{empirical}}(\mathcal{D}_{\mathrm{test}})} [\theta_s(\textbf{X}) , \rho^c\circ \g (\textbf{X})]. $$

When this correlation increases, the concepts $s$ and $c$ tend to be relevant together more often. We report the correlation between each pair $(s, c)$ in Table~\ref{tab:mnist_senn}.

\begin{table}
	\caption{SENN Concepts Correspondence}
	\label{tab:mnist_senn}
	\centering
	\begin{tabular}{c|cccc}
		\hline
		\textbf{Correlation}   $\mathbf{r(s , c)}$  & \textbf{Loop} & \textbf{Vertical Line} & \textbf{Horizontal Line} & \textbf{Curvature}  \\  \hline 
		\textbf{SENN Concept 1} & -0.28 & -0.12 & 0.26 & 0.11 \\
		\textbf{SENN Concept 2} & -0.50 & 0.71 & -0.03 & -0.69 \\
		\textbf{SENN Concept 3} & -0.47 & 0.10 & 0.71 & -0.14\\
		\textbf{SENN Concept 4} & -0.33 &       0.02 &        -0.06 & -0.01  \\
		\textbf{SENN Concept 5} & 0.57 &      -0.00 &        -0.63  &  0.07  \\
		\bottomrule
	\end{tabular}
\end{table}

We note the following:

\begin{enumerate}
	\item SENN Concept 2 correlates well with the Vertical Line Concept.
	\item SENN Concept 3 correlates well with the Horizontal Line Concept
	\item SENN Concept 5 correlates well with the Loop Concept.
	\item SENN Concepts 1 and 4 are not well covered by our concepts.
\end{enumerate}

The above analysis shows the potential of our CAR explanations to better understand the abstract concepts discovered by SENN models. We believe that the community would greatly benefit from the ability to perform similar analyses for other interpretable architectures, such as disentangled VAEs.
	
	\section{Using CAR with NLP}
	\label{appendix:car_nlp}
	CAR is a general framework and can be used in a wide variety of domains that involve neural networks. In the main paper, we show that CAR provides explanations for various modalities:

\begin{enumerate}
	\item Large image dataset
	\item Medical time series
	\item Medical tabular data
\end{enumerate}

We now perform a toy experiment to assess if those conclusions extend to the NLP setting. We train a small CNN on the IMDB Review dataset to predict whether a review is positive or negative. We use GloVe~\cite{Pennington2014} to turn the word tokens into embeddings. We would like to assess whether the concept $c = \mathrm{Positive \ Adjective}$ is encoded in the model's representations.
Examples that exhibit the concept $c$ are sentences containing positive adjectives. We collect a positive set $\mathcal{P}^c$ of $N^c = 90$ such sentences. The negative set $\mathcal{N}^c$ is made of $N^c$ sentences randomly sampled from the Gutenberg Poem Dataset. We verified that the sentences from $\mathcal{N}^c$ did not contain positive adjectives. We then fit a CAR classifier on the representations obtained in the penultimate layer of the CNN.

We assess the generalization performance of the CAR classifier on a holdout concept set made of $N^c = 30$ concept positive and negative sentences (60 sentences in total). The CAR classifier has an accuracy of $87 \%$ on this holdout dataset. This suggests that the concept $c$ is smoothly encoded in the model's representation space, which is consistent with the importance of positive adjectives to identify positive reviews. We deduce that our CAR formalism can be used in a NLP setting. We believe that using CARs to analyze large-scale language model would be an interesting study that we leave for future work.  
	
	\section{Increasing Explainability at Training Time}
	\label{appendix:training_time}
	Improving neural networks explainability at training time constitutes a very interesting area of research but is beyond the scope of our paper. That said, we believe that our paper indeed contains insights that might be the seed of future developments in neural network training. As an illustration, we consider an important insight from our paper: the fact that the accuracy of concept classifiers seems to increase with the depth of the layer for which we fit a classifier. In the main paper, this is mainly reflected in Figure~\ref{fig:overall_concept_acc}. This observation has a crucial consequence: it is not possible to reliably characterize the shallow layers in terms of the concepts we use.

In order to improve the explainability of those shallow layers, one could leverage the recent developments in contrastive learning. The purpose of this approach would be to separate the concept set $\mathcal{P}^c$ and $\mathcal{N}^c$ in the representation space $\mathcal{H}$ corresponding to a shallow layer of the neural network. A practical way to implement this would be to follow \cite{Chen2020b}. Assume that we want to separate concept positives and negatives in the representation space $\mathcal{H}$ induced by the shallow feature extractor $\g : \mathcal{X} \rightarrow \mathcal{H}$. As \cite{Chen2020b}, one can use a projection head $\textbf{p} : \mathcal{H} \rightarrow \mathcal{Z}$ and enforce the separation of the concept sets through the contrastive loss

\begin{align}
	\mathcal{L}^c_{\mathrm{cont}} = \sum_{(\x_i,\x_j) \in (\mathcal{P}^c)^2} -\log \frac{\exp( \tau^{-1} \cdot\cos[\textbf{p} \circ \g (\x_i), \textbf{p} \circ \g (\x_j)])}{\sum_{\x_k \in (\mathcal{P}^c \cup \mathcal{N}^c) \setminus \{ \x_i \}} \exp(\tau^{-1} \cdot\cos[\textbf{p} \circ \g (\x_i), \textbf{p} \circ \g (\x_k)])},
\end{align}	

where $\cos(\textbf{z}_1, \textbf{z}_2) \equiv \frac{\textbf{z}_1^{\intercal}\textbf{z}_2}{\| \textbf{z}_1 \|_2 \cdot \| \textbf{z}_2 \|_2}$ and $\tau \in \mathbb{R}^+$ is a temperature parameter. The effect of this loss is to group the concept positive examples from $\mathcal{P}^c$ together and apart from the concept negatives $\mathcal{N}^c$ in the representation space $\mathcal{H}$. To the best of our knowledge, concept-based contrastive learning has not been explored in the literature. We believe that it would constitute an interesting contribution to the field based on the insights from our paper.
		
	\section{Further Examples} 
	\label{appendix:further_examples}
	In this appendix, we provide several examples to illustrate the experiments from Section~\ref{sec:experiments}.

\textbf{Concept classifiers.} The accuracy of the concept classifiers for all the MNIST and ECG concepts are given in Figures~\ref{fig:mnist_acc_per_concept}~and~\ref{fig:ecg_acc_per_concept}. Each box-plot is built with 10 random seeds where the concept sets $\Pos^c$ and $\Neg^c$ are allowed to vary. We observe that the CAR classifiers are more accurate for each of the observed concepts 

\textbf{Global explanations.} The global concept explanations for all the MNIST concepts and some of the CUB concepts are given in Figures~\ref{fig:mnist_tcar_extra}~and~\ref{fig:cub_tcar_extra}. As the examples from Section~\ref{subsec:car_global_validation}, we see that TCAR explanations are more consistent with the human concept annotations.

\textbf{Saliency maps.} Examples of concept-based an vanilla saliency maps for MNIST, ECG and CUB examples are given in Figures~\ref{fig:mnist_saliency}~,~\ref{fig:ecg_saliency}~,~\ref{fig:cub_saliency_belly}~,~\ref{fig:cub_saliency_back}~and~\ref{fig:cub_saliency_breast}. As explained in the main paper, we observe that concept-based saliency maps are indeed distinct form vanilla saliency maps. Furthermore, saliency maps for different concepts are not interchangeable. In the CUB case, we note that concept are not always identified with the minimal amount of features (e.g. some of the saliency maps from Figure~\ref{fig:cub_saliency_breast} highlight pixels that do not always belong to the bird's breast). This surprising observation seems to occur even for concept-bottleneck models that are explicitly trained to recognize concepts \cite{Margeloiu2021}. We believe that this could be improved by training concept classifiers with images that include a segmentation highlighting the concept of interest (e.g. the bird's breast). We leave this idea for future works.   

\begin{figure}[!ht]
	\centering
	\begin{subfigure}{0.45\linewidth}
		\includegraphics[width=\linewidth]{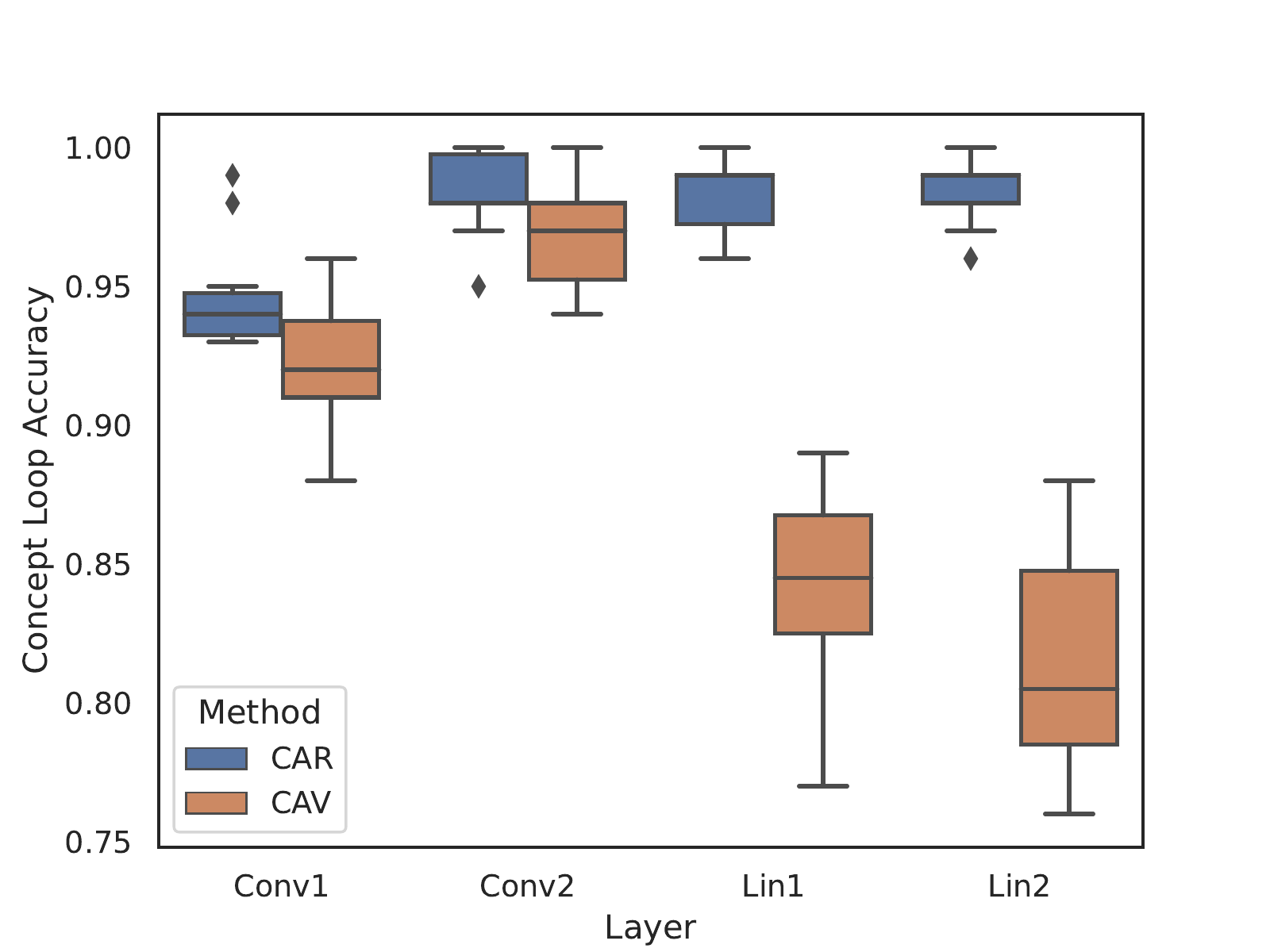}
		\caption{Loop concept}
	\end{subfigure}
	\hfil
	\begin{subfigure}{0.45\linewidth}
		\includegraphics[width=\linewidth]{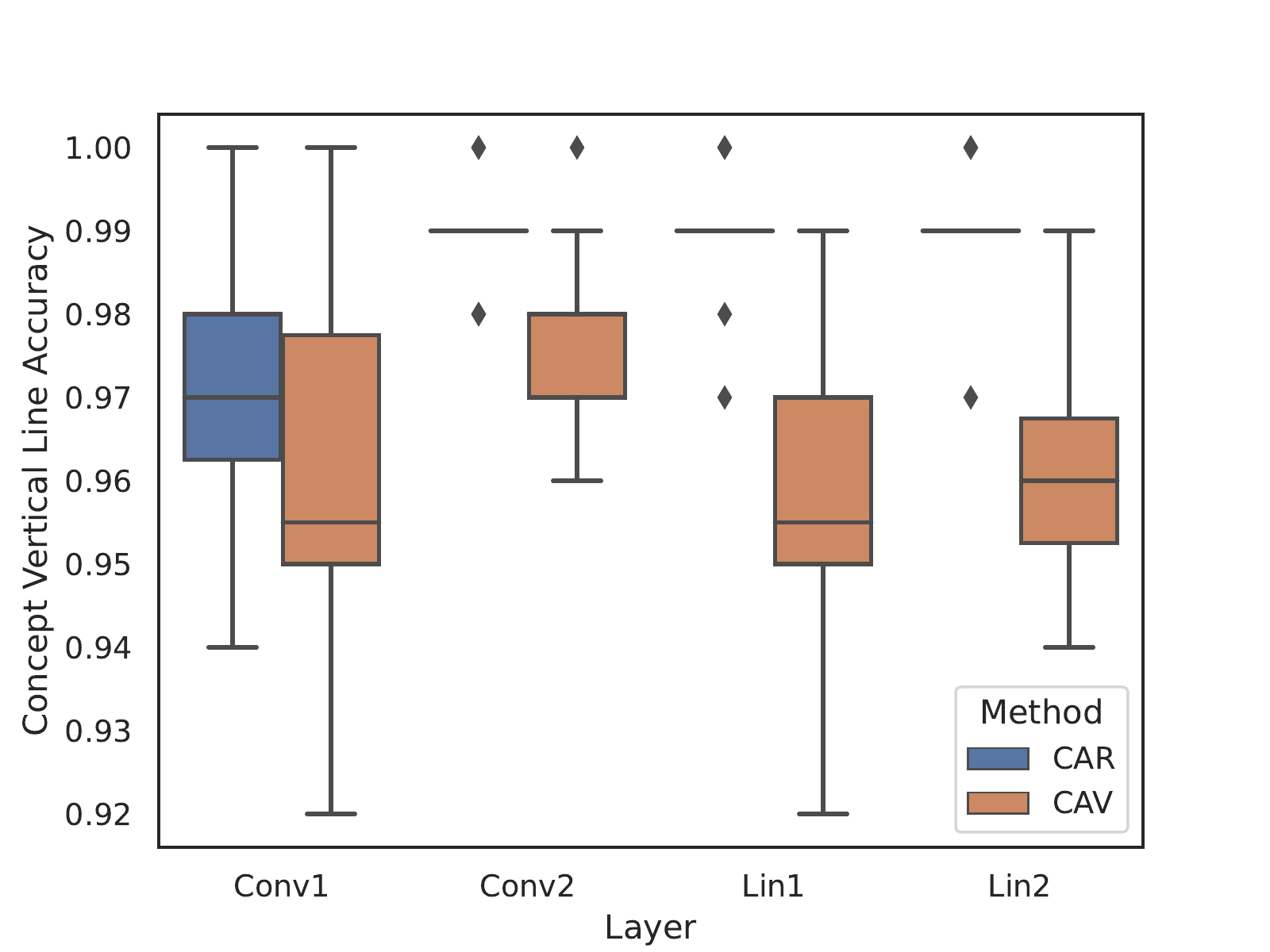}
		\caption{Vertical line concept}
	\end{subfigure}
	\begin{subfigure}{0.45\linewidth}
		\includegraphics[width=\linewidth]{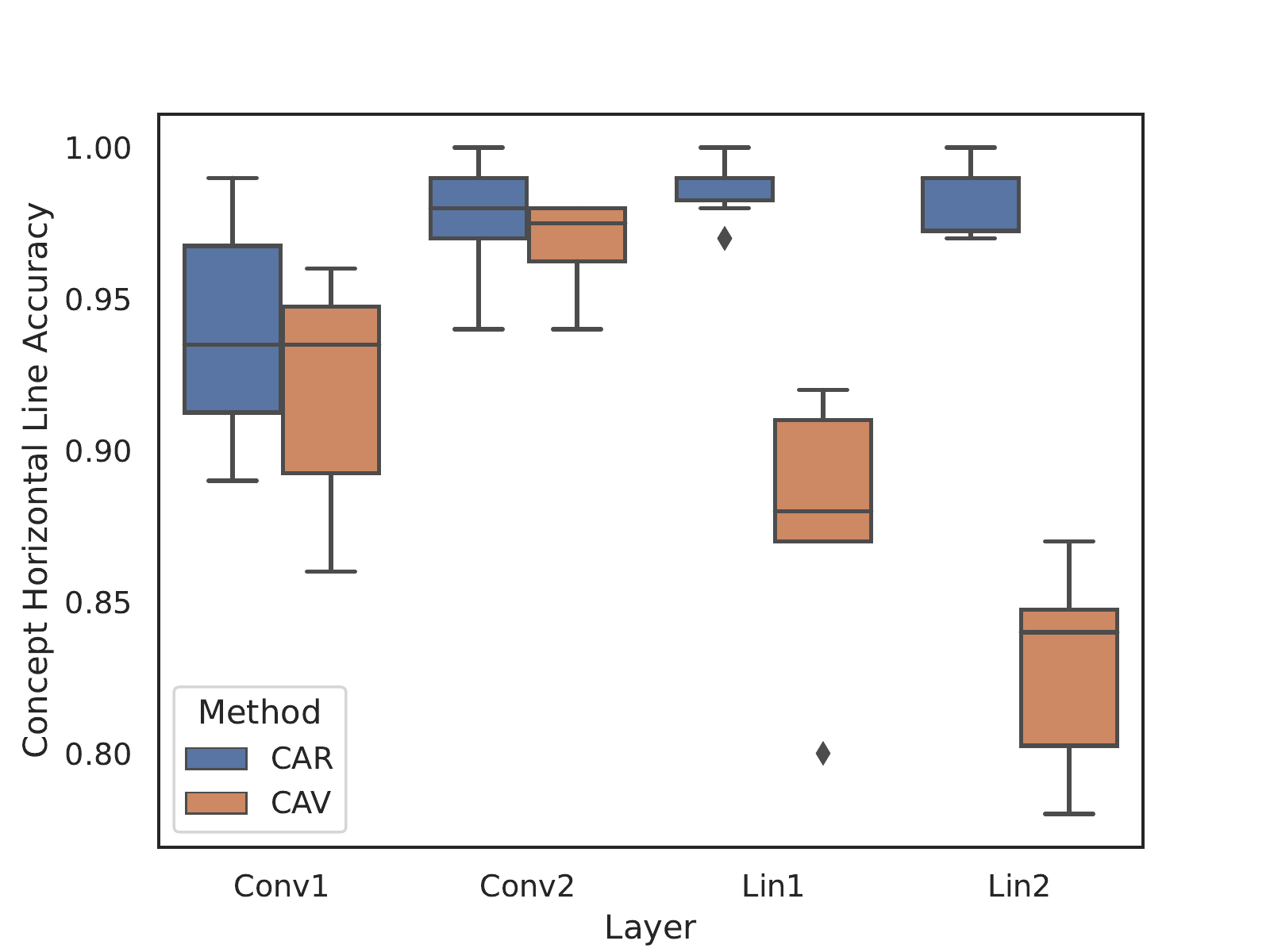}
		\caption{Horizontal line concept}
	\end{subfigure}
	\hfil
	\begin{subfigure}{0.45\linewidth}
		\includegraphics[width=\linewidth]{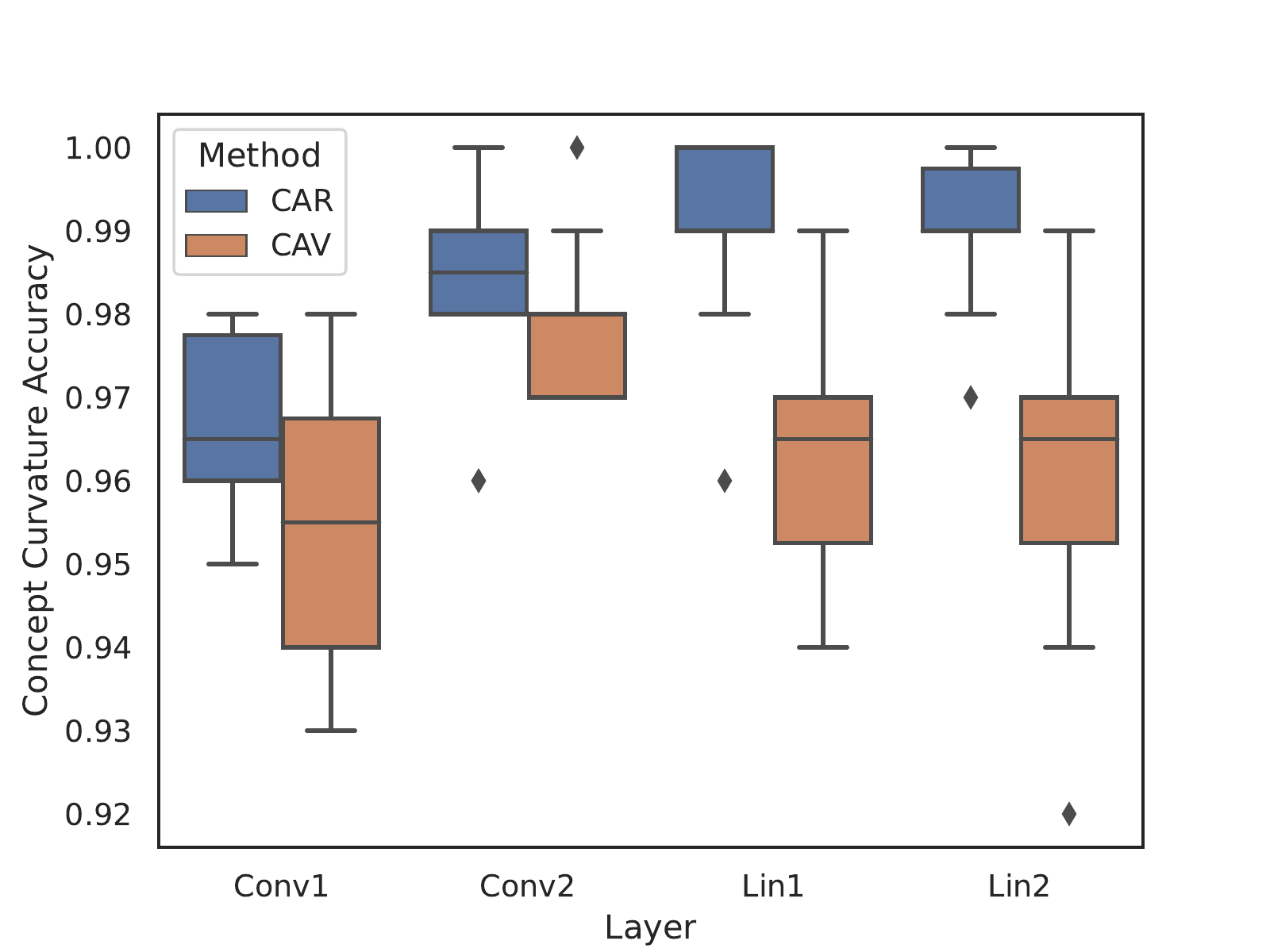}
		\caption{Curvature concept}
	\end{subfigure}
	
	\caption{Concept accuracy for MNIST concepts}
	\label{fig:mnist_acc_per_concept}
\end{figure}

\begin{figure}[!ht]
	\centering
	\begin{subfigure}{0.45\linewidth}
		\includegraphics[width=\linewidth]{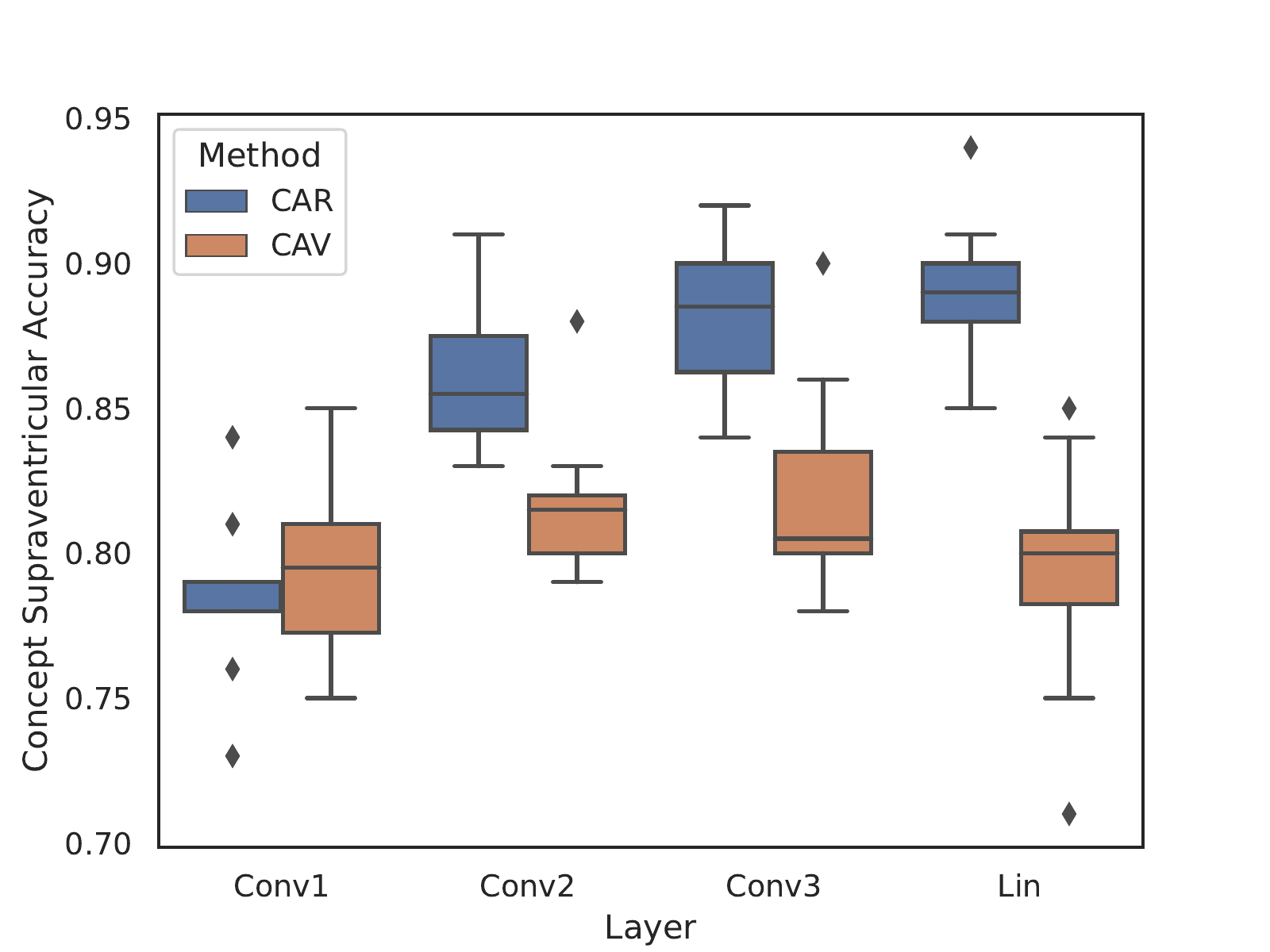}
		\caption{Supraventricular concept}
	\end{subfigure}
	\hfil
	\begin{subfigure}{0.45\linewidth}
		\includegraphics[width=\linewidth]{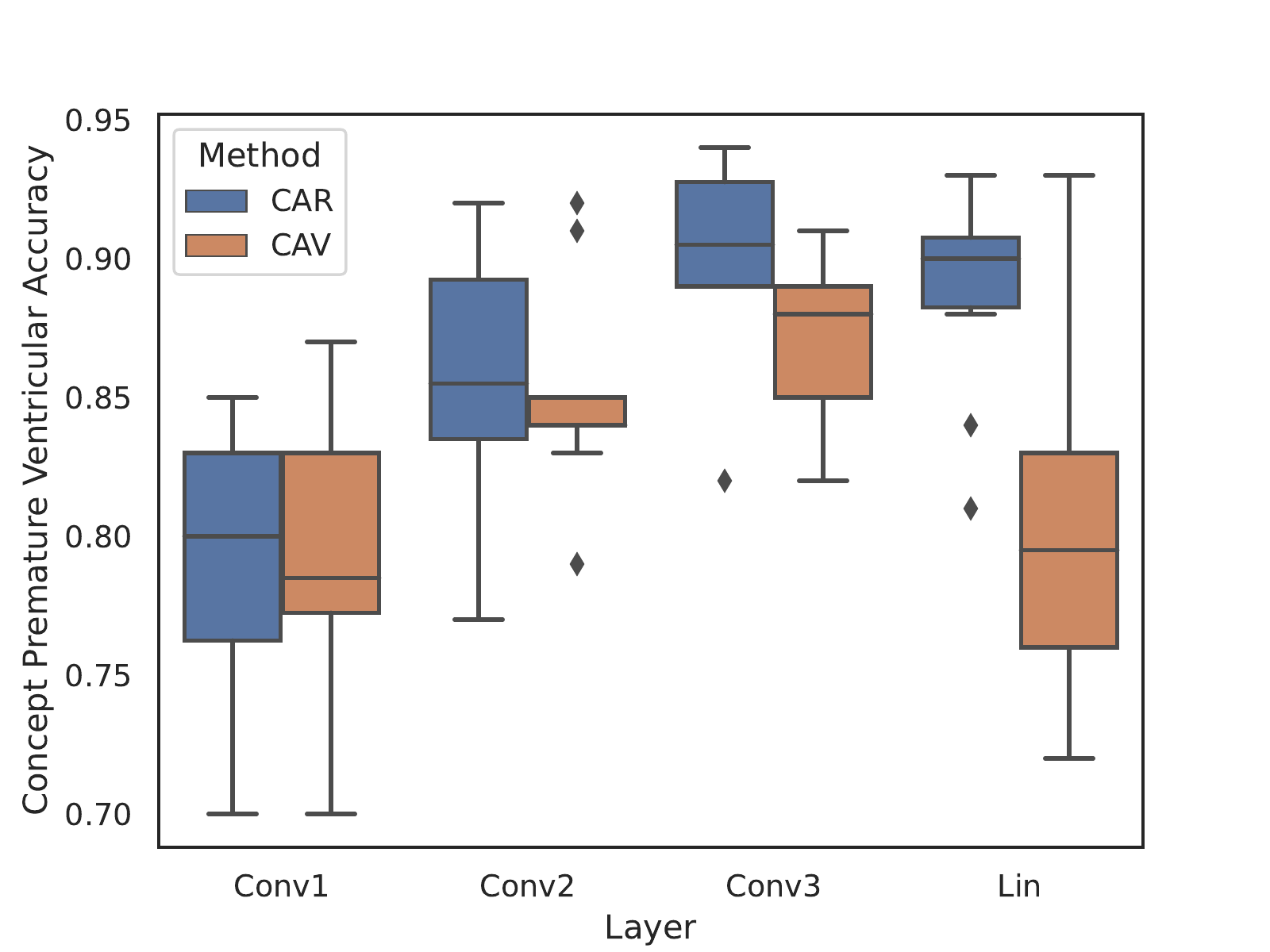}
		\caption{Premature ventricular concept}
	\end{subfigure}
	\begin{subfigure}{0.45\linewidth}
		\includegraphics[width=\linewidth]{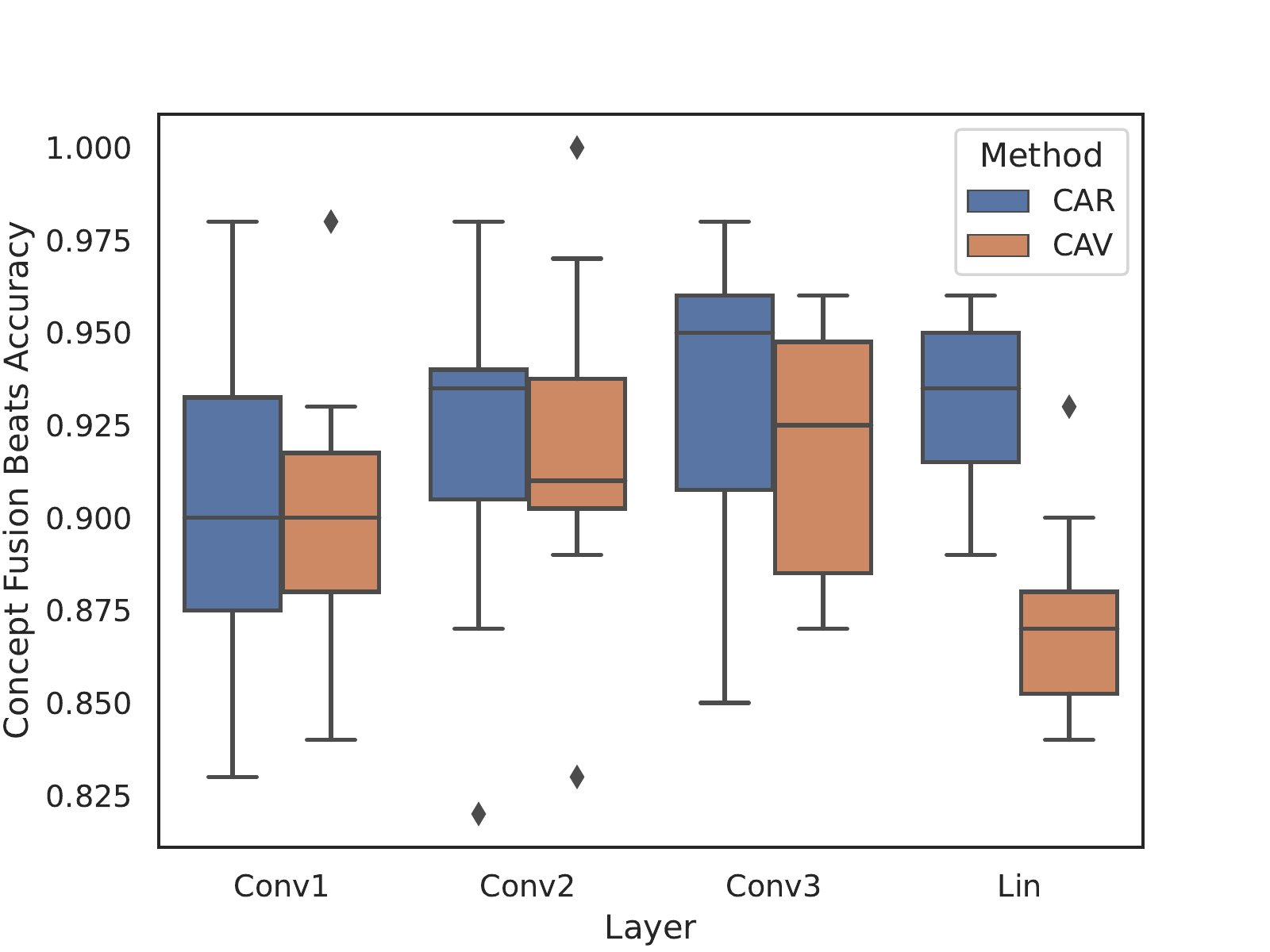}
		\caption{Fusion beats concept}
	\end{subfigure}
	\hfil
	\begin{subfigure}{0.45\linewidth}
		\includegraphics[width=\linewidth]{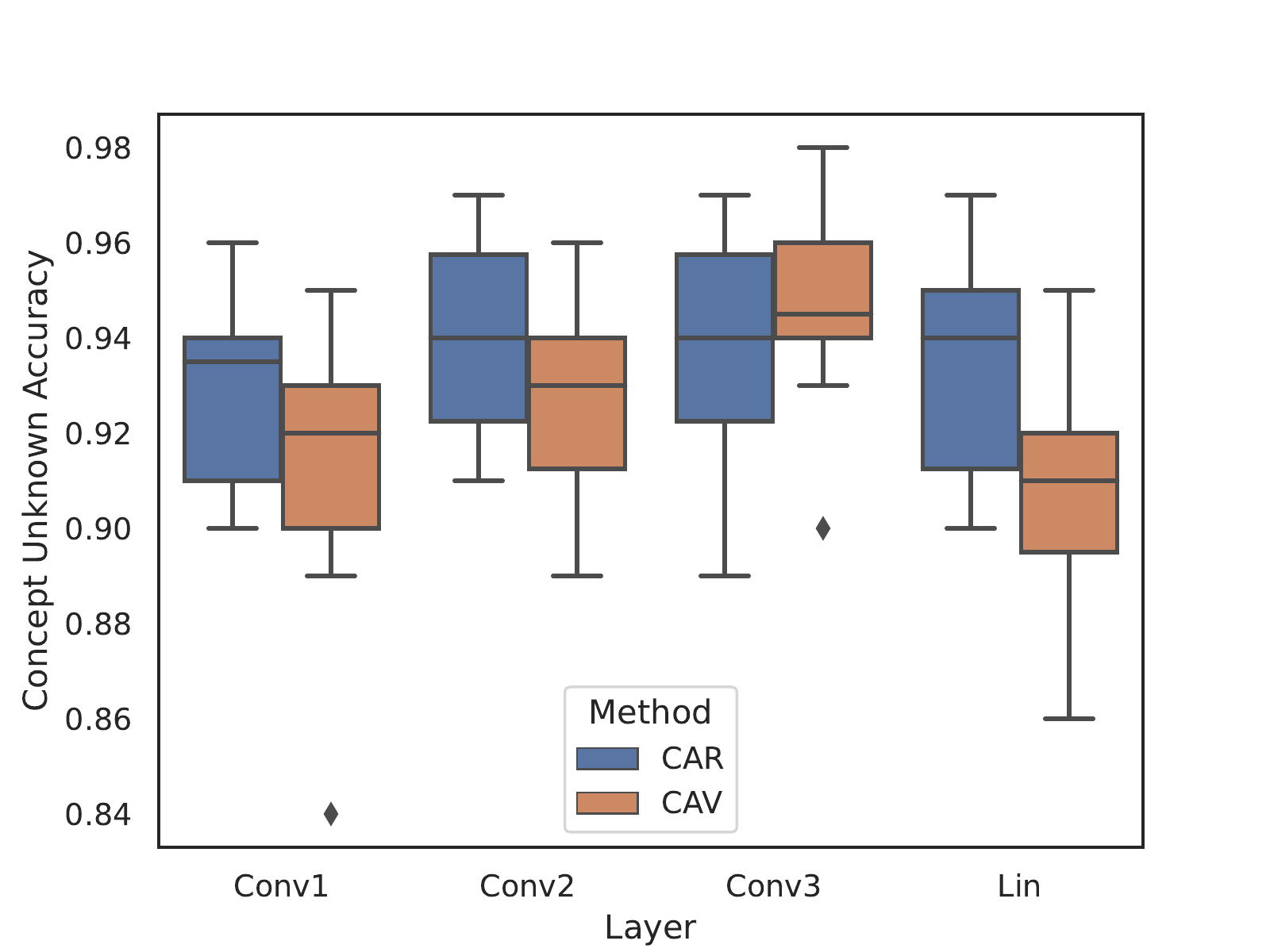}
		\caption{Unknown beat concept}
	\end{subfigure}

	\caption{Concept accuracy for ECG concepts}
	\label{fig:ecg_acc_per_concept}
\end{figure}

\begin{figure}[!ht]
	\centering
	\begin{subfigure}{0.32\linewidth}
		\includegraphics[width=\linewidth]{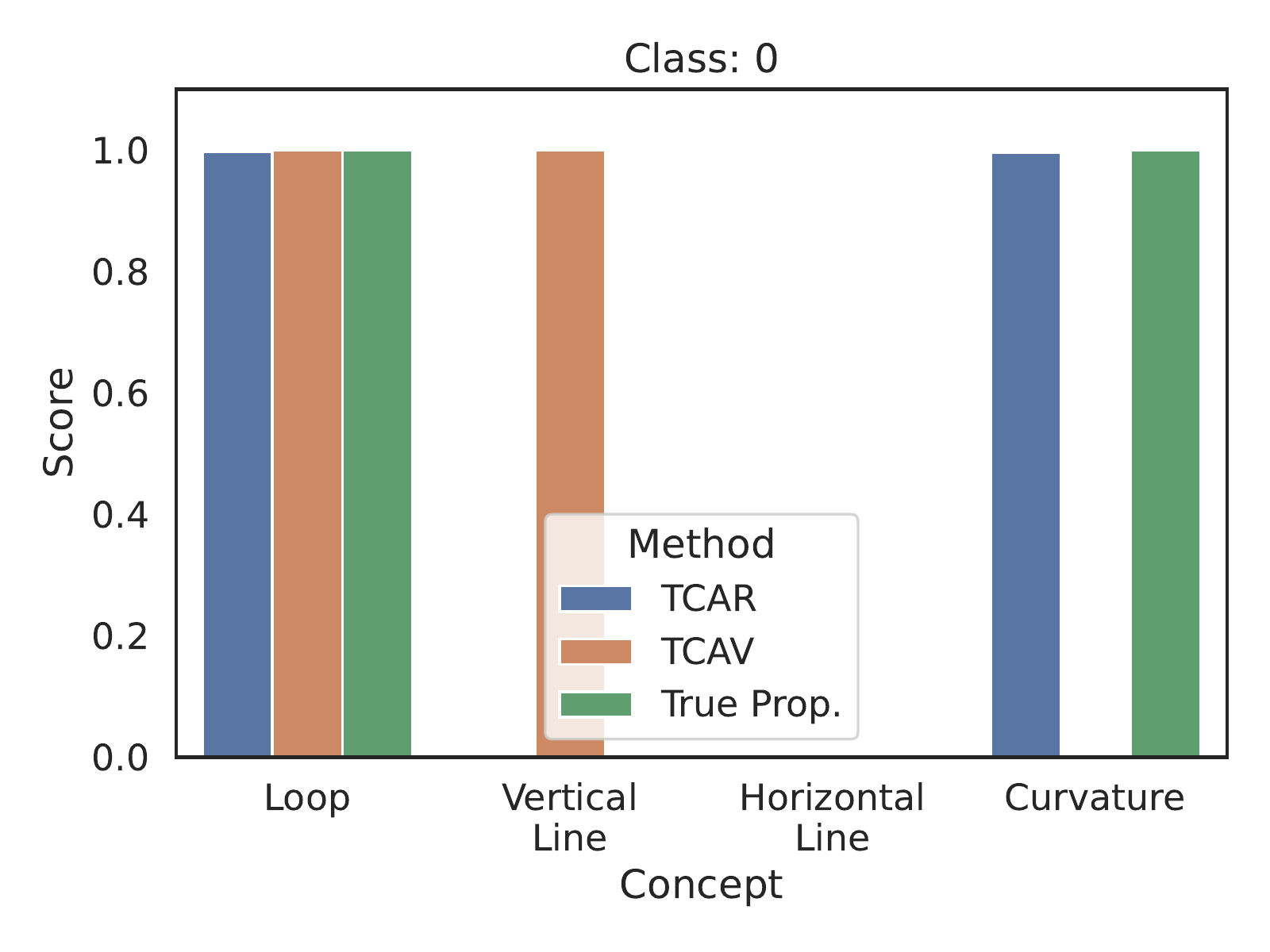}
		\caption{Class 0}
	\end{subfigure}
	\hfil
	\begin{subfigure}{0.32\linewidth}
		\includegraphics[width=\linewidth]{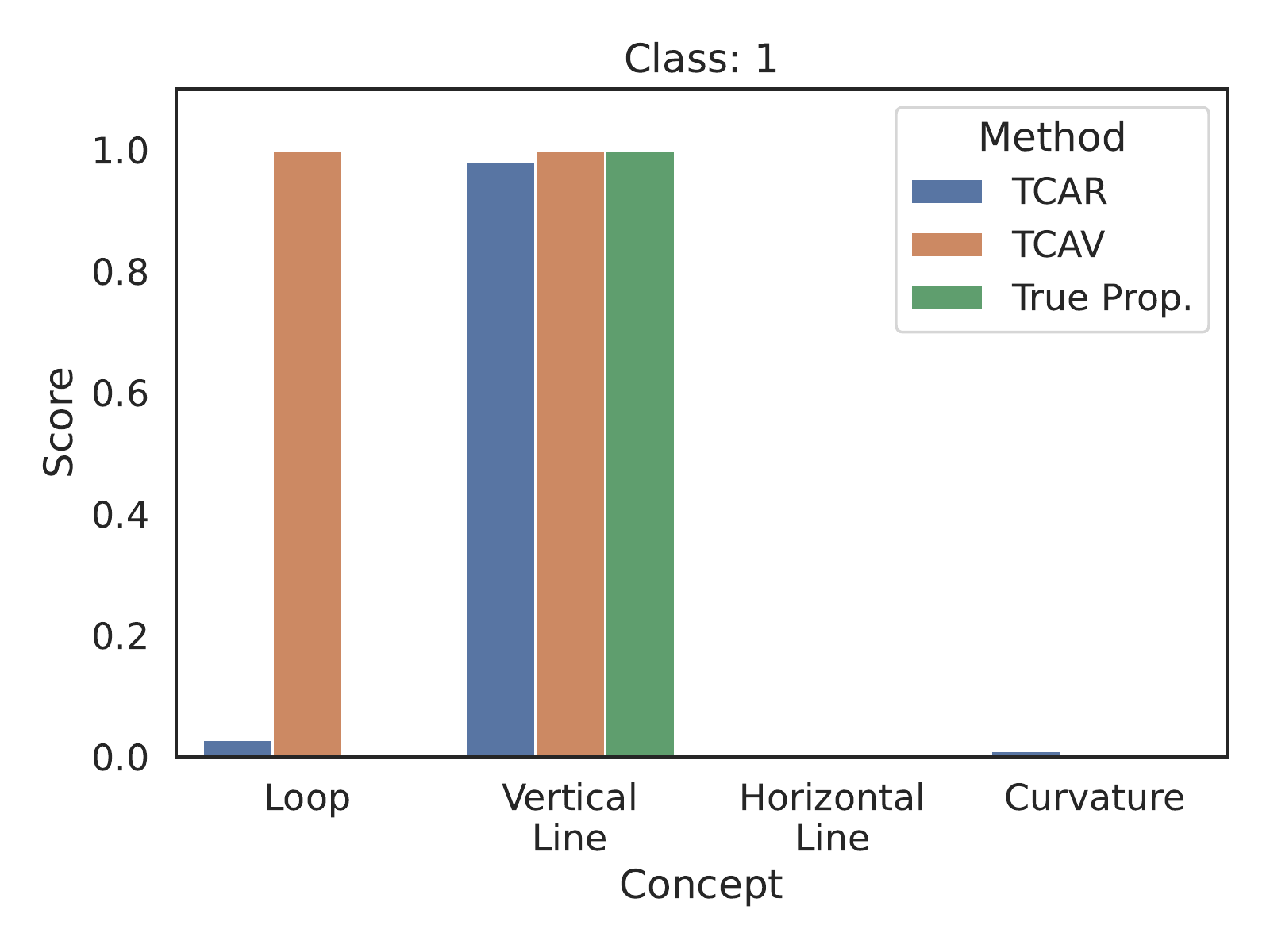}
		\caption{Class 1}
	\end{subfigure}
	\begin{subfigure}{0.32\linewidth}
		\includegraphics[width=\linewidth]{Figures/mnist_global_class2}
		\caption{Class 2}
	\end{subfigure}
	\hfil
	\begin{subfigure}{0.32\linewidth}
		\includegraphics[width=\linewidth]{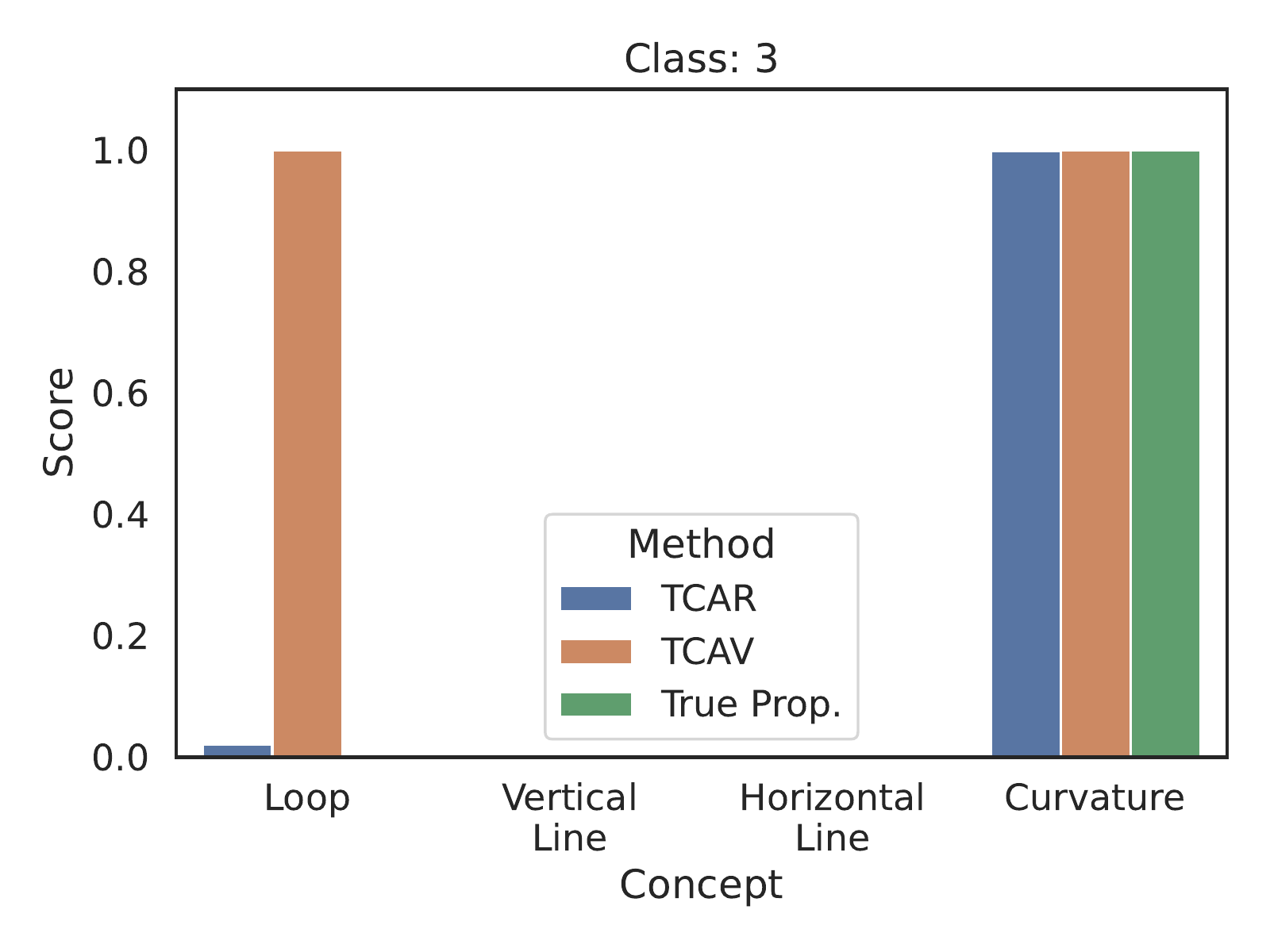}
		\caption{Class 3}
	\end{subfigure}
	\begin{subfigure}{0.32\linewidth}
		\includegraphics[width=\linewidth]{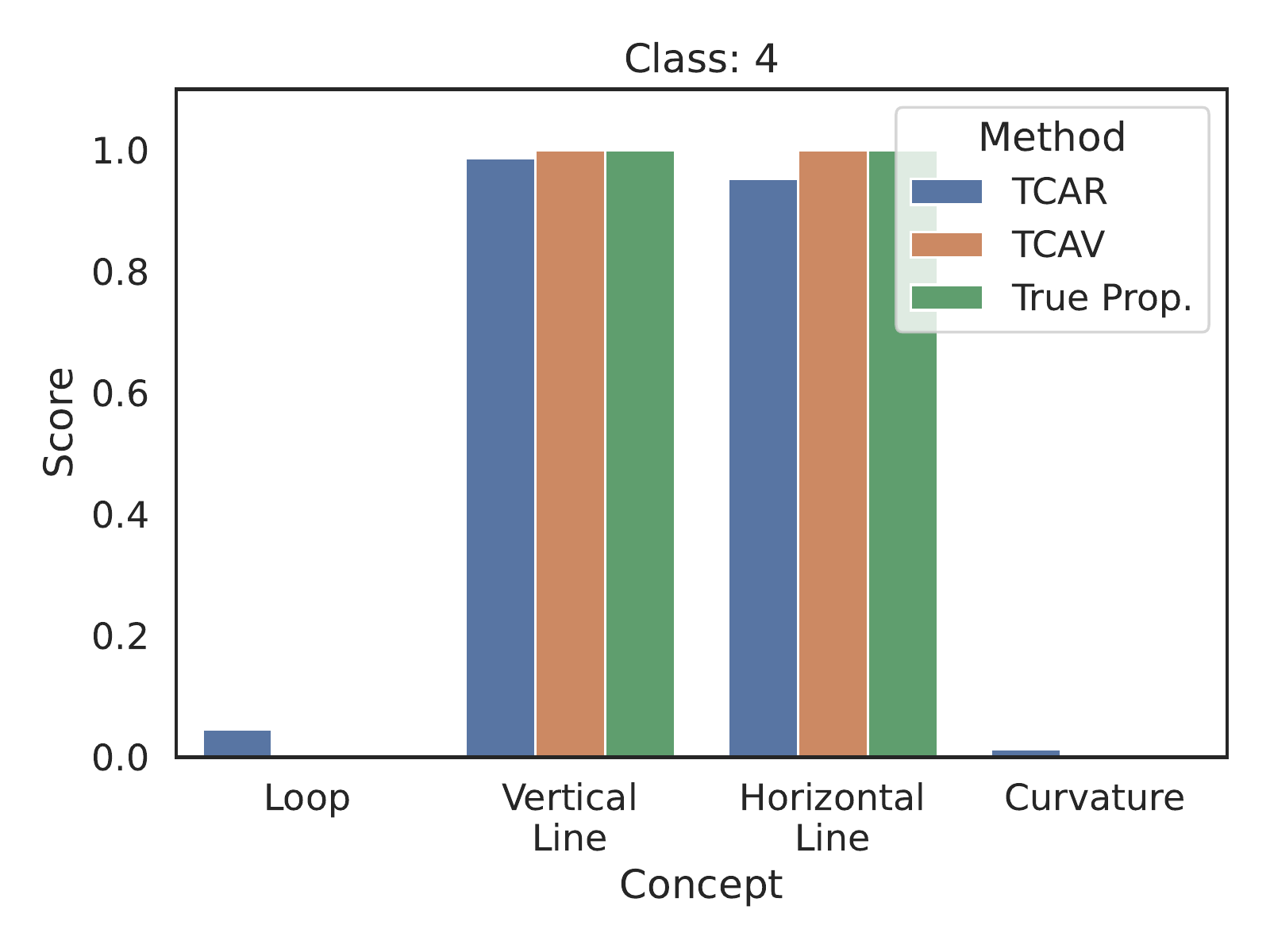}
		\caption{Class 4}
	\end{subfigure}
	\hfil
	\begin{subfigure}{0.32\linewidth}
		\includegraphics[width=\linewidth]{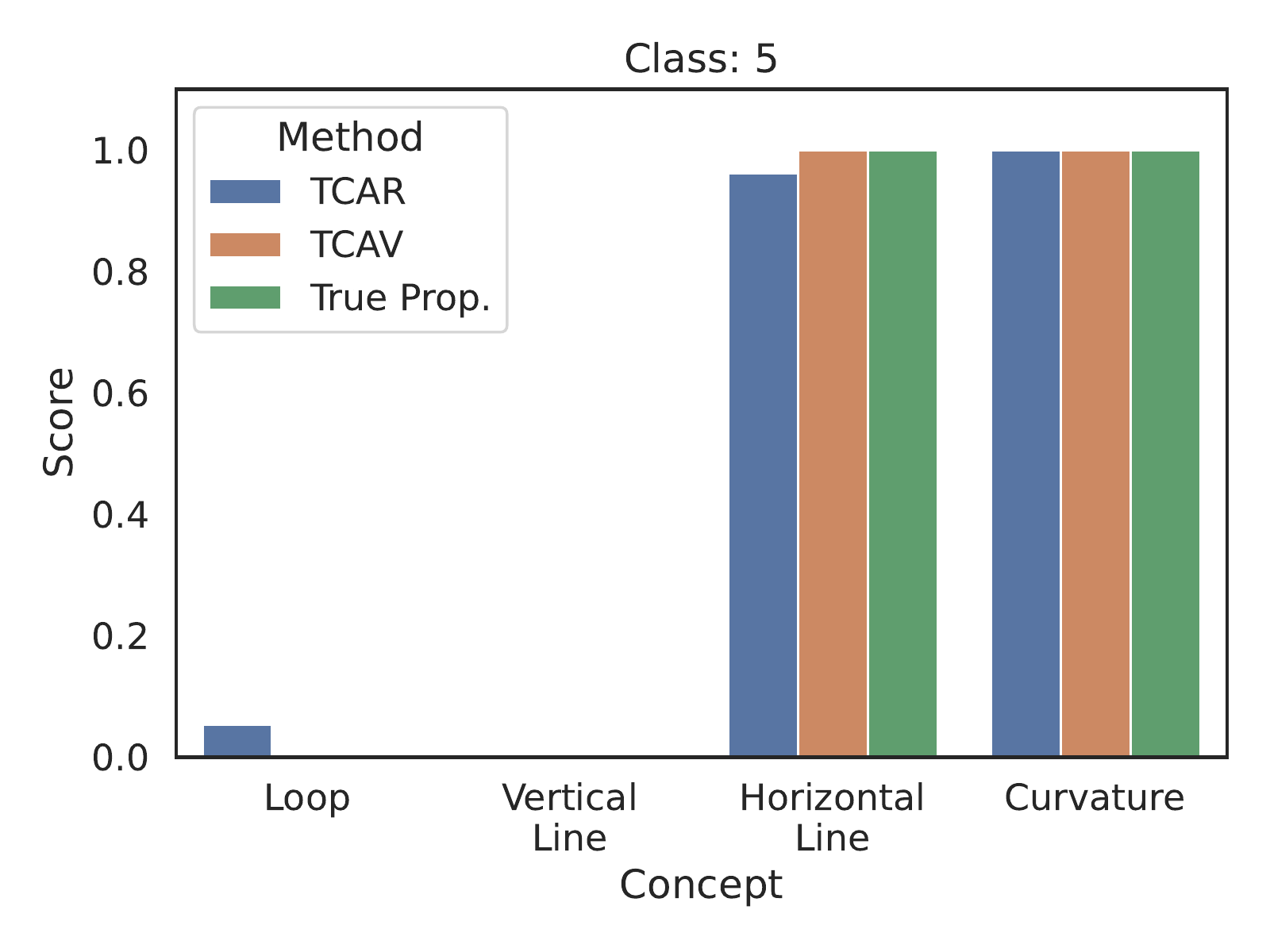}
		\caption{Class 5}
	\end{subfigure}
	\begin{subfigure}{0.32\linewidth}
		\includegraphics[width=\linewidth]{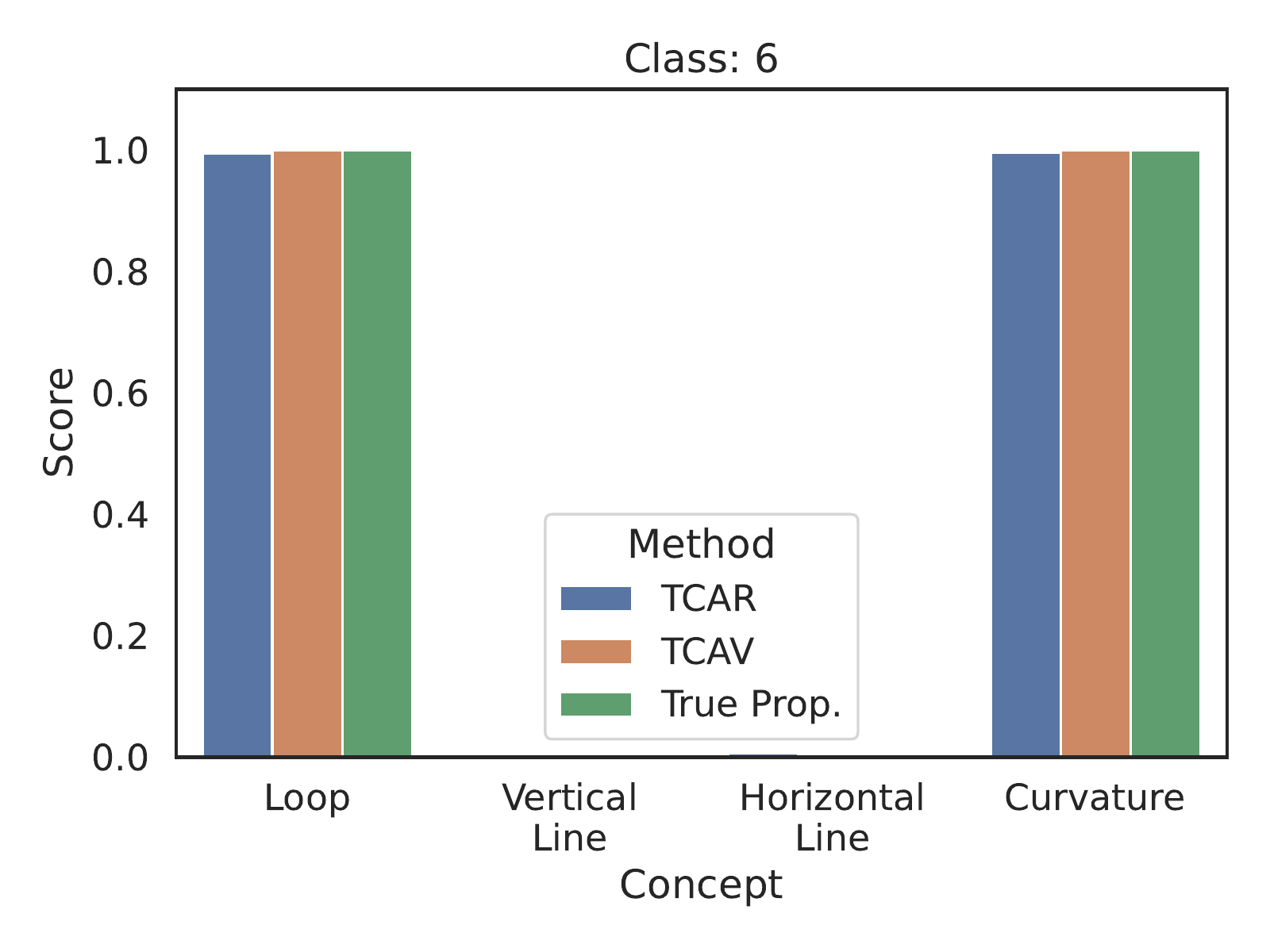}
		\caption{Class 6}
	\end{subfigure}
	\hfil
	\begin{subfigure}{0.32\linewidth}
		\includegraphics[width=\linewidth]{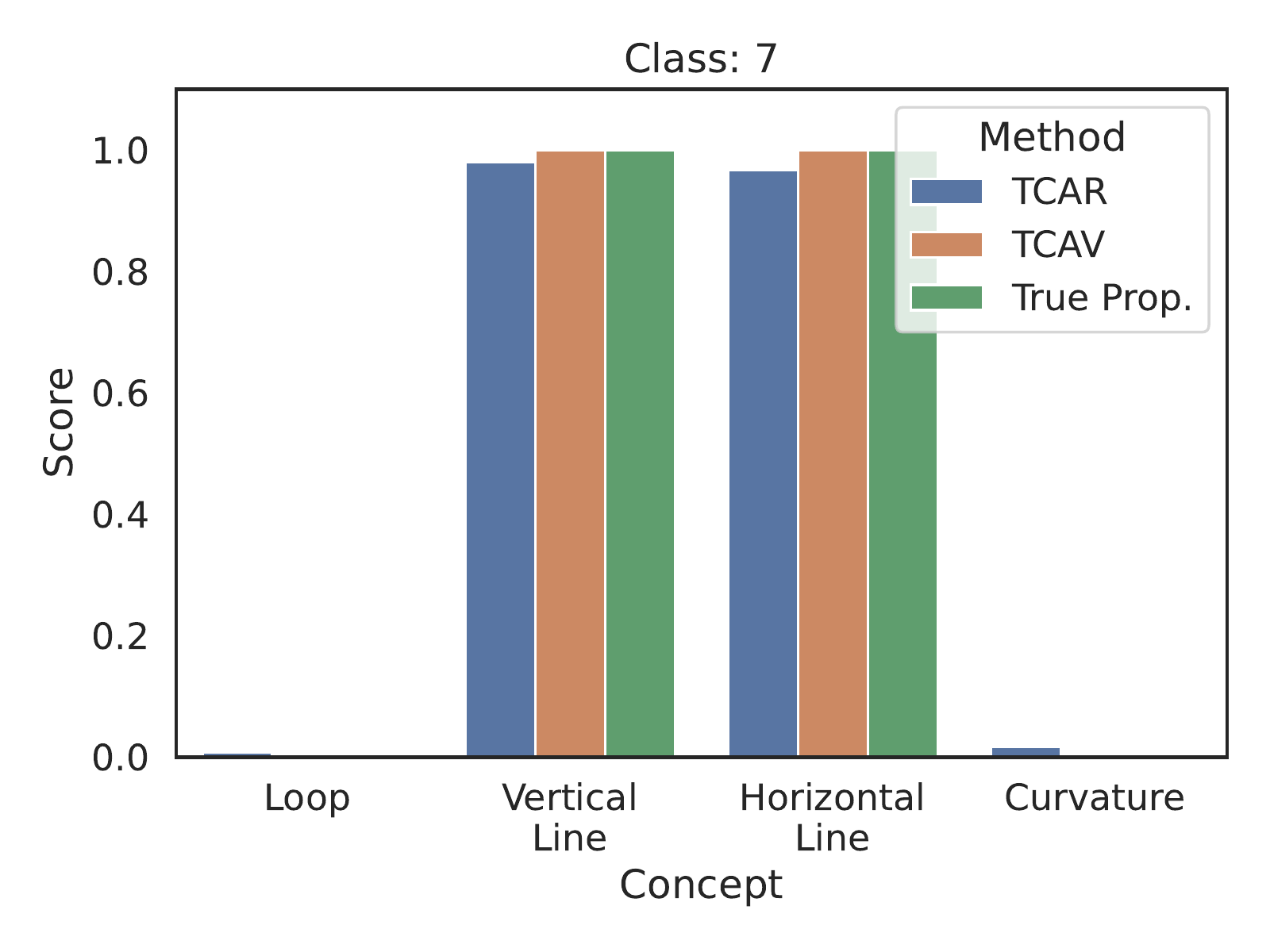}
		\caption{Class 7}
	\end{subfigure}
	\begin{subfigure}{0.32\linewidth}
		\includegraphics[width=\linewidth]{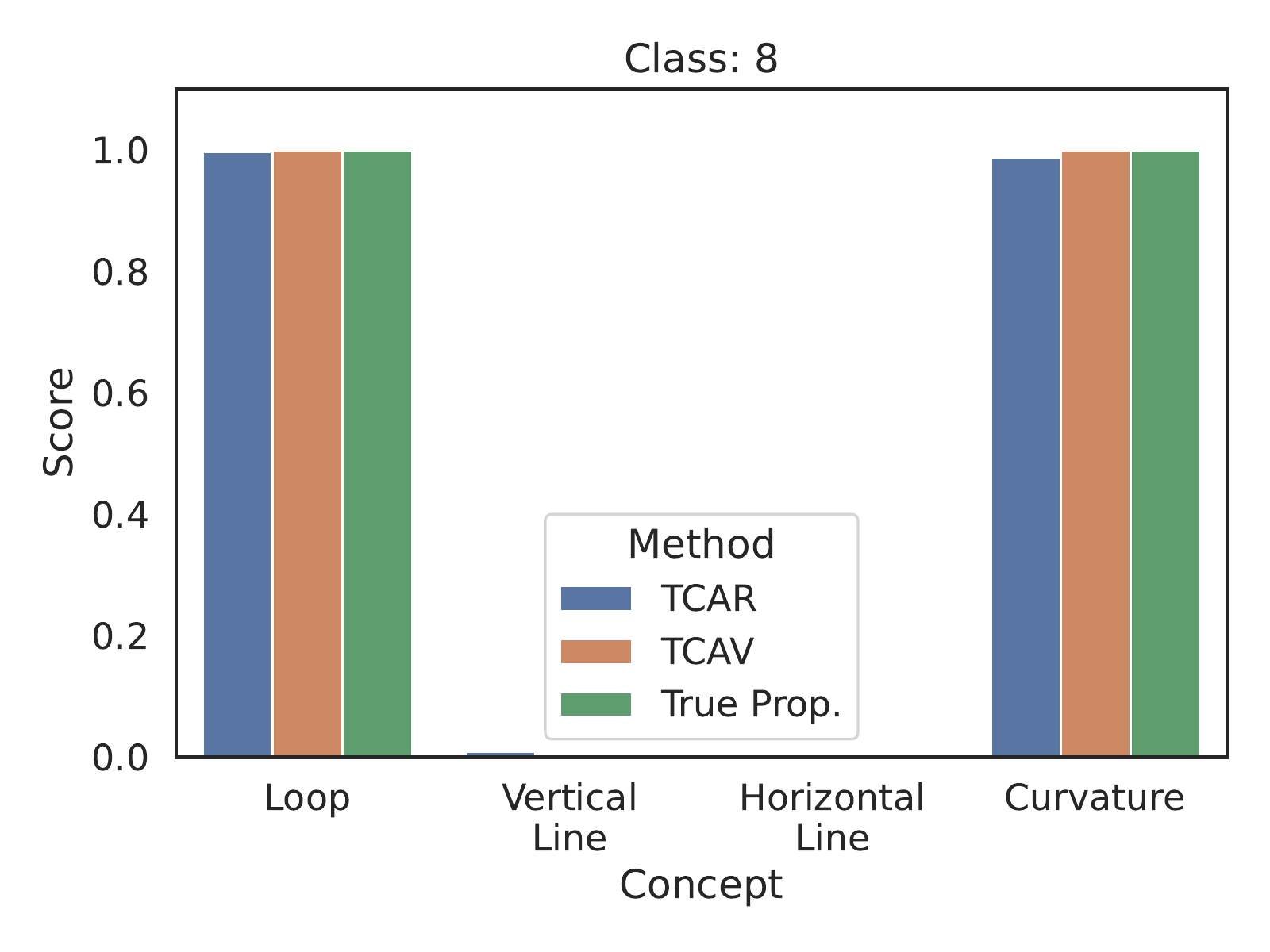}
		\caption{Class 8}
	\end{subfigure}
	\hfil
	\begin{subfigure}{0.32\linewidth}
		\includegraphics[width=\linewidth]{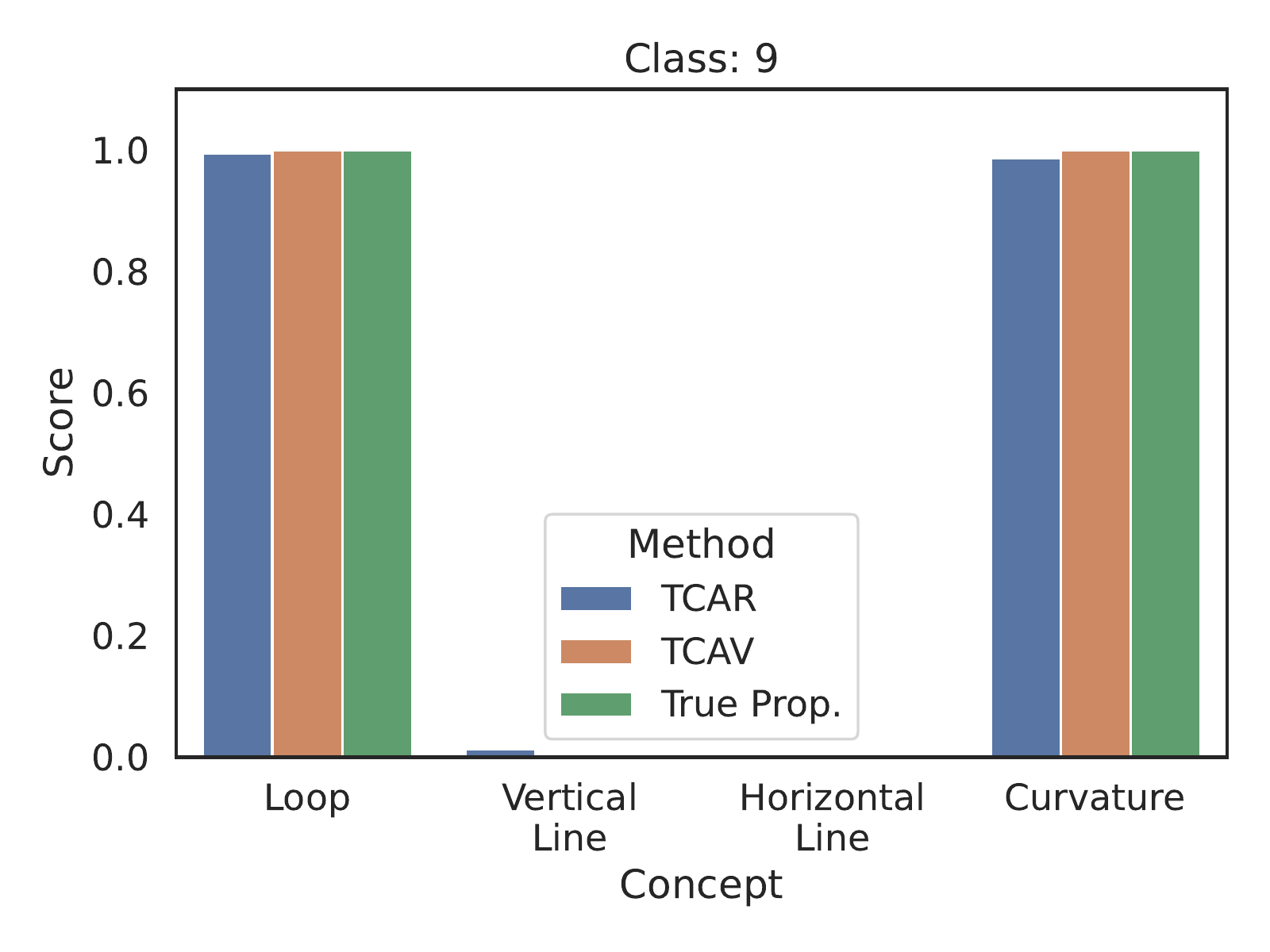}
		\caption{Class 9}
	\end{subfigure}
	
	\caption{Global concept explanations for MNIST}
	\label{fig:mnist_tcar_extra}
\end{figure}

\begin{figure}[!ht]
	\centering
	\begin{subfigure}{0.32\linewidth}
		\includegraphics[width=\linewidth]{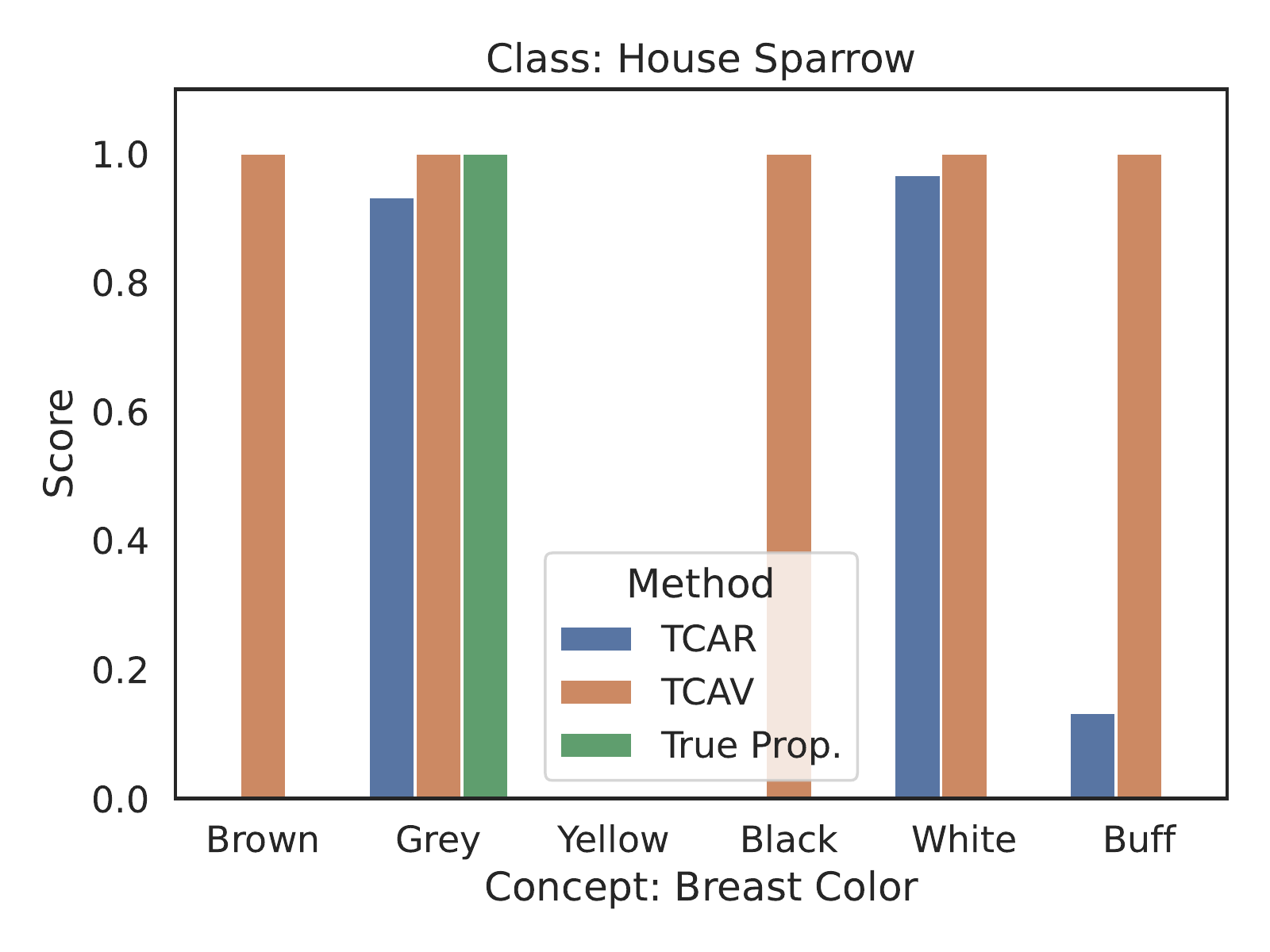}
		\caption{House sparrow breast colour}
	\end{subfigure}
	\hfil
	\begin{subfigure}{0.32\linewidth}
		\includegraphics[width=\linewidth]{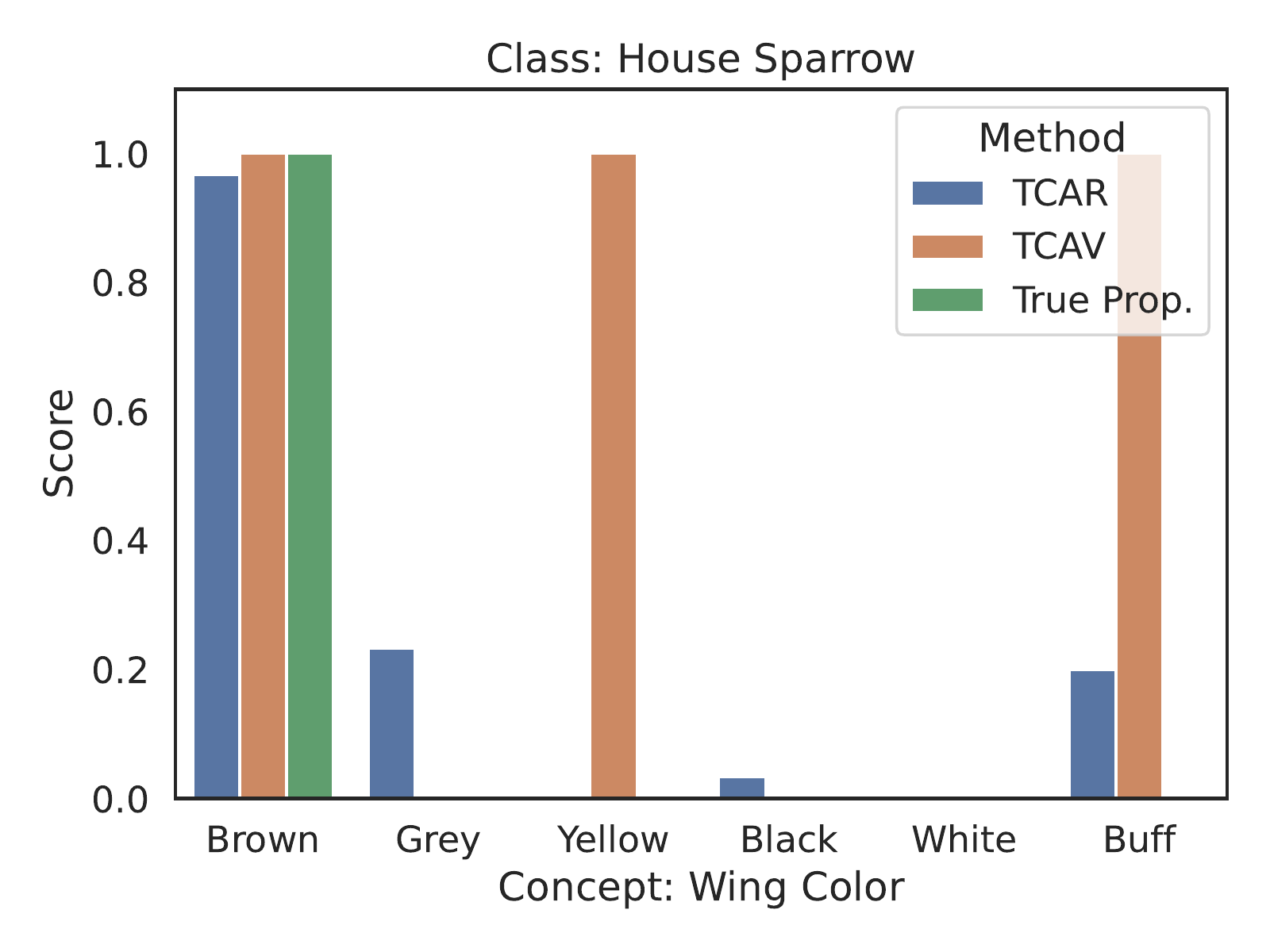}
		\caption{House sparrow wing colour}
	\end{subfigure}
	\hfil
	\begin{subfigure}{0.32\linewidth}
		\includegraphics[width=\linewidth]{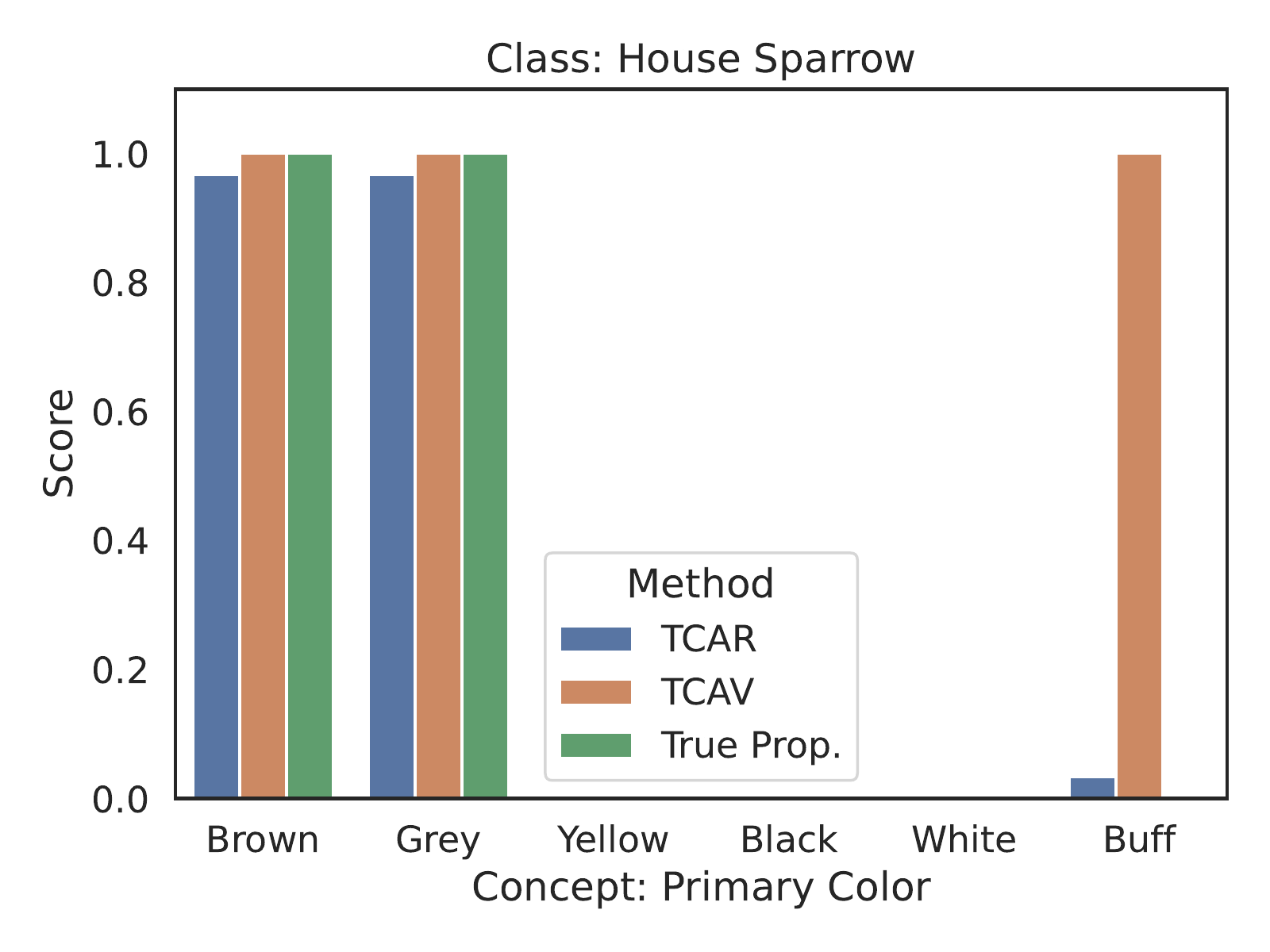}
		\caption{House sparrow primary colour}
	\end{subfigure}
	\begin{subfigure}{0.32\linewidth}
		\includegraphics[width=\linewidth]{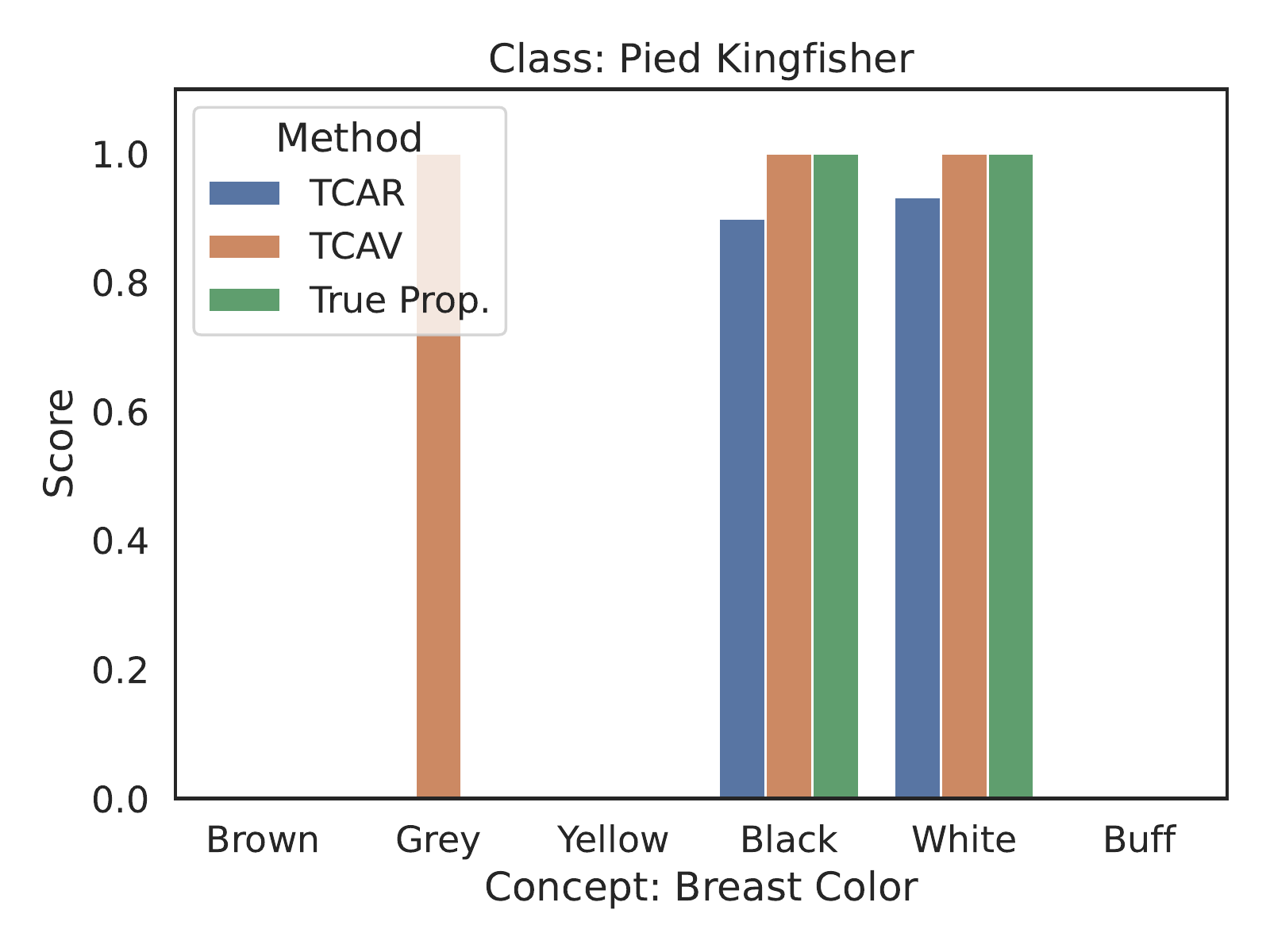}
		\caption{Pied kingfisher breast colour}
	\end{subfigure}
	\hfil
	\begin{subfigure}{0.32\linewidth}
		\includegraphics[width=\linewidth]{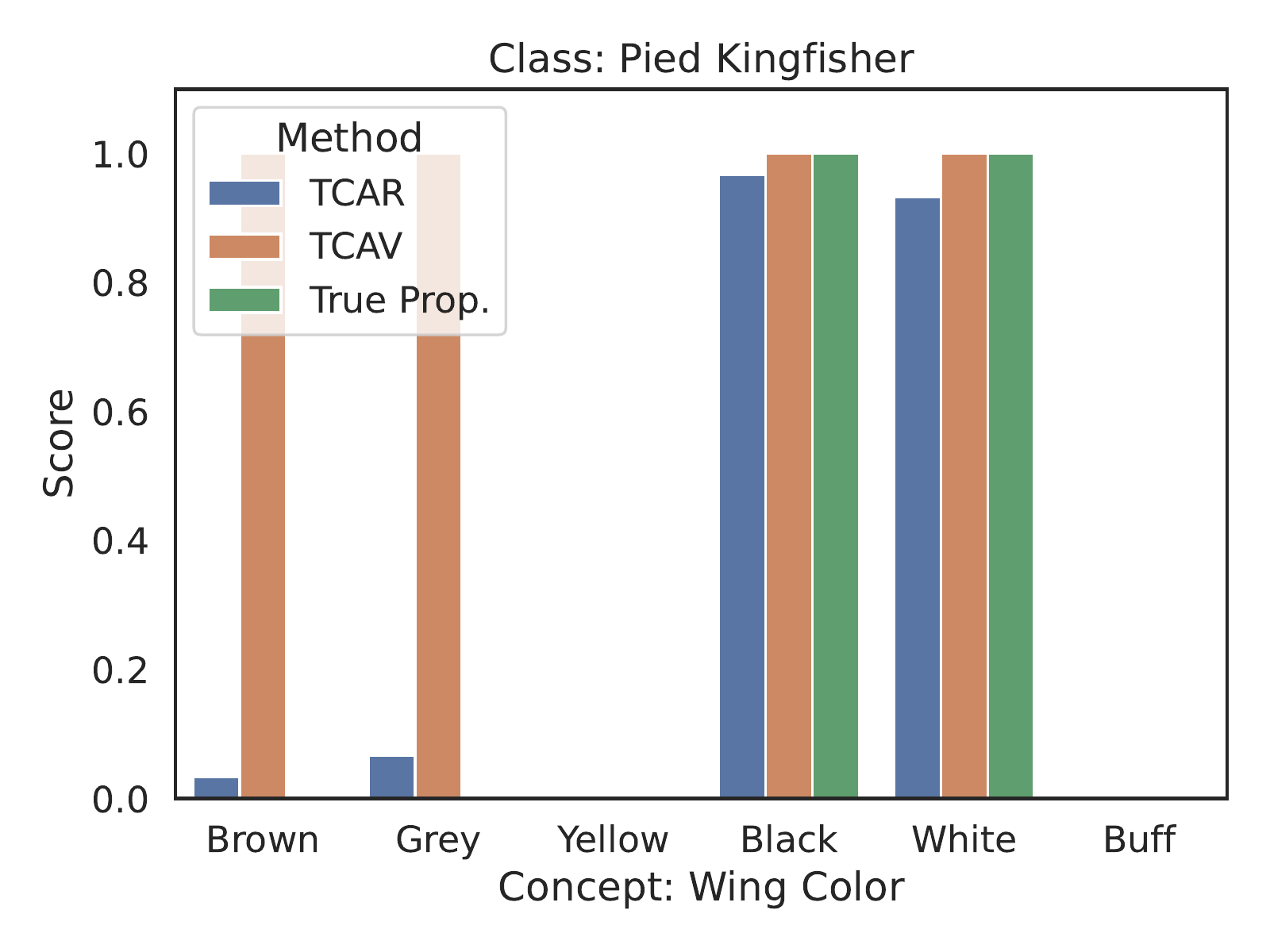}
		\caption{Pied kingfisher wing colour}
	\end{subfigure}
	\hfil
	\begin{subfigure}{0.32\linewidth}
		\includegraphics[width=\linewidth]{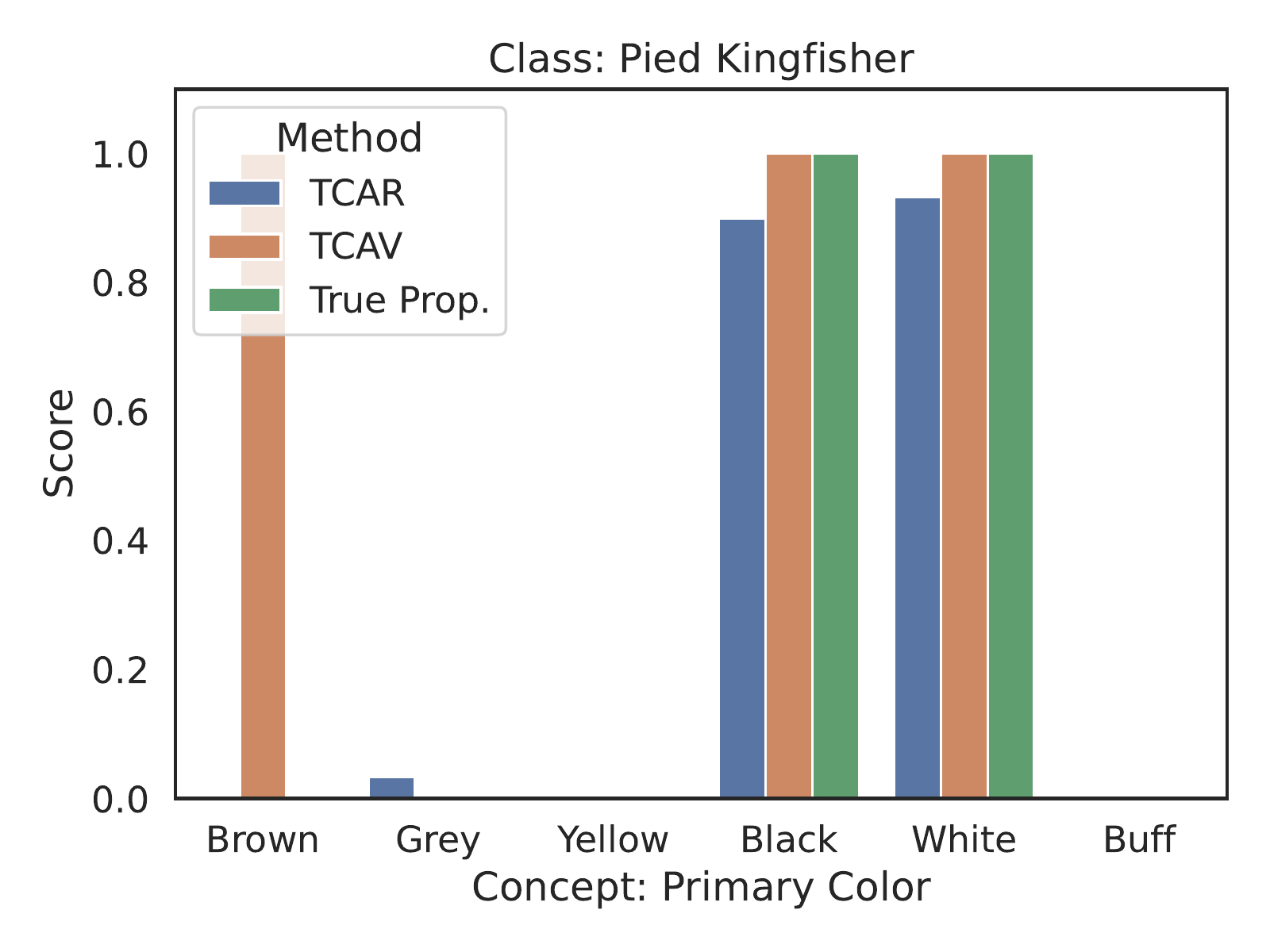}
		\caption{Pied kingfisher primary colour}
	\end{subfigure}
	\begin{subfigure}{0.32\linewidth}
		\includegraphics[width=\linewidth]{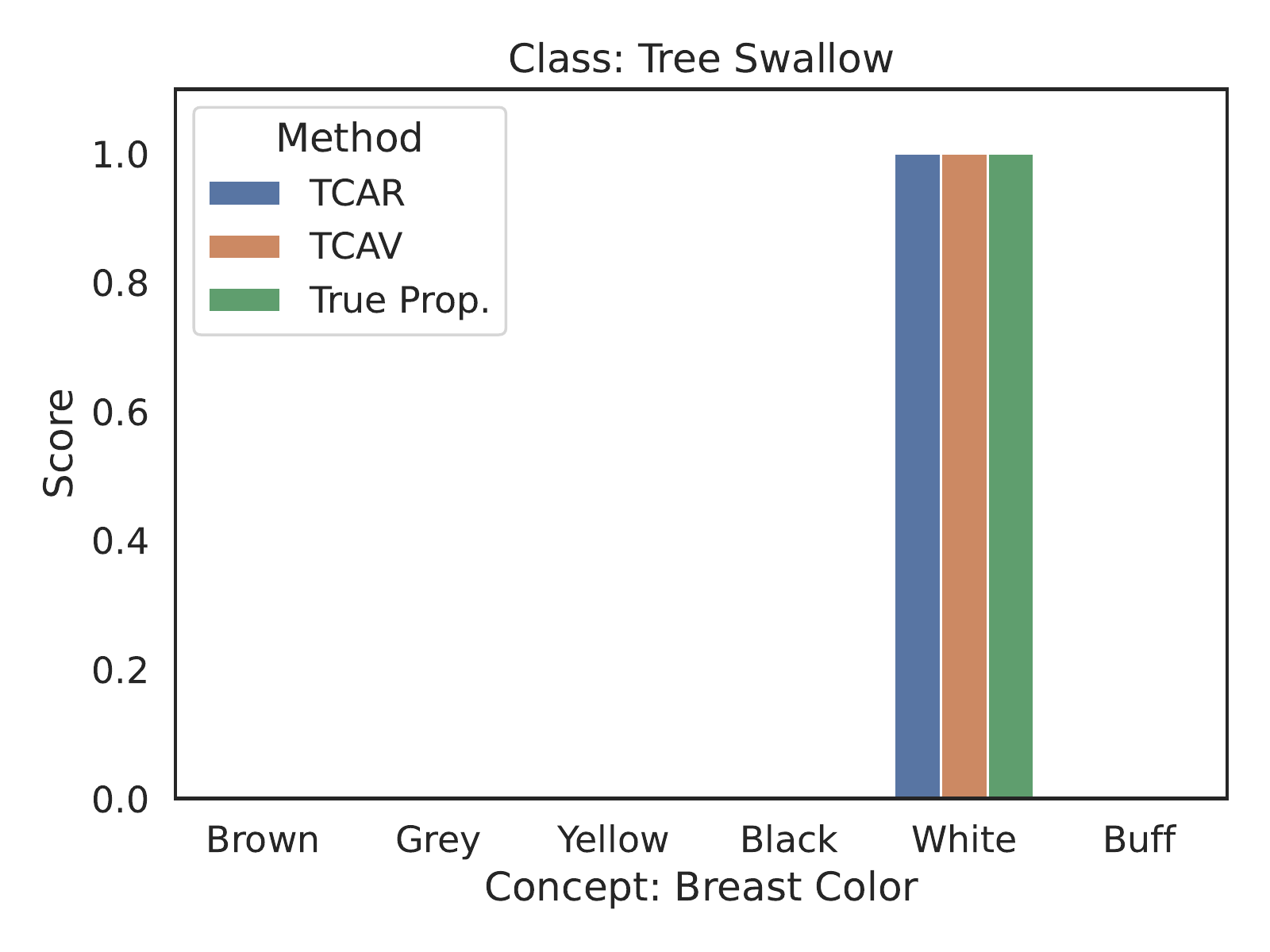}
		\caption{Tree swallow breast colour}
	\end{subfigure}
	\hfil
	\begin{subfigure}{0.32\linewidth}
		\includegraphics[width=\linewidth]{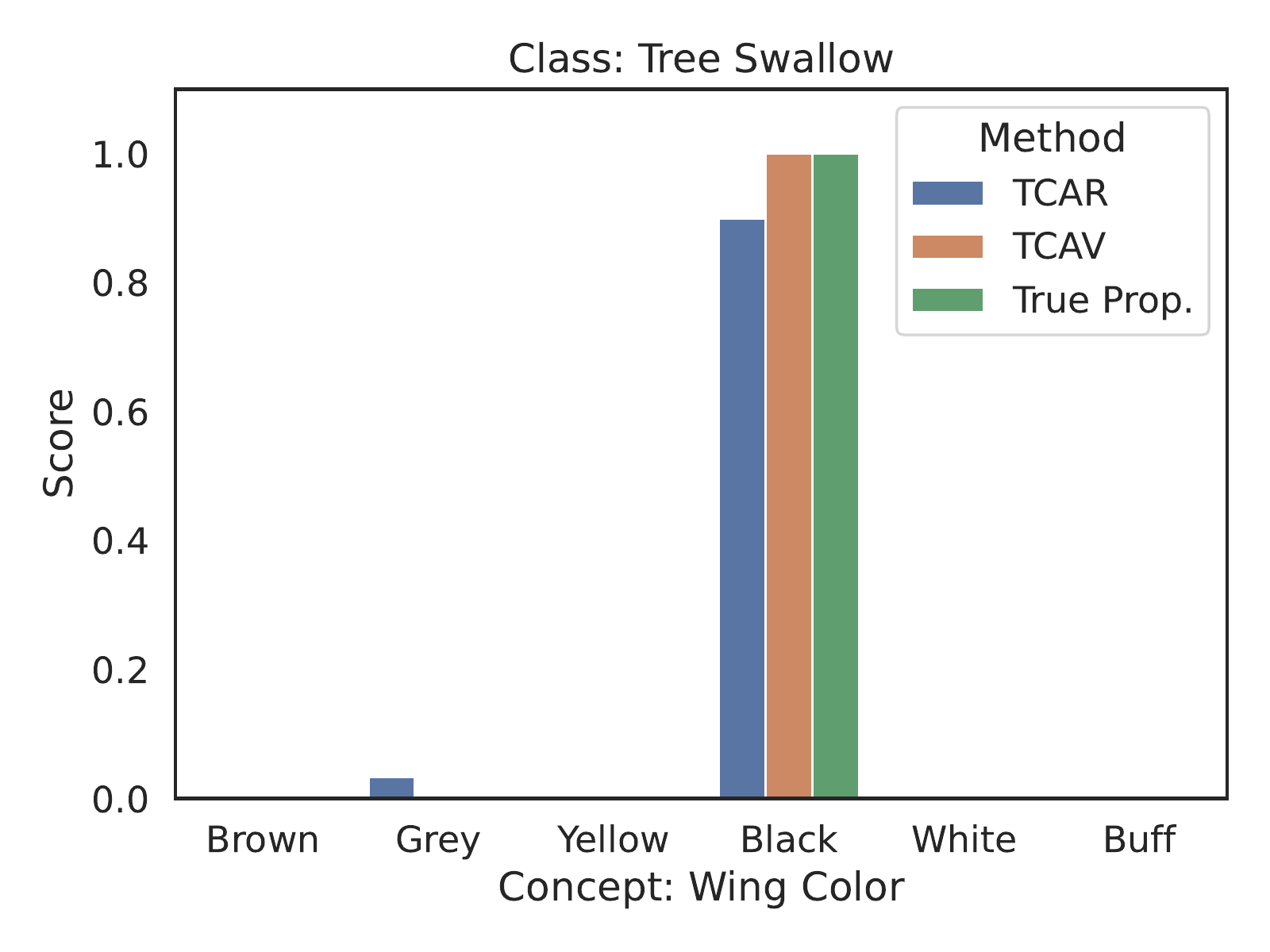}
		\caption{Tree swallow wing colour}
	\end{subfigure}
	\hfil
	\begin{subfigure}{0.32\linewidth}
		\includegraphics[width=\linewidth]{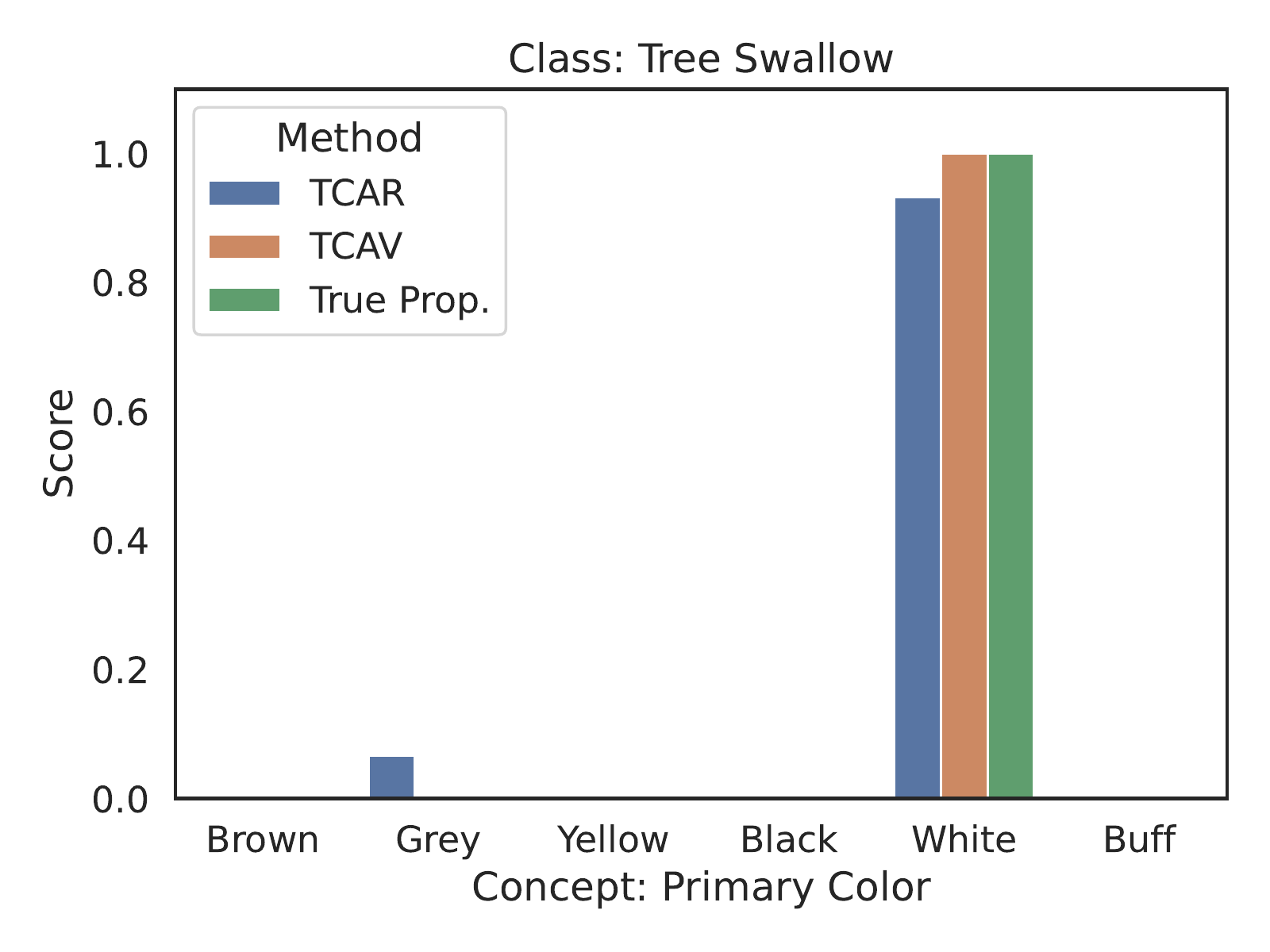}
		\caption{Tree swallow primary colour}
	\end{subfigure}
	
	\caption{Global concept explanations for CUB}
	\label{fig:cub_tcar_extra}
\end{figure}

\begin{figure}[!ht]
	\centering
	\begin{subfigure}{0.15\linewidth}
		\includegraphics[width=\linewidth]{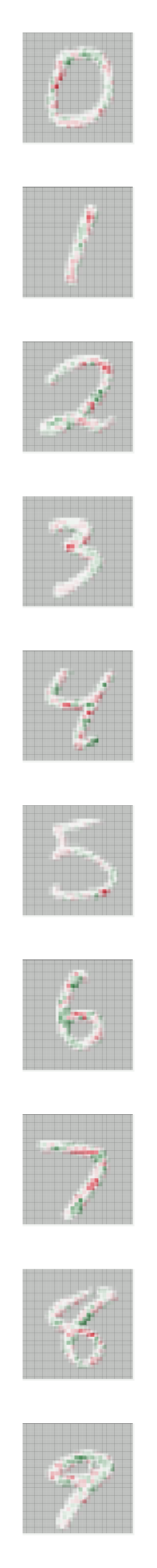}
		\caption{Loop}
	\end{subfigure}
	\hfil
	\begin{subfigure}{0.15\linewidth}
	\includegraphics[width=\linewidth]{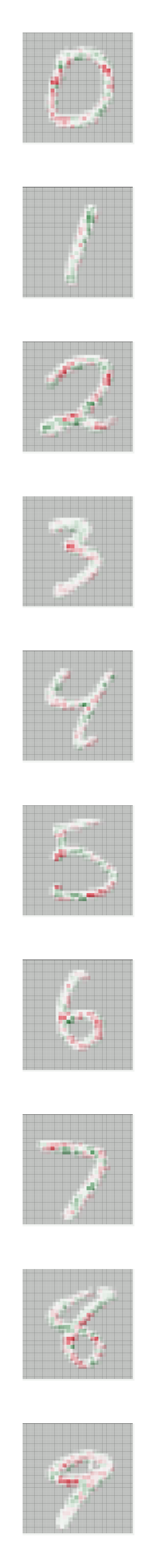}
	\caption{Vertical line}
	\end{subfigure}
	\hfil
	\begin{subfigure}{0.15\linewidth}
		\includegraphics[width=\linewidth]{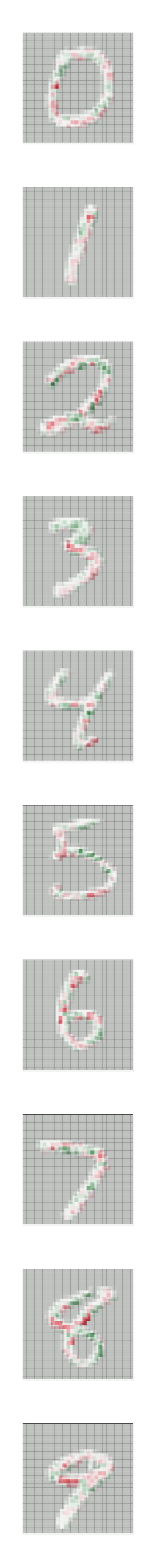}
		\caption{Horiz. line}
	\end{subfigure}
	\hfil
	\begin{subfigure}{0.15\linewidth}
		\includegraphics[width=\linewidth]{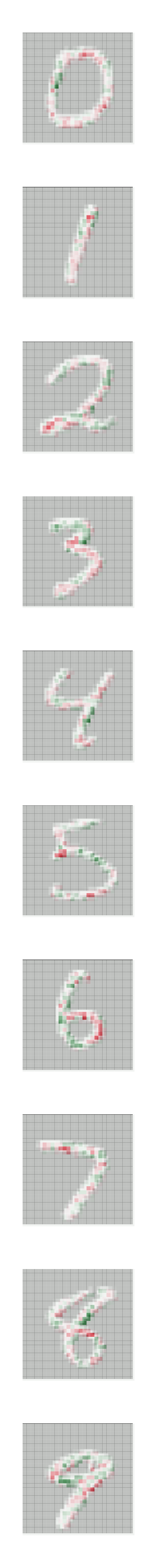}
		\caption{Curvature}
	\end{subfigure}
	\hfil
	\begin{subfigure}{0.15\linewidth}
		\includegraphics[width=\linewidth]{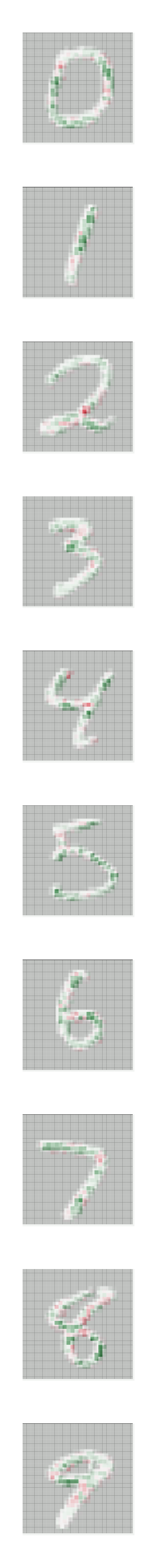}
		\caption{Vanilla}
	\end{subfigure}

	\caption{Concept-based saliency maps for MNIST}
	\label{fig:mnist_saliency}
\end{figure}

\begin{figure}[!ht]
	\centering
	\begin{subfigure}{0.49\linewidth}
		\includegraphics[width=\linewidth]{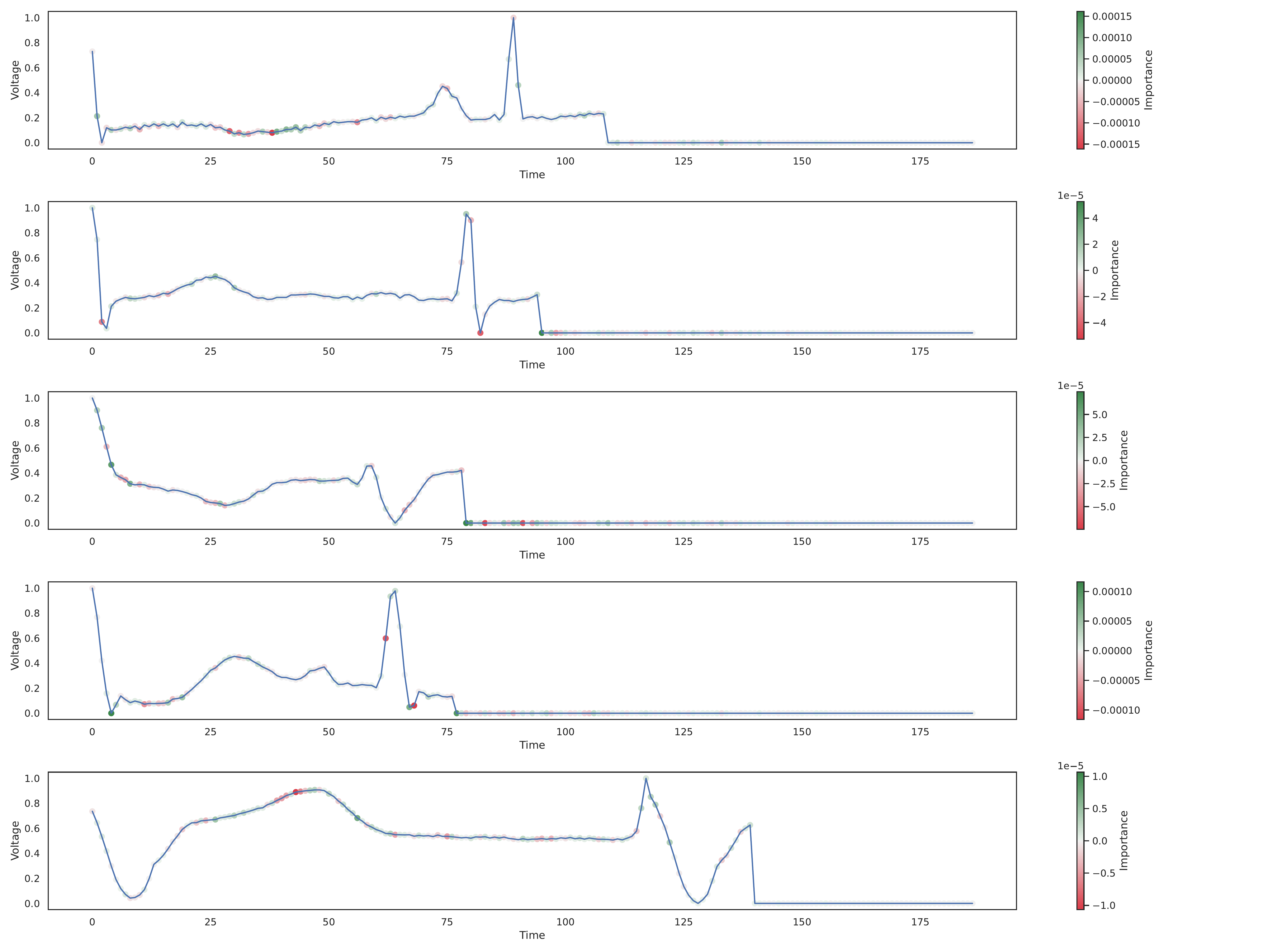}
		\caption{Supraventricular}
	\end{subfigure}
	\hfil
	\begin{subfigure}{0.49\linewidth}
		\includegraphics[width=\linewidth]{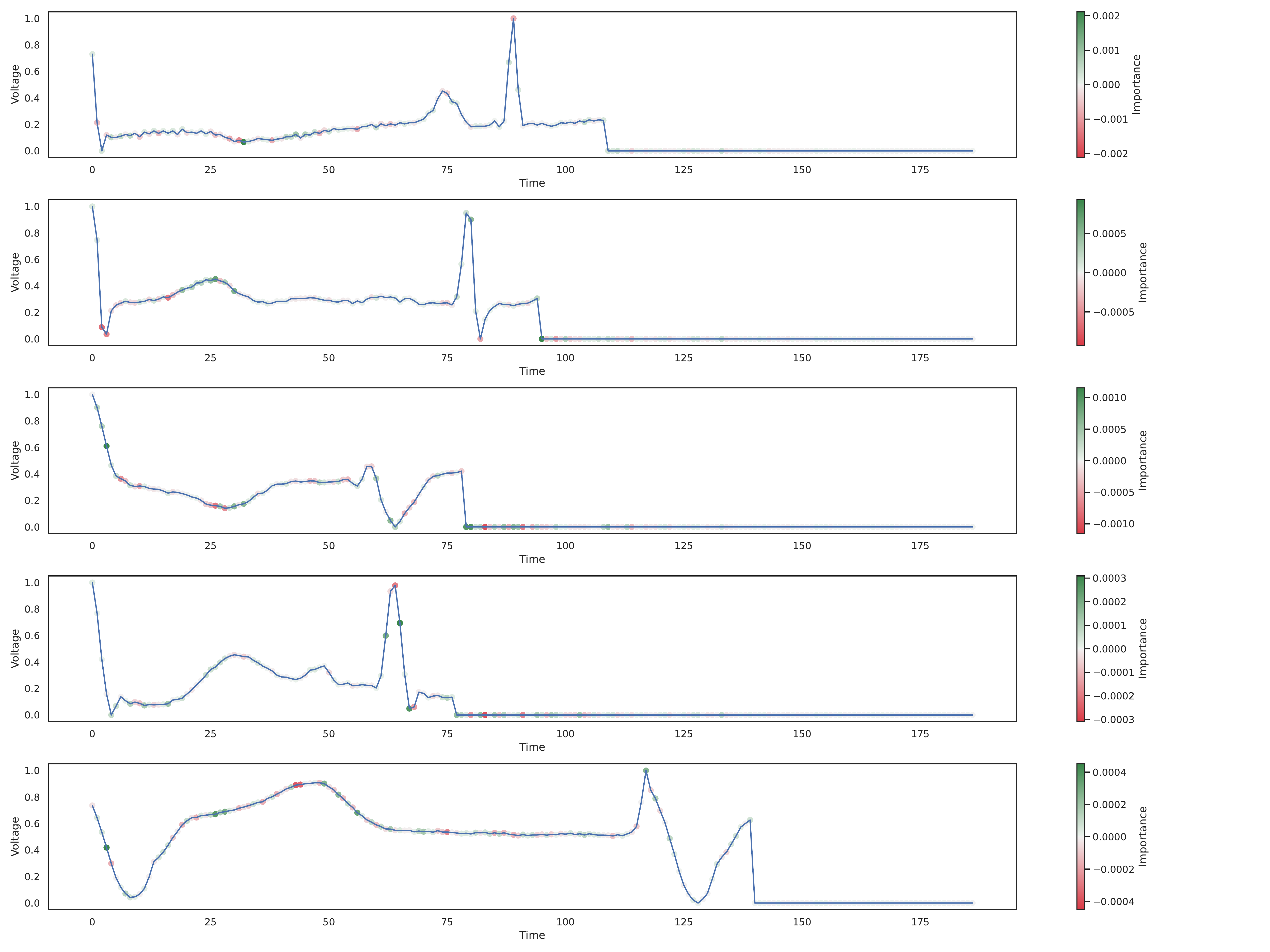}
		\caption{Premature ventricular}
	\end{subfigure}
	\hfil
	\begin{subfigure}{0.49\linewidth}
		\includegraphics[width=\linewidth]{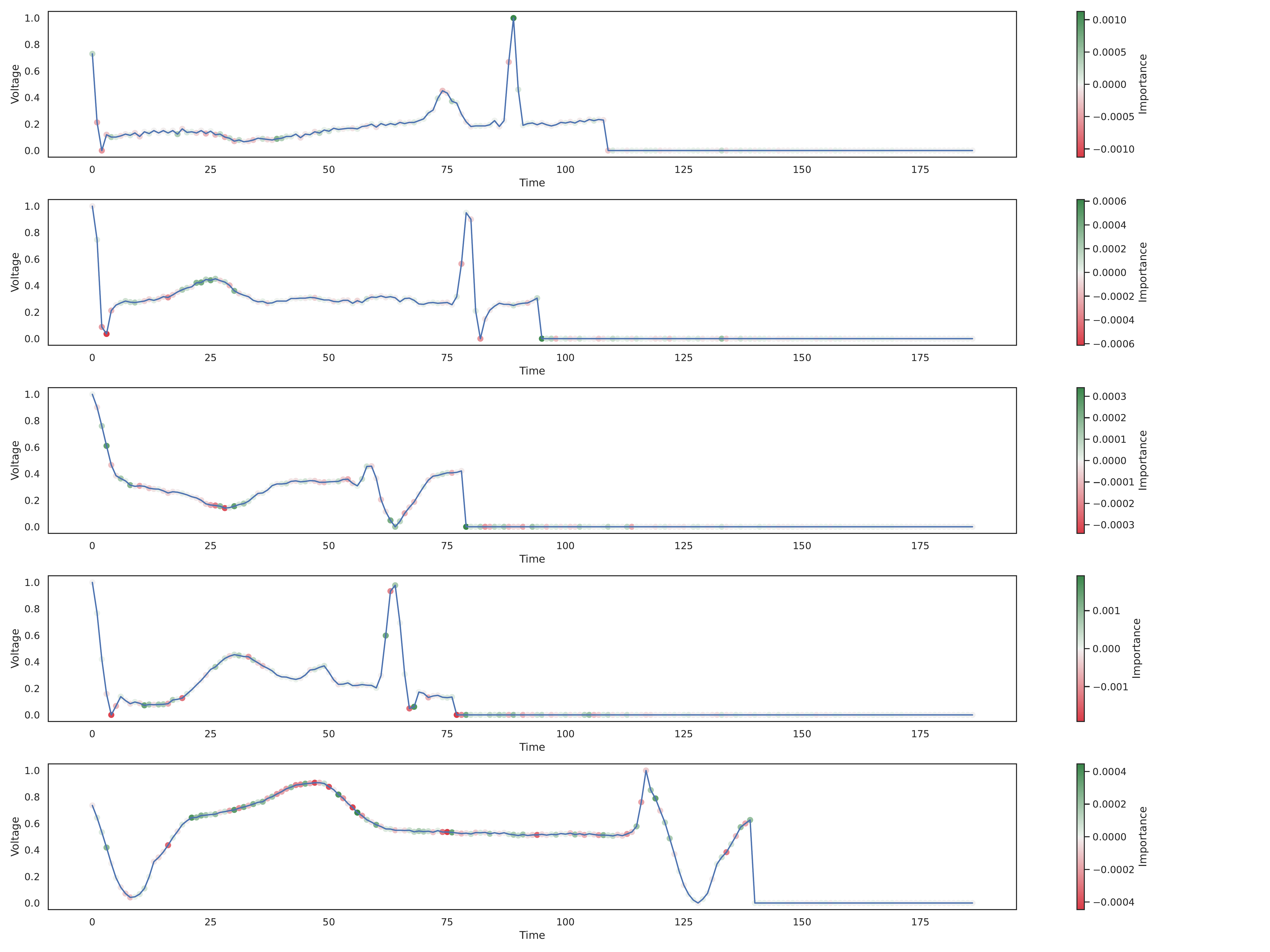}
		\caption{Fusion beats}
	\end{subfigure}
	\hfil
	\begin{subfigure}{0.49\linewidth}
		\includegraphics[width=\linewidth]{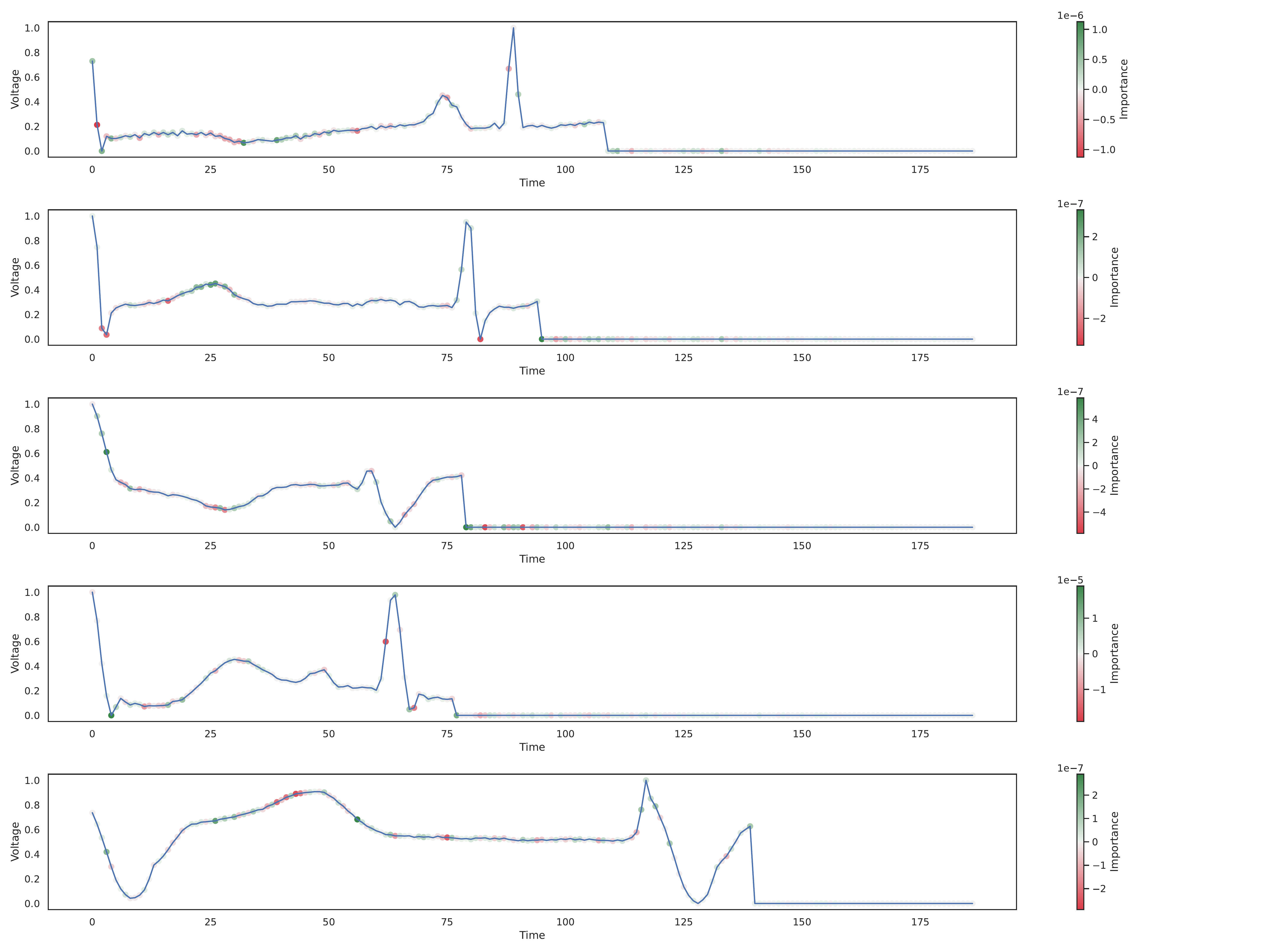}
		\caption{Unknown}
	\end{subfigure}
	\hfil
	\begin{subfigure}{0.49\linewidth}
		\includegraphics[width=\linewidth]{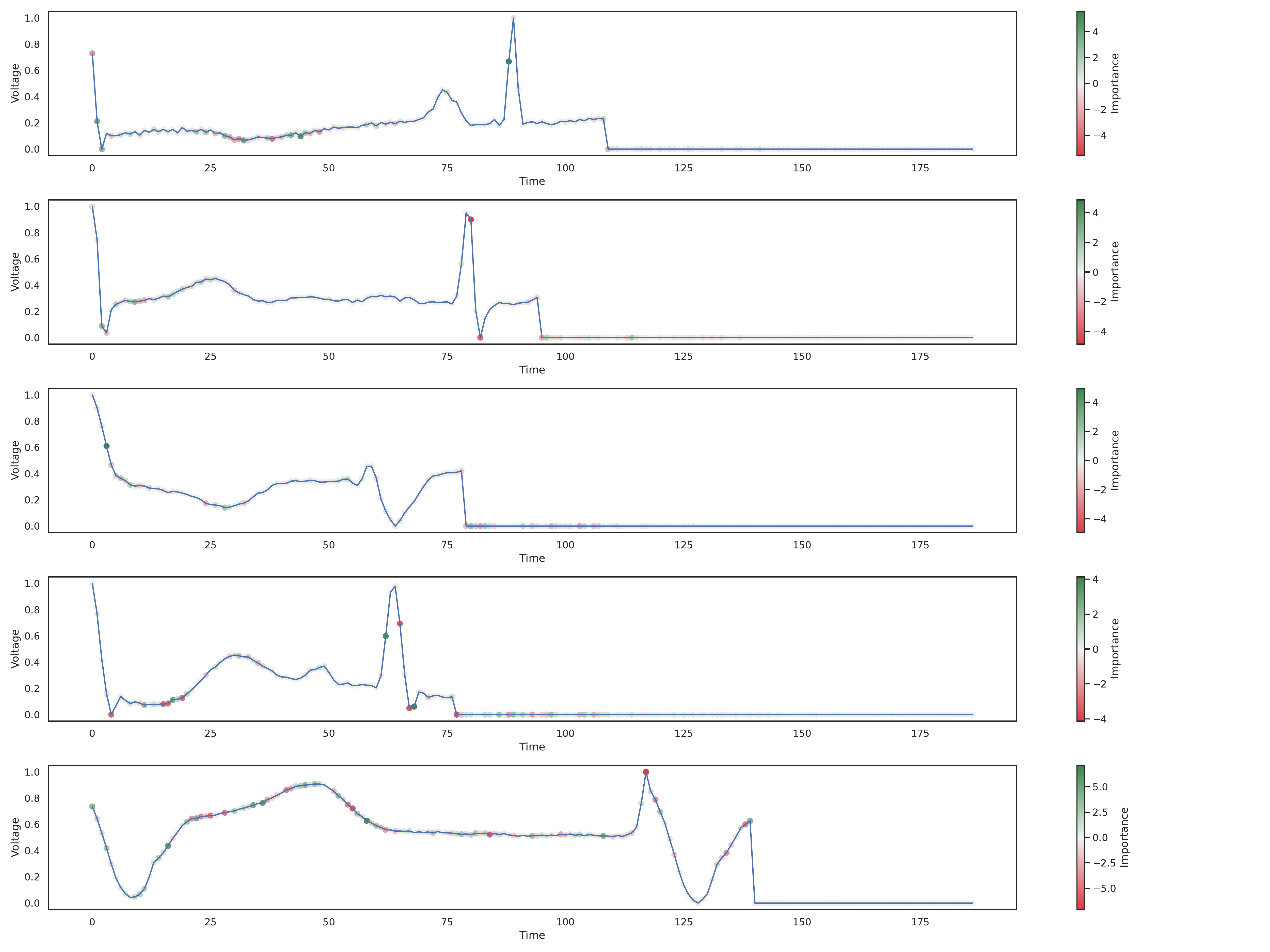}
		\caption{Vanilla}
	\end{subfigure}
	
	\caption{Concept-based saliency maps for ECG}
	\label{fig:ecg_saliency}
\end{figure}

\begin{figure}[!ht]
	\centering
	\begin{subfigure}{0.49\linewidth}
		\includegraphics[width=\linewidth]{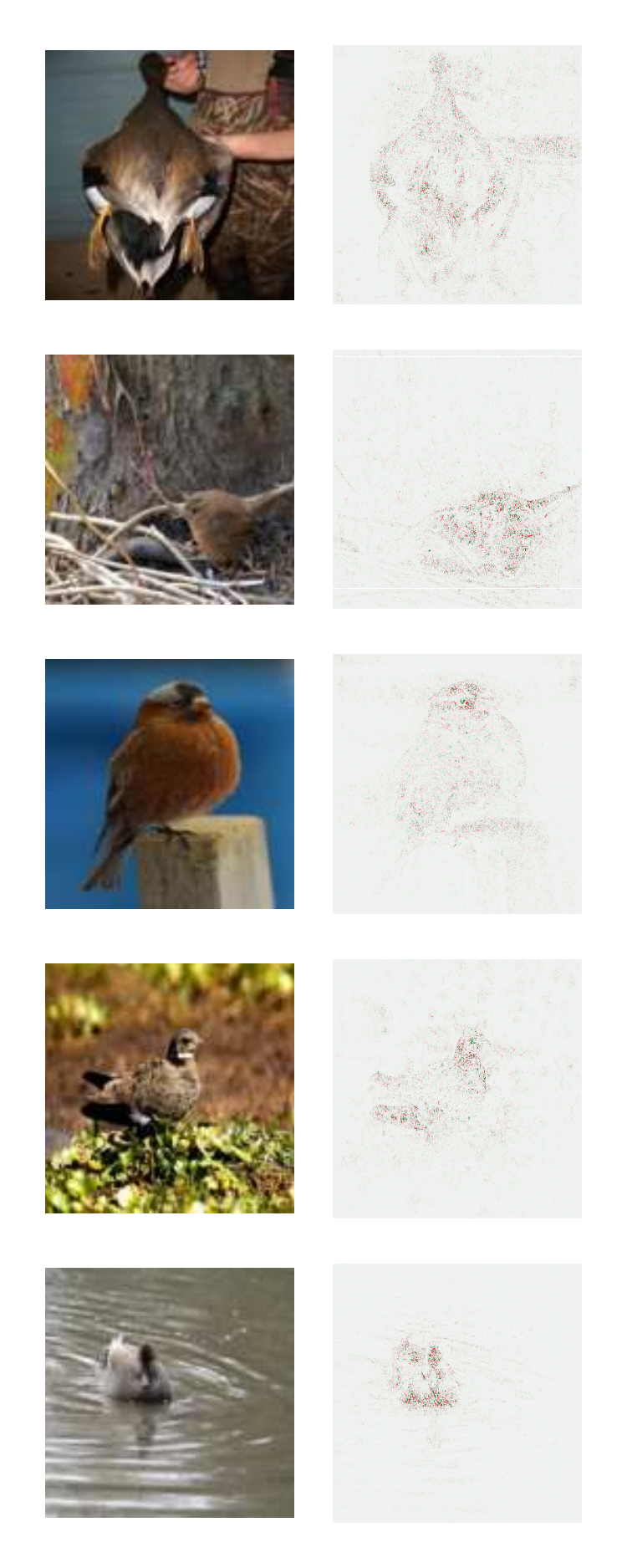}
		\caption{Brown belly}
	\end{subfigure}
	\hfil
	\begin{subfigure}{0.49\linewidth}
		\includegraphics[width=\linewidth]{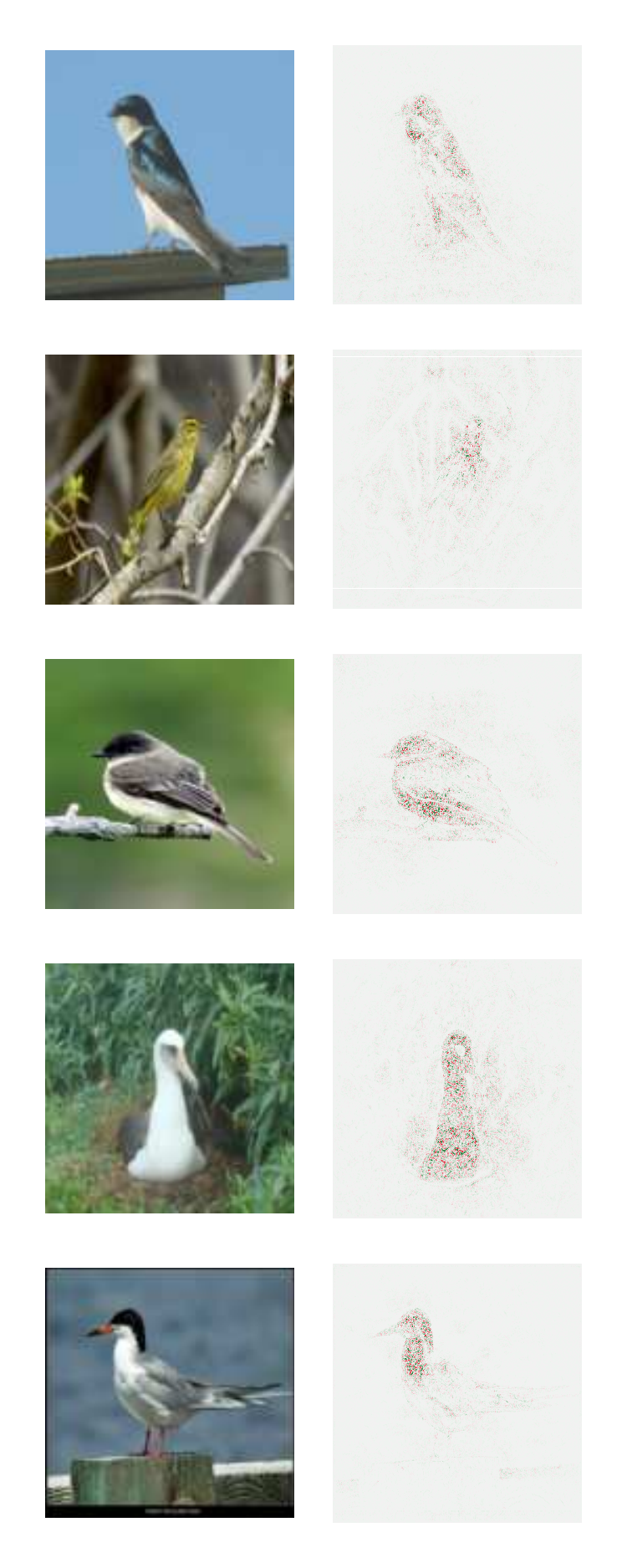}
		\caption{Solid belly pattern}
	\end{subfigure}

	\caption{Concept-based saliency maps for CUB (1/3)}
	\label{fig:cub_saliency_belly}
\end{figure}

\begin{figure}[!ht]
	\centering
	\begin{subfigure}{0.49\linewidth}
		\includegraphics[width=\linewidth]{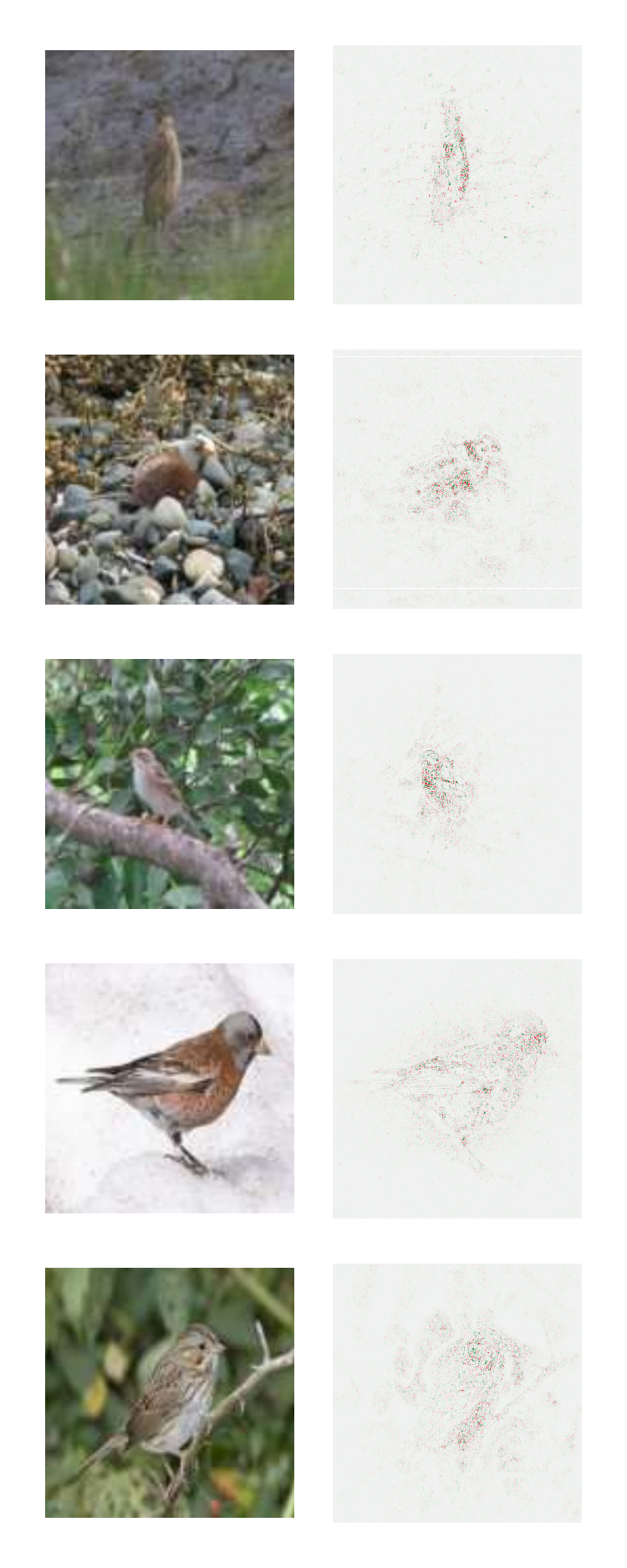}
		\caption{Stripped back pattern}
	\end{subfigure}
	\hfil
	\begin{subfigure}{0.49\linewidth}
		\includegraphics[width=\linewidth]{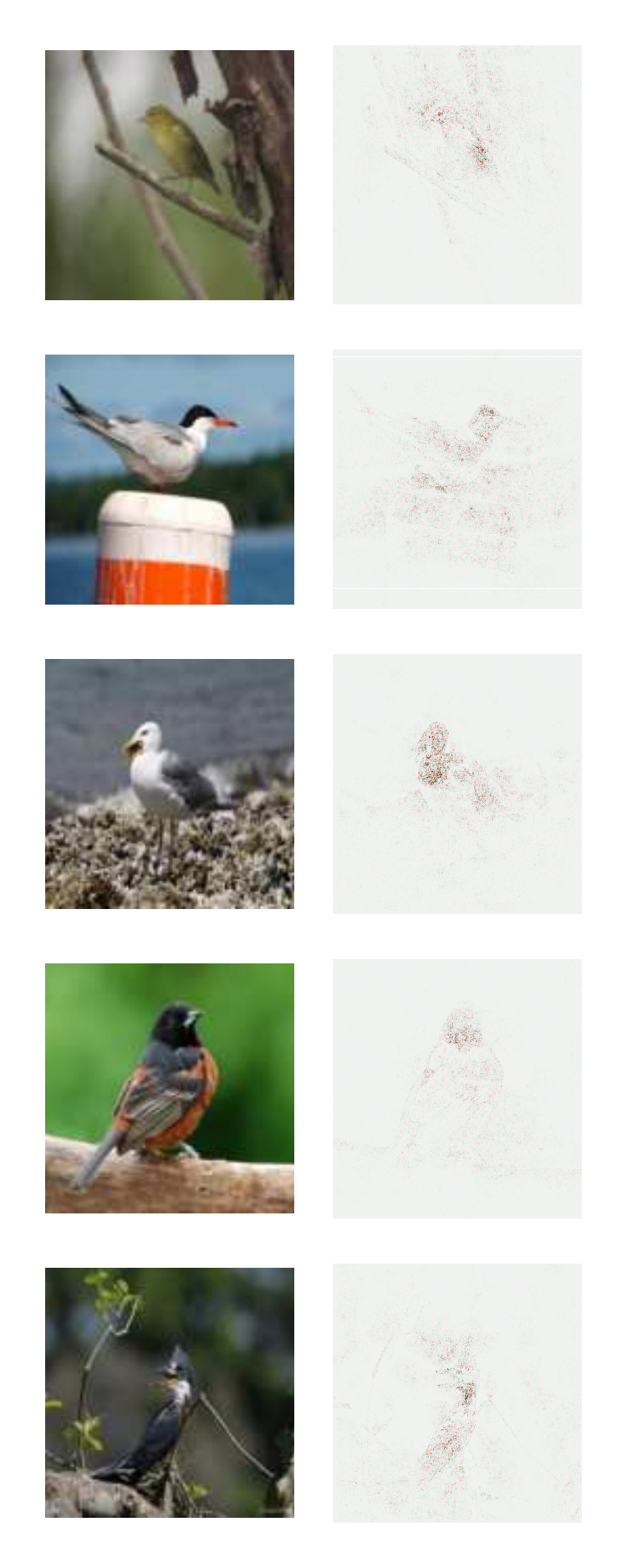}
		\caption{Solid back pattern}
	\end{subfigure}
	
	\caption{Concept-based saliency maps for CUB (2/3)}
	\label{fig:cub_saliency_back}
\end{figure}

\begin{figure}[!ht]
	\centering
	\begin{subfigure}{0.49\linewidth}
		\includegraphics[width=\linewidth]{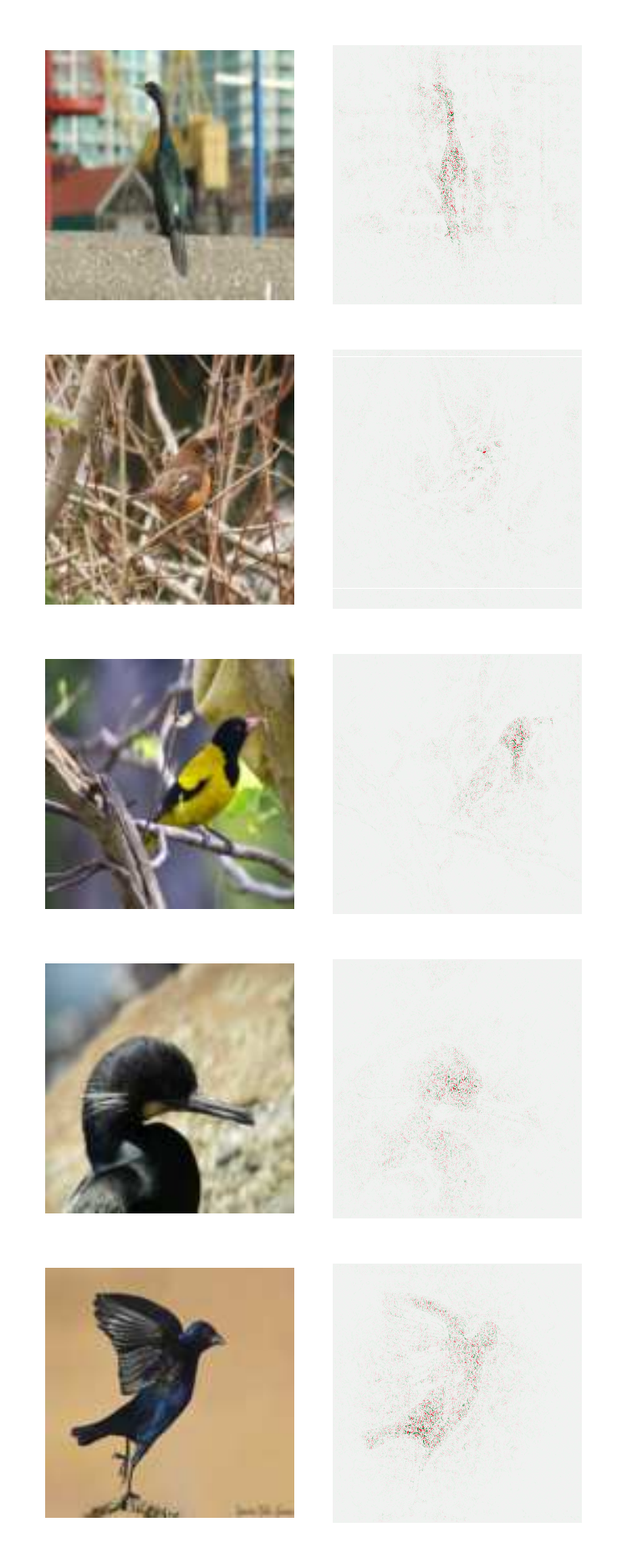}
		\caption{Black breast}
	\end{subfigure}
	\hfil
	\begin{subfigure}{0.49\linewidth}
		\includegraphics[width=\linewidth]{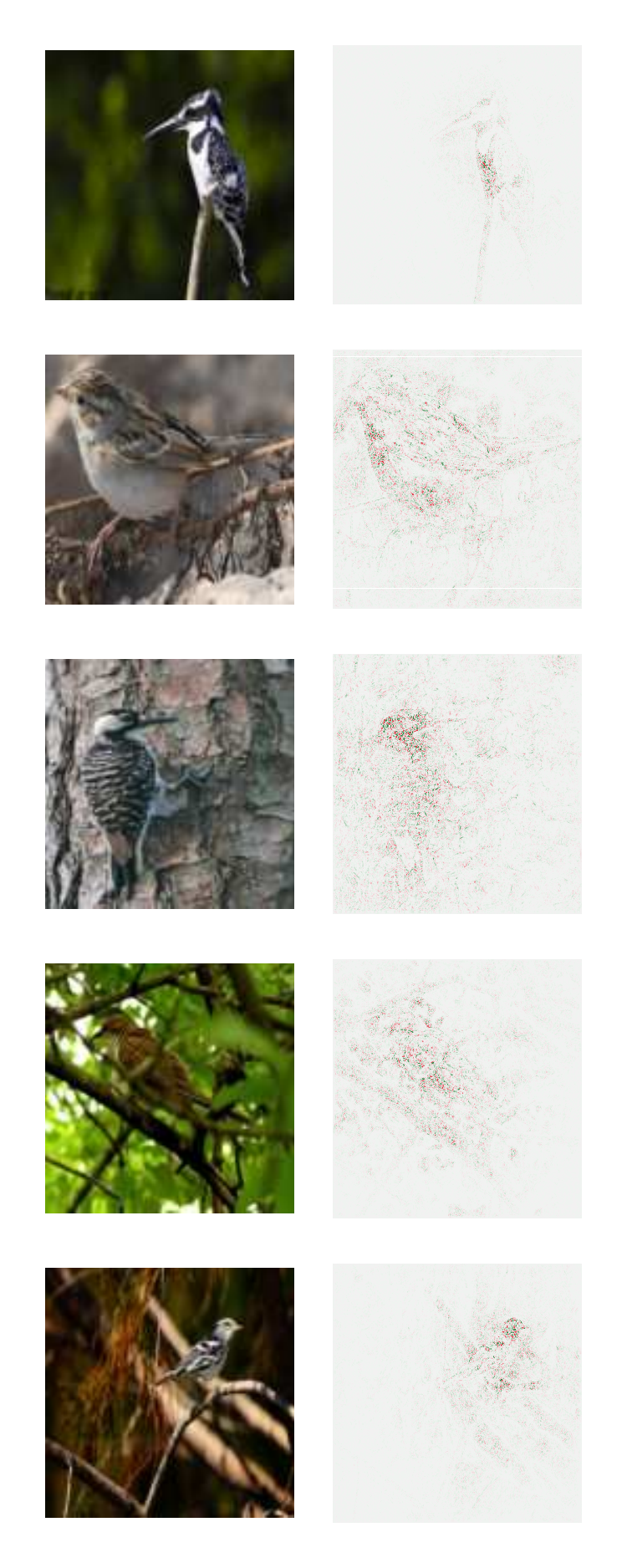}
		\caption{White breast}
	\end{subfigure}
	
	\caption{Concept-based saliency maps for CUB (3/3)}
	\label{fig:cub_saliency_breast}
\end{figure}

\end{document}